\documentclass[preprint,12pt]{elsarticle}

\usepackage{amssymb}

\usepackage[utf8]{inputenc}
\usepackage{lmodern} 
\usepackage[T1]{fontenc}
\usepackage{microtype}

\usepackage{graphicx}
\usepackage{float}
\usepackage[hypcap]{caption}
\usepackage{subcaption}
\usepackage{algorithm2e}
\usepackage{array}
\usepackage{algorithmic}

\usepackage{url}

\usepackage{textcomp}
\usepackage{booktabs}
\usepackage{array}

\usepackage{amsfonts}

\usepackage{geometry}
\usepackage{tikz}
\usetikzlibrary{trees,positioning,shapes}

\usepackage{amssymb,amsmath,amsthm}
\usepackage{amsfonts}
\usepackage{bm}
\usepackage{mathtools}
\usepackage{xcolor, colortbl}
\usepackage{appendix}

\newtheorem{lemma}{Lemma}

\usepackage{soul}
\usepackage{pgfplots,tikz}
\usepackage[eulergreek]{sansmath}
\usepackage{listings}
\usepackage{minted}
\pgfplotsset{compat=1.18} 

\usepackage[english]{babel}
\newcommand{\inblue}[1]{\textcolor{blue}{#1}}

\newcommand{\revision}[1]{\textcolor{black}{#1}}

\protect

\newcommand{\Prod}[3]{\left({#2}, {#3}\right)_{#1}}
\newcommand{\Prodh}[2]{\Prod{h}{#1}{#2}}

\newcommand{\Norm}[2]{\left\|{#2}\right\|_{#1}}
\newcommand{\Normh}[1]{\Norm{h}{#1}}
\newcommand{\NormU}[1]{\Norm{\mathbb{U}}{#1}}
\newcommand{\NormV}[1]{\Norm{\mathbb{V}}{#1}}
\newcommand{\Normdh}[1]{\Norm{\nabla, h}{#1}}

\newcommand{\DiscreteDer}[2]{\nabla_{#1}^{#2}}
\newcommand{\Dp}[1][]{\DiscreteDer{#1}{+}}
\newcommand{\Dm}[1][]{\DiscreteDer{#1}{-}}
\newcommand{\Dc}[1][]{\DiscreteDer{#1}{c}}
\newcommand{\Dxp}{\Dp[x]}
\newcommand{\Dxm}{\Dm[x]}
\newcommand{\Dxc}{\Dc[x]}
\newcommand{\Dyp}{\Dp[y]}
\newcommand{\Dym}{\Dm[y]}
\newcommand{\Dyc}{\Dc[y]}

\newcommand{\Dtp}{\Dp[t]}
\newcommand{\Dtm}{\Dm[t]}
\newcommand{\Dtc}{\Dc[t]}

\newcommand{\Tp}[1]{\tau_{#1}^{+}}
\newcommand{\Tm}[1]{\tau_{#1}^{-}}
\newcommand{\Txp}{\Tp{x}}

\newcommand{\Ttp}{\Tp{t}}
\newcommand{\Txm}{\Tm{x}}

\newcommand{\Ttm}{\Tm{t}}

\newcommand{\Dh}{D_h}
\newcommand{\Dzh}{D_{0,h}}

\protect

\usepackage[colorlinks=true, linkcolor=blue, urlcolor=blue]{hyperref}
\usepackage{cleveref}

\journal{Applied Soft Computing}

\begin{document}

\begin{frontmatter}



\title{Collocation-based Robust Physics Informed Neural Networks for time-dependent simulations of pollution propagation under thermal inversion conditions on Spitsbergen}


\author[inst1]{Maciej Sikora}
\author[inst1]{Leszek Siwik}

\affiliation[inst1]{organization={AGH University of Krakow, Faculty of Computer Science},
            addressline={Al. Mickiewicza 30}, 
            city={Krak\'ow},
            postcode={30-059}, 
            country={Poland}}


\author[inst3]{Natalia Leszczy\'nska}

\affiliation[inst3]{organization={Medical University of Silesia-Katowice, Faculty of Medical Sciences},
            addressline={ul. Poniatowskiego 15}, 
            city={Katowice},
            postcode={40-055}, 
            country={Poland}}

\author[inst4]{Tomasz Maciej Ciesielski}

\affiliation[inst4]{organization={The University Centre in Svalbard, Longyearbyen},
            addressline={Box 156 N-9171}, 
            city={Longyearbyen},
            country={Norway}}

\author[inst6,inst7]{Eirik Valseth}
\affiliation[inst6]{organization={Simula Research Laboratory},
            addressline={Kristian Augusts gate 23}, 
            city={Oslo},
            postcode={0164}, 
            country={Norway}}
\affiliation[inst7]{organization={Norwegian University of Life Sciences},
            addressline={Postboks 5003}, 
            city={As},
            postcode={1432}, 
            country={Norway}}

\author[inst9]{Manuela Bastidas Olivares}

\affiliation[inst9]{organization={Universidad Nacional de Colombia},
adressline={ Carrera 45 No. 26-85 - Uriel Gutiérrez Building},
city={Bogota},
country={Colombia}}

\author[inst1]{Marcin \L{}o\'s}

\author[inst1]{Tomasz S\l{}u\.zalec}

\author[inst8]{\\ Jacek Leszczy\'nski}

\affiliation[inst8]{organization={AGH University of Krakow, Faculty of Energy and Fuels},
            addressline={Al. Mickiewicza 30}, 
            city={Krak\'ow},
            postcode={30-059}, 
            country={Poland}}

\author[inst1]{ Maciej Paszy\'nski\corref{cor1}}
\cortext[cor1]{Corresponding author. Email address: paszynsk@agh.edu.pl}

\begin{abstract}
In this paper, we propose a Physics-Informed Neural Network framework for time-dependent simulations of pollution propagation originating from moving emission sources. We formulate a robust variational framework for the time-dependent advection-diffusion problem and establish the boundedness and inf-sup stability of the corresponding discrete weak formulation. Based on this mathematical foundation, we construct a robust loss function that is directly related to the true approximation error, defined as the difference between the neural network approximation and the (unknown) exact solution. Additionally,~a collocation-based strategy is introduced to speed up neural network training.  
\revision{We also extend our model to the non-linear Burgers-type equations.}
As a case study, we investigate pollution propagation caused by snowmobile traffic in Longyearbyen, Spitsbergen, supported by detailed in-field measurements collected using dedicated sensors. The proposed framework is applied to analyze the effects of thermal inversion on pollutant accumulation. 
Our results demonstrate that thermal inversion traps dense and humid air masses near the ground, significantly enhancing particulate matter (PM) concentration and worsening local air quality. 
\revision{We compare our method to linear computational cost explicit dynamics solver using isogeometric analysis. We show that our method can solve non-linear non-stationary problem four times slower than IGA solver can solve the linear problem. It makes it particularly interesting to investigate the non-linear extensions with similar computational cost as the linear formulation.}
\end{abstract}


\end{frontmatter}

\section{Introduction}

Air pollution has been an underestimated or even completely ignored health-degenerating factor for many years. However, over time, it has become clear that it negatively affects almost every organ and system in the body, including, but not limited to, the respiratory, cardiovascular, nervous, reproductive, and immune systems~\cite{Capello2019,Wolf2022}.      
As the main components of pollution, we identify the particulate matter (PM), nitrogen dioxide (\(\text{NO}_2\)), sulfur dioxide (\(\text{SO}_2\)), carbon monoxide (\(\text{CO}\)), and ozone (\(\text{O}_3\)). Among these, the European Environment Agency recognizes fine particulate matter (\(\text{PM}_{2.5}\)) as one of the most dangerous. 

According to the World Health Organization (WHO)  Air Pollution Data Portal \cite{who_air_pollution}, around 7 million premature deaths occur annually worldwide (of which 600,000 concern children under the age of five) and may be directly and undoubtedly attributed to exposure to pollutants that we inhale, while 99\% of the population lives in areas where the concentration of pollutants exceeds WHO safety norms and targets~\cite{whoguide}.

Among the most common diseases directly linked to air pollutants, the following are usually referred to: asthma~\cite{Guarnieri2014,Tiotiu2020} and (exacerbated) allergies~\cite{Li2020,Damato2016}, chronic obstructive pulmonary disease~\cite{Duan2020,Hsu2021}, lung (as well as bladder and breast) cancer~\cite{Turner2017,Tourner2020,Martoni2018}, heart attacks, and (ischemic) strokes (via atherosclerosis and increased blood pressure)~\cite{Franchini2012,Rajagopalan2018, Bourdel2017}, Alzheimer’s~\cite{Fu2020,Peters2023} and Parkinson’s~\cite{Murata2022,Kwon2024} diseases, depression~\cite{Borroni2022,Gladka2022}, anxiety~\cite{Zundel2022,Braithwaite2019}, and type 2 diabetes due to its effects on inflammation and insulin resistance~\cite{Lim2019,Paul2020,Hu2020}.

In this paper, we focus on analyzing the air quality in the Longyearbyen valley on Spitsbergen Island in the Svalbard archipelago (c.a. 1,300 km south of the North Pole).

The city of Longyearbyen is located in a valley and is \mbox{3/4} surrounded by mountains. Limited sunlight causes high cloud cover, and low temperatures result in high air density. It results in the thermal inversion effect (especially during the polar night period of time), enhancing the pollutant accumulation near the ground, significantly enhancing particulate matter (PM) concentration, and worsening local air quality. 

The Longyearbyen residents complain about the influence of pollution on their health. From the survey summarized in \cite{SIKORA2025113394}, around one-fifth of respondents reported respiratory issues, mostly during winter months and in the evening or at night, i.e., during the periods and seasons when the pollutants may not be easily transferred to other areas while being trapped near ground level.

As the anthropogenic pollution sources they are indicated both the global one, i.e., the migration of air pollution from continental countries, as well as the local-ones namely, the local power plant operating on diesel fuel and increased snowmobiles traffic resulting from the rapid expansion of tourism~\cite{ TIMLIN2022100122, Pacyna, ANDERSSONSTAVRIDIS2024167562}.

The impact of the power plant emitting pollutants continuously and at the same location has been analyzed in \cite{SIKORA2025113394}. In this paper, we focus on the increasing snowmobile traffic affecting local air quality more significantly. According to the Environmental Monitoring of Svalbard and Jan Mayen system the number of registered snowmobiles in the region increased from 200 in 1973, up to 3000 in 2024. 

As the computational and simulation framework, we present the Collocation-based Robust Variational Physics Informed Neural Network (CRVPINN) for numerical simulations of pollution propagation.
The framework presented here is an extension of the stationary code
to time-dependent problems discussed in~\cite{LOS2025107839}, and it is the first application of the CRVPINN method to non-stationary problems. 

The goal of the computer simulation was to understand the influence of thermal inversions on the propagation of snowmobile pollution in Spitsbergen, and the simulation results showed that the polluted air, while trapped near ground level, cannot be carried out of the region and significantly affects local air quality. 

\revision{We introduce the Physics Informed Neural Network method with robust loss function for the time-dependent simulations of pollution propagation.  
We introduce the robust loss function for time-dependent advection-diffusion problems, based on our original proofs of inf-sub stability and boundedness of the discrete weak formulation in the space-time setup.
We also introduce the collocation method to speed up the training process. We solve linear advection-diffusion equations with carefully selected forcing term to model the snowmobile movement. We also extend the linear model to the non-linear Burgers-type equations. We compare the robust loss computations with explicit dynamics solver using smooth approximation with B-spline basis functions, exploiting the Kronecker product structure of the mass matrix to obtain linear computational cost solver. We show that our robust loss solver for both linear and non-linear case is four times slower than then the explicit dynamics solver using the B-spline based discretizations to solve a linear problem.}

As an interesting case study, we focus on the pollution propagation from snowmobiles at Spitsbergen. We simulate the accumulation of pollution generated by snowmobiles at ground level due to the vertical temperature profiles at Spitsbergen using Physics Informed Neural Networks (PINN).
We also present detailed measurements of pollution emitted by snowmobiles in Spitsbergen using Airly sensors.

The structure of the paper is as follows. We derive the collocation-based robust variational PINN formulation of the advection-diffusion model in Section 2. 
The following Section 3 contains the proofs of the robustness of the method, showing the boundedness and inf-sup stability of the corresponding weak formulation.
\revision{In Section 4 we verify the method on model problem with manufactured solution, and we compare its efficiency to IGA-ADS \cite{10.1007/978-3-031-35995-8_13,LOS2024102238} solver.}
Next, in Section 5, we focus on the in-field measurements performed from snowmobiles in Longyearbyen. \revision{Section 6 contains the numerical experiments with linear and non-linear model, and the qualitative comparison with the in-field measurements. We conclude the paper in Section 7 and provide the code description and the code links in the Appendix.}

\section{Neural networks estimating the pollution propagation from snowmobiles} \label{sec:pinnsec}

\revision{We consider the non-stationary (transient) advection--diffusion equation defined on the spatial domain
$\Omega=[0,L_x]\times[0,L_y]$, $t\in[0,T]=I,
$
for the pollutant concentration
$u \colon \Omega\times[0,T]\rightarrow\mathbb{R}$. 
The governing equation is
\begin{equation}
\frac{\partial u}{\partial t}
+
v_y(y)\frac{\partial u}{\partial y}
-
K_x\frac{\partial^2 u}{\partial x^2}
-
K_y\frac{\partial^2 u}{\partial y^2}
=
S(x,y,t),
\label{PDE}
\end{equation}
where $K_x$ and $K_y$ denote the diffusion coefficients in the horizontal and vertical directions, $v_y(y)$ is the vertical advection velocity, $S(x,y,t)$ represents a moving pollution source. These equations can be also rewritten in a more compact way
    \begin{equation}
    \frac{\partial u(\mathbf{x},t)}{\partial t}
    -\nabla \cdot (\epsilon\nabla u(\mathbf{x},t)) 
    +{\bf b} \cdot \nabla u(\mathbf{x},t) = S(\mathbf{x},t),
    \mbox{for}\;(\mathbf{x},t)\in \Omega \times I, \label{eq:advection}
    \end{equation}
    with  ${\bf b}=(0,v_y(y))$, and $\epsilon=\begin{bmatrix} K_x & 0 \\ 0 & K_y \end{bmatrix}=\begin{bmatrix} 1.0 & 0 \\ 0 & 0.1 \end{bmatrix}$.}
\revision{In the present model no horizontal advection is assumed,
$v_x=0$,
while the vertical transport velocity is prescribed as
\begin{equation}
v_y(y)=
\begin{cases}
4, & y\ge 0.1,\\
0, & y<0.1.
\end{cases}
\end{equation}
Consequently, the pollutant is transported upward only above the near-ground layer.}

\revision{The pollution source models emissions generated by moving snowmobiles (or other moving vehicles). The source propagates in the positive $x$ direction with constant velocity while its intensity decays in the transverse ($y$) direction.
The maximum pollution intensity at height $y$ is described by the Gaussian profile
\begin{equation}
H(y)
=
H_{\max}
\exp\!\left(
-\frac{y^2}{2\sigma_y^2}
\right),
\end{equation}
where
$
H_{\max}=0.5,
\sigma_y=0.4.
$
The moving front position is
$
x_f(t)=Vt,
$
with constant velocity
$V=10.
$
To obtain a smooth moving front, we define
$
\alpha(x,t)
=
\frac{x_f(t)-x}{W},
$
where
$
W=0.8
$
controls the width of the transition zone.
The normalized variable is clipped to the interval $[0,1]$,
$
\hat{\alpha}
=
\min\left(1,\max(0,\alpha)\right),
$
and a smooth cosine ramp is used,
$
R(\hat{\alpha})
=
\frac12
\left(
1-\cos(\pi\hat{\alpha})
\right).
$
The resulting moving pollution source is therefore
\begin{equation}
S(x,y,t)
=
\begin{cases}
H(y)\,R(\hat{\alpha}),
&
y\le0.2,
\\
0,
&
y>0.2.
\end{cases}
\label{eq:source}
\end{equation}
Thus, the source occupies a narrow region near the ground while moving along the positive $x$ direction with constant speed.}


    
    \subsection{Physics Informed Neural Networks} \label{sec:pinns}
    
    To numerically solve~\eqref{PDE}, we first employ PINNs~\cite{raissi2019physics,cai2021physics,WOS:000526518300074} to approximate the solution of partial differential equations  (PDEs) using neural networks. \revision{We employ a fully connected neural network that represents the pollution scalar field.} It assumes that the neural network input is the point in the space-time domain $(\mathbf{x},t)=(x,y,t)$ and the output from the neural network is the concentration of pollution at point $\mathbf{x}$ at time moment $t$, denoted by $u_\theta(\mathbf{x},t)$.    \revision{Namely, \begin{equation}
    u_\theta({\bf x},t) = {\cal NN}({\bf x},t) = \sigma_n(A_n \sigma_{n-1} ( \cdots \sigma_1( 
    A_1 [{\bf x},t]^T
    + \beta_1) \cdots) + \beta_n)
\end{equation}
where $\{A_i\}_{i=1,...,n}$ are layer weights,  $\sigma$ are nonlinear activation functions, and $\{\beta_i\}_{i=1,...,n}$ are biases. In general, we denote by $\theta = \{A_i,\beta_i\}_i$ all the parameters of the neural network.}
    It minimizes the PDE residual \eqref{eq:advection}:

    \begin{equation}\label{LossPDE}
    \begin{aligned}
    {\cal RES}_{PDE}(u_\theta(\mathbf{x},t)) &= \left(\right.
    \frac{\partial u_\theta(\mathbf{x},t)}{\partial t}
    -\nabla \cdot (\epsilon\nabla u_\theta(\mathbf{x},t)) & \left.
    +{\bf b} \cdot \nabla u_\theta(\mathbf{x},t) -\revision{S}(\mathbf{x},t)\right),
    \\ & & \mbox{for}\;(\mathbf{x},t)\in \Omega \times I,\\
    {\cal L}_{PDE} &= \sum_{(\mathbf{x}_i,t_i)}\left(    {\cal RES}_{PDE}(u(\mathbf{x}_i,t_i))\right)^2,
    \end{aligned}
    \end{equation}
    
\noindent together with the initial condition residual:    

    \begin{equation}\label{LossInitial}
    \begin{aligned}
    {\cal RES}_{Init}(u_\theta(\mathbf{x},0)) &= \left(
    u_\theta(\mathbf{x},0)-u_0(\mathbf{x})\right),
    \; \mbox{for}\;\mathbf{x}\in \Omega,\\
    {\cal L}_{Init} &= \sum_{(\mathbf{x}_i,0)}\left(    {\cal RES}_{Init}(u(\mathbf{x}_i,0))\right)^2,
    \end{aligned}
    \end{equation}

\noindent  and the boundary condition residual:

    \begin{equation}\label{LossBC}
    \begin{aligned}
    {\cal RES}_{BC}(u_\theta(\mathbf{x},t)) &= \left(
    \frac{\partial u_\theta\mathbf{x},t)}{\partial {\bf n}}-0\right),
      \; \mbox{for}\;(\mathbf{x},t)\in \partial \Omega \times I,\\
    {\cal L}_{BC} &= \sum_{(\mathbf{x}_i,t_i)}\left(    {\cal RES}_{BC}(u(\mathbf{x}_i,t_i))\right)^2.
    \end{aligned}
    \end{equation}
\revision{where ${\bf n}$ is the vector normal to the boundary.}
Note that the residual loss function is based on the strong residual of the PDE and is computed by selecting a set of collocation points $(\mathbf{x},t)\in \Omega \times I$. Additionally, note that the boundary loss uses selected collocation points located on the boundary $(\mathbf{x},t)\in \partial \Omega \times I$, and the initial loss uses the collocation points selected from the spatial domain $\mathbf{x} \in \Omega$ at the initial time.

    \subsection{Variational Physics Informed Neural Networks}

In this section, we present a transformation of the strong-form residuals of PINN into the weak-form residual of Variational Physics Informed Neural Networks \cite{kharazmi2021hp}. For this purpose, we consider the space-time domain $\Omega \times I$ and perform a shift with respect to the initial boundary condition.
We define the shift $u_{shift}$ such that:
 \begin{equation}
 \begin{aligned}
 u(\mathbf{x},t) &= w(\mathbf{x},t)+u_{shift}(\mathbf{x},t), \\ 
 u_{shift}(\mathbf{x},t) & = (1-\frac{1}{T}t)u_0(\mathbf{x}).
 \label{eq:shift}
 \end{aligned}
 \end{equation}
\revision{In other words, we split the solution into the part $u_{shift}$ that fulfills the initial condition $u_{shift}({\bf x},0)=(1-\frac{0}{T})u_0({\bf x})=u_0({\bf x})$ (using a function $1-\frac{t}{T})$, and the remaining part of the solution denoted by $w({\bf x},t)$. We substitute the decomposition into our equation \eqref{eq:advection}, 
   \begin{equation}\label{PDEweakshifta}
    \begin{aligned}
    \frac{\partial (w+u_{shift})}{\partial t}-&\nabla \cdot(\epsilon\nabla (w+u_{shift})) +{\bf b} \cdot \nabla (w+u_{shift})  = S,&~\mbox{in}\;\Omega \times I\\
    \end{aligned}
    \end{equation}
Since we know the formulas for $u_{shift}$ term we can move it to the right-hand side to obtain}

    \begin{equation}\label{PDEweakshift}
    \begin{aligned}
    \frac{\partial w}{\partial t}-&\nabla \cdot(\epsilon\nabla w) +{\bf b} \cdot \nabla w  =  
     &\underbrace{\revision{S}+\frac{\partial u_{shift}}{\partial t}-\nabla \cdot(\epsilon\nabla  u_{shift}) +{\bf b} \cdot \nabla u_{shift}}_{\mathbf{F}},&~\mbox{in}\;\Omega \times I\\
    \end{aligned}
    \end{equation}
    \noindent In the following, we overwrite the notation and replace $w$ with $u$ \revision{ and we have also denoted the new right-hand side by ${\bf F}$.
  We define $V = L^2(0,T; H^1_0(\Omega))$ and select a family of test functions $v \in V$. }
 We multiply the strong formulation (\ref{PDEweakshift}) 
 \revision{   \begin{equation}\label{weak1}
    \begin{aligned}
 \left(   \frac{\partial u}{\partial t}, v \right)_{\Omega \times I} -\left(\nabla \cdot (\epsilon\nabla u),  v\right)_{\Omega \times I}&+\left({\bf b} \cdot \nabla u,v\right)_{\Omega \times I}+(\mathbf{F},v)_{\Omega \times I},&\forall v \in V.\\
    \end{aligned}
    \end{equation}}
 and integrate it by parts \revision{the term 
$-\left(\nabla \cdot (\epsilon\nabla u),  v\right)_{\Omega \times I}=
\left(\epsilon\nabla u,\nabla  v\right)_{\Omega \times I}-(\epsilon \nabla u \cdot {\bf n}, v)_{\partial \Omega \times I}$.}
During the derivation of the weak formulation, we employed the fact that test functions are zero on the boundary to cancel out the boundary integral term \revision{$-(\epsilon \nabla u \cdot {\bf n}, v)_{\partial \Omega \times I}=0$}. \revision{We get}
    \begin{equation}\label{weak}
    \begin{aligned}
 \left(   \frac{\partial u}{\partial t}, v \right)_{\Omega \times I} +\left(\epsilon\nabla u,\nabla  v\right)_{\Omega \times I}
 -(\epsilon \nabla u \cdot {\bf n}, v)_{\partial \Omega \times I} 
 &+\left({\bf b} \cdot \nabla u,v\right)_{\Omega \times I}+(\mathbf{F},v)_{\Omega \times I},&\forall v \in V.\\
    \end{aligned}
    \end{equation}

Since we have eliminated the non-zero initial boundary conditions, we can now introduce a single loss based on the weak residual and test with the functions that are zero on the initial boundary.

    \begin{equation}\label{LossPDEweak}
    \begin{aligned}
    {\cal RES}_{PDE-weak}(u_\theta(\mathbf{x},t),v(\mathbf{x},t)) &= \\
 \left(   \frac{\partial u_\theta}{\partial t}, v \right)_{\Omega \times I} +\left(\epsilon\nabla u_\theta,\nabla  v\right)_{\Omega \times I}&+\left({\bf b} \cdot \nabla u_\theta,v\right)_{\Omega \times I}+(\mathbf{F},v)_{\Omega \times I}\\
    {\cal L}_{PDE-weak} &= \sum_{v(\mathbf{x},t)}\left(    {\cal RES}_{PDE-weak}(u_\theta(\mathbf{x},t),v(\mathbf{x},t))\right)^2.
    \end{aligned}
    \end{equation}

\noindent Note that this loss function is based on the weak-form residual, and it is computed by selecting a set of test functions $v \in V$. Also, having this weak residual, we do not need to minimize the initial and boundary condition residuals, since they are already included in the weak residual. 

    \subsection{Robust Variational Physics Informed Neural Networks} \label{sec:rvpinn}

As has been shown in \cite{ROJAS2024116904}, the VPINN loss is not robust and is not directly related to the true error. 
Thus, to control the convergence of the training, we introduce the Robust Variational Physics Informed Neural Network \cite{ROJAS2024116904} method with the following robust loss function:
    \begin{equation}\label{LossPDErobust}
    \begin{aligned}
    {\cal L}_{PDE-robust} &=  {\cal RES}_{PDE-weak}^T(u_\theta(\mathbf{x},t),v(\mathbf{x},t)){\cal G}_{vpinn}^{-1} {\cal RES}_{PDE-weak}(u_\theta(\mathbf{x},t),v(\mathbf{x},t)),
    \end{aligned}
    \end{equation}

\noindent where ${\cal G}_{vpinn}^{-1}$ is the inverse of the Gram matrix ${\cal G}_{vpinn}$. The Gram matrix is selected using the inner product norm of the test space $V$. In our case, we select: 

    \begin{equation}
    \begin{aligned}
    {{\cal G}_{vpinn}}_{m,n} &=  \left( v_m, v_n \right)_{V } = \int_0^T \int_\Omega \nabla v_m(\mathbf{x},t) \cdot \nabla v_n(\mathbf{x},t) d\mathbf{x} dt,
    \end{aligned}
    \end{equation}

\noindent where $v_m,v_n \in V$ is the selected set of test functions. Thus, the dimension of the Gram matrix ${\cal G}_{vpinn}$ is equal to the dimension of the test space.

    \subsection{Collocation-based Robust Variational Physics Informed Neural Networks} \label{sec:crvpinn}

The last step in our derivation is the replace the Robust VPINN by the Collocation-based PINN \cite{LOS2025107839}. In this way we replace the continuous description of the RVPINN and expensive three-dimensional integration with discrete domain constructed over the discrete set of collocation points. We also replace the gradients of the weak formulation by discrete finite-difference gradients computed over the collocation points.
In order to do that, we first define the set of points:

\begin{equation}
\Omega_h \times I_h := \{ (ih_x,jh_y,kh_t) \in \revision{[0,1]^2\times [0,T]}: 0 \leq i \leq N_x, 0 \leq j \leq N_y, 0 \leq j \leq N_t\},
\end{equation} where $h_x=1/N_x, h_y=1/N_y, h_t=T/N_t$. We consider:
\begin{equation}
D_h := \{ u:\Omega_h \times I_h\longrightarrow\mathbb{R}\}\cong \mathbb{R}^{(N_x+1)(N_y+1)(N_t+1)},
\end{equation} and equip it with the following discrete inner product and induced norm:
\begin{equation}
(u,v)_h := h_xh_yh_t \sum_{p\in \Omega_h \times I_h} u(p)v(p),\quad \|u\|_h^2 :=(u,u)_h, \qquad u,v\in D_h.
\end{equation}
Now, we follow the notation convention below for simplicity:
\begin{equation}
u_{i,j,k}:=u(ih_x,jh_y,kh_t),\qquad 0\leq i\leq N_x, 0\leq j\leq N_y, 0\leq k\leq N_t,
\end{equation} 
\revision{Let~$\Gamma_h = \partial \Omega \cap \Omega_h$ denote the set of discrete spatial boundary points.}
We also introduce a space \revision{of space-time functions} that \revision{satisfy the} zero-Dirichlet boundary condition
\revision{at the spatial boundaries}
\begin{equation}
D_{0,h} = \revision{\{ u \in D_h : \left. u\right|_{\Gamma_h \times I_h} = 0\}},
\end{equation}
\revision{as well as the space of~$D_{0,h}$ functions that vanish at~$t = 0$:}
\revision{
\begin{equation}
    \Dzh^0 = \{ u \in \Dzh : u(\bm{x}, 0) = 0 \quad \forall \bm{x} \in \Omega_h \}
\end{equation}
}
We introduce a canonical \revision{$\Prodh{\cdot}{\cdot}$-orthogonal} basis for $D_h$ given by a set of functions $\delta_{i,j,k}:\Omega_h \times I_h \longrightarrow\mathbb{R}$, $0\leq i\leq N_x, 0\leq j\leq N_y, 0\leq k\leq N_t, $ that behave as Kronecker deltas over $\Omega_h\times I_h$,

\begin{equation}
\label{eq:test}
\delta_{i,j,k}(\mathbf{x},t)=
\displaystyle{\left\{\begin{aligned}
&1 \quad \textrm{ if }\mathbf{x}=x_{i,j,k}, \\
&0 \quad \textrm{ if }\mathbf{x}\neq x_{i,j,k}.
\end{aligned}
\right.}
\end{equation} 
We introduce the finite difference gradient \revision{operators} given by
\begin{align}
\revision{\Dp} u_{i,j,k} := &\left(\revision{\Dxp} u_{i,j,k},\revision{\Dyp} u_{i,j,k},\revision{\Dtp} u_{i,j,k}\right)  :=  \\
&\left(\frac{u_{i+1,j,k}-u_{i,j,k}}{h_x},\frac{u_{i,j+1,k}-u_{i,j,k}}{h_y},\frac{u_{i,j,k+1}-u_{i,j,k}}{h_t},\right),\\
\revision{\Dm} u_{i,j,k} := &\left(\revision{\Dxm} u_{i,j,k},\revision{\Dym} u_{i,j,k},\revision{\Dtm} u_{i,j,k}\right) := \\
&\left(\frac{u_{i,,k}-u_{i-1,j,k}}{h_x},\frac{u_{i,j,k}-u_{i,j-1,k}}{h_y},\frac{u_{i,j,k}-u_{i,j,k-1}}{h_t}\right),
\\
\revision{\Dc u \coloneqq} & \revision{\frac{1}{2}\left(\Dp u + \Dm u\right)}
\end{align} for $0\leq i\pm 1, j\pm 1\leq N$.
\revision{In some places we will restrict the derivatives to the spatial directions~$x$ and~$y$,
and denote~$\Dp[\bm{x}] = (\Dxp, \Dyp)$ and~$\Dm[\bm{x}] = (\Dxm, \Dym)$}.
As a result, we can define the following discrete inner product according to these gradient values: 
\begin{align}
\label{eq:H1}
(u,v)_{\nabla,h} &:= 
(\revision{\Dxp}u,\revision{\Dxp}v)_{h} + 
(\revision{\Dyp}u,\revision{\Dyp}v)_{h}
\\
&\phantom{:}= (\revision{\Dxm}u,\revision{\Dxm}v)_{h} + 
(\revision{\Dym}u,\revision{\Dym}v)_{h},
\end{align} with the corresponding induced norm
\begin{equation}
\label{eq:H1norm}
\|u\|_{\nabla,h}^2 :=(u,u)_{\nabla,h} = \Vert \revision{\Dxp} u\Vert_{h}^2 + \Vert \revision{\Dyp} u\Vert_{h}^2
\end{equation} 

We rewrite now the weak formulation into the discrete weak form. We start from the discrete strong form.

\begin{equation}
  \revision{\Dtc}u +  \beta_x  \revision{\Dxc}u + \beta_y \revision{\Dyc} u - \epsilon\Delta_h u = f,
\end{equation} where $f,\beta_x,\beta_y \in \revision{\Dh}$ are given source and coefficient functions, and $\Delta_h$ is the 2D discrete Laplacian defined pointwise as:
\begin{equation}
    \Delta_h u_{i,j,k}:=\frac{u_{i+1,j,k}-2u_{i,j,k}+u_{i-1,j,k}}{{h_x}^2}+\frac{u_{i,j+1,k}-2u_{i,j,k}+u_{i,j-1,k}}{{h_y}^2}
\end{equation} 
Testing with Kronecker Deltas:
$v \in D_{0,h}$,
\begin{equation}
\begin{aligned}
( \revision{\Dtc}u +\beta_x \revision{\Dxc}u + \beta_y \revision{\Dyc}u 
-\epsilon \Delta_h u,v)_{h} =
(f,v)_{h},
\end{aligned}
\label{eq:weak}
\end{equation} 
We perform discrete integration by parts to the term with
the Laplacian,
\revision{noting that~$\Delta_h = \Dm[\bm{x}] \circ \Dp[\bm{x}]$}, to obtain the following discrete weak variational reformulation: find $u\in \revision{\Dzh^0}$ such that:
\begin{equation}
\begin{aligned}\label{eq:discweak}
\overbrace{(\revision{\Dtc u}+\beta_x \revision{\Dxc u} + \beta_y \revision{\Dyc u}, v)_{h} +
\epsilon(\revision{\Dp[\bm{x}]} u, \revision{\Dp[\bm{x}]}v)_{h}}^{b(u,v)} 
= \overbrace{(f,v)_{h}}^{l(v)},\qquad \forall v\in D_{0,h},
\end{aligned}
\end{equation} where $b$ and $l$ are the corresponding discrete bilinear and linear forms over \revision{$\Dzh^0 \times \Dzh$} and $D_{0,h}$, respectively,
\revision{and~$\Dp[\bm{x}] = (\Dxp, \Dyp)$ denotes the spatial gradient}.

Following \cite{LOS2025107839} we construct the robust collocation-based  loss function as:
\begin{equation}
\text{RES}(u) = \{b(u, r(u)_{i,j})-l(r(u)_{i,j})\}_{0 < i< N_x,0 < j< N_y 0 < k< N_t}
\end{equation}
\begin{equation}
\begin{aligned}
  {\cal L}_{crvpinn}=  \text{RES}(u)^T\; \mathbf{G}^{-1} \; \text{RES}(u).
\end{aligned}
\end{equation}
where~$\mathbf{G}$ is the Gram matrix of the inner product $(\cdot,\cdot)_{\nabla,h}$. 


\section{\revision{Error Analysis and Residual Minimization}}

In this section, we provide the detailed mathematical justification for minimizing the weak residual norm. We establish that the norm of the Riesz representative of the residual is spectrally equivalent to the true error norm for the Space-Time Convection-Diffusion equation.

\revision{To carry out the proof, we need to establish some facts about the discrete calculus
introduced above first.}
\revision{
Let~$\tau_\alpha^{\pm}$ denote the shift operator in direction~$\alpha$,
that is
\begin{equation}
    \left(\Txp u\right)_{ijk} =
    \begin{cases}
    u_{i+1,j,k} & \text{for }i < N_x \\
    0 & \text{for }i = N_x
    \end{cases}
    \quad
    \left(\Txm u\right)_{ijk} = 
    \begin{cases}
    u_{i-1,j,k} & \text{for }i > 0 \\
    0 & \text{for }i = 0
    \end{cases},
\end{equation}
and similarly for~$\tau_y^\pm$ and~$\tau_t^\pm$.}
\revision{
\begin{lemma}
\label{lem:shifts-are-adjoin}
Shift operators~$\Tp{\alpha}$ and~$\Tm{\alpha}$ are adjoint,
that is~$\Prodh{\Tp{\alpha} u}{v} = \Prodh{u}{\Tm{\alpha} v}$
for all~$u, v \in \Dh$.
\end{lemma}
\begin{proof}
    We will prove it for~$\alpha = x$; for other directions, the proof is the same.
    Since the scalar product is bilinear, it is enough to consider the case
    when~$u$ is a basis function, 
    that is, $u = \delta_{ijk}$ for some~$0 \leq i, j, k \leq N$.
    If~$i = 0$, then~$\Txp \delta_{ijk} \equiv 0$, so~$\Prodh{\Txp \delta_{ijk}}{v} = 0$.
    On the other hand, $\left(\Txm v\right)_{ijk} = 0$,
    and consequently~$\Prodh{\delta_{ijk}}{\Txp v} = 0$ as well.
    For~$i \neq 0$, function~$\Txp \delta_{ijk}$ is nonzero only at the point
    corresponding to~$(i - 1, j, k)$,
    so
    \begin{equation}
        \frac{1}{h_xh_yh_t}\Prodh{\Txp \delta_{ijk}}{v} = v_{i-1,j,k} = \,\left(\Txm v\right)_{ijk}
        = \frac{1}{h_xh_yh_t}\Prodh{\delta_{ijk}}{\Txm v}
    \end{equation}
    which completes the proof.
\end{proof}
}

\revision{
\begin{lemma}
\label{lem:shift-flips-gradient-polarity}
$\Tm{\alpha} \circ \Dp[\alpha] = \Dm[\alpha]$.
\end{lemma}%
\begin{proof}
    Let~$u \in \Dh$, then for~$i > 0$ we have
    \begin{equation}
        \left(\Txm \Dxp u\right)_{ijk} = \left(\Dxp u\right)_{i-1,j,k} = h_x^{-1}(u_{ijk} - u_{i-1,j,k})
        = \left(\Dxm u\right)_{ijk}.
    \end{equation}
    For~$i = 0$, both sides are zero.
    Proof for other directions is the same.
\end{proof}
}

\revision{
\begin{lemma}[Discrete product rule]
\label{lem:discrete-product-rule}
Any pair~$u,v \in \Dh$ satisfies
\begin{equation}
\begin{aligned}
  \Dp[\alpha] (u v) &= \left(\Dp[\alpha] u\right)\,v + \Tp{\alpha}u\left(\Dp[\alpha] v\right) \\
  \Dm[\alpha] (u v) &= \left(\Dm[\alpha] u\right)\,v + \Tm{\alpha}u\left(\Dm[\alpha] v\right).
\end{aligned}
\end{equation}
\end{lemma}
\begin{proof}
    For~$i < N$ we have
    \begin{equation}
    \begin{aligned}
        \Dxp(u v)_{ijk} &= u_{i+1,j,k} v_{i+1,j,k} - u_{ijk}v_{ijk} \\&=
        u_{i+1,j,k}v_{i+1,j,k} - u_{i+1,j,k}v_{ijk} + u_{i+1,j,k}v_{ijk} - u_{ijk}v_{ijk} \\&=
        u_{i+1,j,k}(\nabla_{x+}v)_{ijk} + (\nabla_{x+}u)_{ijk}v_{ijk} \\&=
        \left(\Txp u\right)_{ijk}(\nabla_{x+}v)_{ijk} + (\nabla_{x+}u)_{ijk}v_{ijk}.
    \end{aligned}
    \end{equation}
    For~$i = N$, both sides are zero.
    Proofs for other directions and for~$\Dm[\alpha]$ are the same.
\end{proof}
}

\revision{
\begin{lemma}
\label{lem:zero-gradient-integral}
Any~$u \in \Dzh$ satisfies~$\Prodh{\Dp[\alpha] u}{1} = 0$ for~$\alpha \in \{x, y\}$.
\end{lemma}
\begin{proof}
    \begin{equation}
    \begin{aligned}
        \Prodh{\Dxp u}{1} &= h_x h_y h_t \sum_{(\bm{x}, t) \in \Omega_h \times I_h} (\Dxp u)(\bm{x}, t) \\
        &= h_y h_t\sum_{j,k = 0}^N \left\{ \sum_{i = 0}^N \left[u_{i+1,j,k} - u_{ijk}\right] \right\} \\
        &= h_y h_t\sum_{j,k = 0}^N \left[ u_{Njk} - u_{0jk} \right] = 0,
    \end{aligned}
    \end{equation}
    since~$\left.u\right|_{\Gamma_h} = 0$.
\end{proof}
}

\revision{
\begin{lemma}[Discrete integration by parts]
\label{lem:discrete-by-parts}
Any pair~$u,v\in \Dh$ such that~$uv \in \Dzh$ satisfies
\begin{equation}
    \Prodh{\Dp[\alpha] u}{v} = - \Prodh{u}{\Dm[\alpha] v}
\end{equation}
for~$\alpha \in \{x, y\}$.
\end{lemma}
\begin{proof}
    By~\cref{lem:zero-gradient-integral}, $\Prodh{\Dp[\alpha](uv)}{1} = 0$.
    Combining it with~\cref{lem:discrete-product-rule} we obtain
    \begin{equation}
        \Prodh{\Dp[\alpha] u}{v} + \Prodh{\Tp{\alpha}u}{\Dp[\alpha] v} =
        \Prodh{\Dp[\alpha](uv)}{1} = 0
    \end{equation}
    so~$\Prodh{\Dp[\alpha] u}{v} = - \Prodh{\Tp{\alpha}u}{\Dp[\alpha] v}$.
    By~\cref{lem:shifts-are-adjoin} we get
    $\Prodh{\Tp{\alpha}u}{\Dp[\alpha] v} = \Prodh{u}{\Tm{\alpha}\Dp[\alpha] v}$,
    and finally~\cref{lem:shift-flips-gradient-polarity}
    gives us~$\Prodh{u}{\Tm{\alpha}\Dp[\alpha] v} = \Prodh{u}{\Dm[\alpha] v}$,
    which proves the desired equality.
\end{proof}
}

\revision{
\begin{lemma}[discrete Poincaré inequality]
\label{lem:poincare-inequality}
There exists a constant~$C_p$ independent of the grid size such that
\begin{equation}
    \Normh{u} \leq C_p\Normdh{u}
\end{equation}
for every~$u \in \Dzh$.
\end{lemma}
\begin{proof}
    Let~$g \in \Dh$ be given by~$g_{ijk} = ih$.
    Then~$\left(\Dxm g\right)_{ijk} = 1$ unless~$i = N$.
    Therefore, by~\cref{lem:discrete-by-parts}
    \begin{equation}
        \Normh{u}^2 = \Prodh{\Dxm g}{u^2} = \Prodh{g}{\Dxp (u^2)}.
    \end{equation}
    Using~\cref{lem:discrete-product-rule} we get
    \begin{equation}
        \Dxp u^2 = (u + \Txp u) \Dxp u.
    \end{equation}
    Since~$\Normh{\Txp u} = \Normh{u}$,
    using Cauchy-Schwarz inequality we obtain
    \begin{equation}
        \Normh{u}^2 = \Prodh{g}{(u + \Txp u) \Dxp u} \leq 2 \Normh{g} \Normh{u} \Normh{\Dxp u}
    \end{equation}
    and thus~$\Normh{u} \leq 2 \Normh{g} \Normh{\Dxp} \leq 2 \Normdh{u}$,
    since~$\Normh{g} \leq \Normh{1} = 1$.
    \qed
\end{proof}
}

\revision{
\begin{lemma}
\label{lem:positive-product}
For any~$u \in \Dzh^0$ we have~$\Prodh{\Dtc u}{u} \geq 0$.
\end{lemma}
\begin{proof}
    By~\cref{lem:discrete-product-rule}, we have
    \begin{equation}
        \Dtp(u^2) = \left(\Dtp u\right) (u + \Ttp u)
    \end{equation}
    so
    \begin{equation}
        \Prodh{\Dtp (u^2)}{1} = \Prodh{\Dtp u}{u} + \Prodh{\Dtp u}{ \Ttp u}
    \end{equation}
    Using~\cref{lem:shifts-are-adjoin} and~\cref{lem:shift-flips-gradient-polarity} we get
    \begin{equation}
        \Prodh{\Dtp u}{\Ttp u} = \Prodh{\Ttm \Dtp u}{u} = \Prodh{\Dtm u}{u}
    \end{equation}
    so that
    \begin{equation}
    \begin{aligned}
        \Prodh{\Dtp (u^2)}{1} &= \Prodh{\Dtp u}{u} + \Prodh{\Dtm u}{u} = 2 \Prodh{\Dtc u}{u}
    \end{aligned}
    \end{equation}
    But
    \begin{equation}
        \Prodh{\Dtp (u^2)}{1} = h_x h_y \sum_{\bm{x} \in \Omega_h}\left[u^2(\bm{x}, T)^2 - u^2(\bm{x}, 0)\right] \geq 0
    \end{equation}
    since~$u(\bm{x}, 0) = 0$.
\end{proof}
}

\revision{
\textbf{Functional Spaces and Problem Setting.}
Let us define the discrete trial and test spaces as:
\begin{equation}
    \mathbb{U} = \Dzh^0,\quad \mathbb{V} = \Dzh
\end{equation}
with norms
\begin{equation}
\begin{aligned}
    \NormV{v}^2 &= \Normh{\Dp[\bm{x}] v}^2 \coloneqq \Normh{\Dxp v}^2 + \Normh{\Dyp v}^2 \\
    \NormU{u}^2 &= \Norm{\mathbb{V}'}{\Dtc u}^2 + \Normh{\Dp[\bm{x}] u}^2
\end{aligned}
\end{equation}
The variational problem is to find~$u^\star \in \mathbb{U}$ such that~$b(u^\star, v) = l(v)$ for all~$v \in \mathbb{V}$.
}

\textbf{Riesz Representation of the Residual.} Let $u_\theta \in \mathbb{U}$ be an approximate solution. The weak residual functional $\mathcal{R}(u_\theta) \in \mathbb{V}'$ is defined by $\langle \mathcal{R}(u_\theta), v \rangle = b(u_\theta, v) - l(v)$. By the Riesz Representation Theorem, there exists a unique $r_\theta \in \mathbb{V}$ such that $(r_\theta, v)_{\mathbb{V}} = b(u_\theta, v) - l(v)$ for all $v \in \mathbb{V}$. Since $b(u^*, v) = l(v)$, this implies the fundamental error identity:
\begin{equation}
    (r_\theta, v)_{\mathbb{V}} = b(u_\theta - u^*, v) = b(e, v), \quad \forall v \in \mathbb{V}, \label{eq:identity}
\end{equation}
where $e = u_\theta - u^*$ is the approximation error.

\textbf{Boundedness and Stability.} To relate $\|r_\theta\|_{\mathbb{V}}$ to $\|e\|_{\mathbb{U}}$, we prove that $B$ is continuous and coercive (inf-sup stable).
\revision{To prove continuity, let us bound each term individually:}
\revision{
\begin{equation}
    b(u, v) = \underbrace{\Prodh{\Dtc u}{v}}_{I_1} +
    \underbrace{\Prodh{\bm{\beta} \cdot \Dc[\bm{x}] u}{v}}_{I_2} +
    \underbrace{\epsilon \Prodh{\Dp[\bm{x}] u}{\Dp[\bm{x}] v}}_{I_3}
\end{equation}
Using Cauchy-Schwartz and~\cref{lem:poincare-inequality}, we obtain
\begin{equation}
    \left| I_1 \right| \leq \Normh{\Dtc u} \Normh{v} \leq C_p\NormU{u}\NormV{v}
\end{equation}
since~$\Normh{v} \leq C_p\NormV{v}$.
For~$I_2$, we have
\begin{equation}
    \left| I_2 \right| \leq \Norm{\infty}{\bm{\beta}} \Normh{\Dc[\bm{x}]u} \Normh{v}
    \leq C_p\Norm{\infty}{\bm{\beta}} \NormU{u} \NormV{v}
\end{equation}
since
\begin{equation}
    2 \Normh{\Dc[\bm{x}]u} = \Normh{\Dp[\bm{x}] u + \Dm[\bm{x}] u} \leq
    \Normh{\Dp[\bm{x}] u} + \Normh{\Dm[\bm{x}] u} = 2 \Normh{\Dp[\bm{x}] u}
\end{equation}
Finally,
\begin{equation}
    \left| I_3 \right| \leq \epsilon \Normh{\Dp[\bm{x}] u} \Normh{\Dp[\bm{x}] v}
    \leq \epsilon \NormU{u} \NormV{v}
\end{equation}
Putting all the estimates together, we are left with
\begin{equation}
b(u, v) \leq (C_p + C_p\Norm{\infty}{\bm{\beta}} + \epsilon)\NormU{u}\NormV{v}
\end{equation}
which proves continuity of~$b$ with a constant~$M = C_P(1 + C_b) + \epsilon$.
}

\revision{
To prove inf-sup stability, we need to show that there exists~$\gamma > 0$ such that
$\sup_{v \in \mathbb{V}} \frac{b(u, v)}{\NormV{v}} \geq \gamma \NormU{u}$
for all~$u \in \mathbb{U}$.
This is equivalent to showing the upper bounds~$\Normh{\Dp[\bm{x}] u}, \Normh{\Dtc u} \lesssim \sup_{v \in \mathbb{V}} \frac{b(u, v)}{\NormV{v}}$.
For the first bound, we can take~$v = u$ to get
\begin{equation}
    b(u, u) = \Prodh{\Dtc u}{u} + \Prodh{\bm{\beta} \cdot \Dc[\bm{x}] u}{u} + \epsilon \Normh{\Dp[\bm{x}] u} ^2
\end{equation}
The first term is non-negative by~\cref{lem:positive-product}.
The second one vanishes, since
\begin{equation}
\begin{aligned}
    \Prodh{\beta_x \Dxp u}{u} &= \Prodh{\Dxp u}{\beta_x u} \\&=
    -\Prodh{u}{\Dxm \left(\beta_x u\right)} \\&=
    -\Prodh{u}{\left(\Dxm u\right)\,\beta_x + \Txm u\left(\Dxm \beta_x\right)} \\&=
    -\Prodh{\beta_x\Dxm u}{u},
\end{aligned}
\end{equation}
so
\begin{equation}
    \Prodh{\beta_x \Dxc u}{u} = 
    \frac{1}{2}
    \left[
        \Prodh{\beta_x \Dxp u}{u} + \Prodh{\beta_x\Dxm u}{u}
    \right] = 0.
\end{equation}
and the same applies to~$\Prodh{\beta_y \Dyc u}{u} = 0$.
This leaves us with
\begin{equation}
    b(u, u) \geq \epsilon \Normh{\Dp[\bm{x}] u}^2
\end{equation}
Since~$\NormV{u} = \Normh{\Dp[\bm{x}] u}$, we have~$\sup_{v \in \mathbb{V}} \frac{b(u, v)}{\NormV{v}} \geq \frac{b(u, u)}{\NormV{u}} \geq \epsilon \Normh{\Dp[\bm{x}] u}$.
}

\revision{
For the second bound,
\begin{equation}
\begin{aligned}
    \Norm{\mathbb{V}'}{\Dtc u} &= \sup_{v \in \mathbb{V}} \frac{\Prodh{\Dtc u}{v}}{\NormV{v}}
    = \sup_{v \in \mathbb{V}} \frac{b(u, v) - \Prodh{\bm{\beta} \cdot \Dc[\bm{x}] u}{v} +
    \epsilon \Prodh{\Dp[\bm{x}] u}{\Dp[\bm{x}] v}}{\NormV{v}} \\
    &\leq \sup_{v \in \mathbb{V}} \frac{b(u, v)}{\NormV{v}} + C_p(1 + \Norm{\infty}{\bm{\beta}}) \Normh{\Dp[\bm{x}] u}
\end{aligned}
\end{equation}
and since~$\Normh{\Dp[\bm{x}] u} \lesssim \sup_{v \in \mathbb{V}} \frac{b(u, v)}{\NormV{v}}$, this proves that~$\Norm{\mathbb{V}'}{\Dtc u} \lesssim \sup_{v \in \mathbb{V}} \frac{b(u, v)}{\NormV{v}}$.
Combining these yields the stability condition $\|u\|_{\mathbb{U}} \le \frac{1}{\gamma} \sup_{v} \frac{b(u,v)}{\|v\|_{\mathbb{V}}}$.
}
Then, we get
\begin{equation}
    \|r_\theta\|_{\mathbb{V}}^2 = (r_\theta, r_\theta)_{\mathbb{V}} = b(e, r_\theta) \leq M \|e\|_{\mathbb{U}} \|r_\theta\|_{\mathbb{V}} \implies \|r_\theta\|_{\mathbb{V}} \leq M \|e\|_{\mathbb{U}}.
\end{equation}
For the \textit{Lower Bound}, we use the Inf-Sup stability of $b$. There exists a $v \in \mathbb{V}$ such that $\gamma \|e\|_{\mathbb{U}} \|v\|_{\mathbb{V}} \leq b(e, v)$. Substituting $b(e,v) = (r_\theta, v)_{\mathbb{V}}$ and applying Cauchy-Schwarz:
\begin{equation}
    \gamma \|e\|_{\mathbb{U}} \|v\|_{\mathbb{V}} \leq (r_\theta, v)_{\mathbb{V}} \leq \|r_\theta\|_{\mathbb{V}} \|v\|_{\mathbb{V}} \implies \|e\|_{\mathbb{U}} \leq \frac{1}{\gamma} \|r_\theta\|_{\mathbb{V}}.
\end{equation}

In conclusion, the derived bounds $\frac{1}{M} \|r_\theta\|_{\mathbb{V}} \leq \|u_\theta - u^*\|_{\mathbb{U}} \leq \frac{1}{\gamma} \|r_\theta\|_{\mathbb{V}}$ confirm that the Riesz residual norm is an equivalent metric to the true error.
The norm of the residual is given by
\begin{equation}
    \|r_M(u_\theta)\|_{\mathbb{V}}^2 = \mathbf{r}(u_\theta)^T \mathbf{G}^{-1} \mathbf{r}(u_\theta) =: \|\mathbf{r}(u_\theta)\|_{\mathbf{G}^{-1}}^2,
\end{equation}
where:
\begin{itemize}
    \item $\mathbf{r}(u_\theta) = [B(u_\theta, v_m) - l(v_m)]_{m=1}^M$ is the vector of weak residuals.
    \item $\mathbf{G} = [(v_m, v_n)_{\mathbb{V}}]_{mn}$ is the space-time Gram matrix with entries defined by the test space inner product:
    \begin{equation}
        \mathbf{G}_{mn} = \Prodh{\Dp[\bm{x}]v_n}{\Dp[\bm{x}] v_m}
    \end{equation}
\end{itemize}

As a result, we end up with a discretized residual minimization scheme given by
\begin{equation}
    \min_{u_\theta \in \mathbb{U}} \|\mathbf{r}(u_\theta)\|_{\mathbf{G}^{-1}}^2.
\end{equation}

\section{\revision{Numerical verification of the method}}

\subsection{\revision{Model problem with manufactured solution}}

\revision{To verify the accuracy of the proposed numerical method, we consider a transient advection--diffusion equation for which an analytical solution is prescribed using the method of manufactured solutions. The governing equation is defined on the unit square
$
\Omega=(0,1)^2, t\in(0,T],
$
and reads
\begin{equation}
\frac{\partial u}{\partial t}
+
\frac{\partial u}{\partial x}
+
\frac{\partial u}{\partial y}
-
{\epsilon}\Delta u
=
f(x,y,t).
\label{eq:model_problem}
\end{equation}
The problem is supplemented with homogeneous Dirichlet boundary conditions,
\begin{equation}
u(x,y,t)=0,
\;
(x,y)\in\partial\Omega,\quad t>0,
\label{eq:dirichlet_bc}
\end{equation}
and the initial condition
\begin{equation}
u(x,y,0)
=
\phi(x)\sin(\pi y).
\label{eq:initial_condition}
\end{equation}}

\revision{The exact solution is prescribed as
\begin{equation}
u_{\mathrm{exact}}(x,y,t)
=\phi(x) \sin(\pi y) e^{-t}
\label{eq:manufactured_solution2}
\end{equation}
where 
\begin{equation}
\phi(x)=\frac{x(1-e^{\frac{1}{\epsilon}})+e^{\frac{x}{\epsilon}}-1}{e^{\frac{x}{\epsilon}}-1}
\label{eq:phi}
\end{equation}
The exact solution satisfies both the initial condition and the homogeneous zero Dirichlet boundary conditions.
Substituting \eqref{eq:manufactured_solution2} into \eqref{eq:model_problem} yields the forcing term
\begin{equation}
f(x,y,t)
=
e^{- t}\Big[
\sin(\pi y)\left(1+(\epsilon (\pi^2-1) \phi(x)\right)+\pi \cos(\pi y) \phi(x)
\Big].
\label{eq:forcing_term}
\end{equation}
This benchmark is used throughout the numerical experiments to assess the convergence and accuracy of the proposed CRVPINN formulation, since the exact solution is available at each collocation point in space and time.}

\revision{We solve the model problem with manufactured solution on Google Colab A100 GPU.}

\revision{Both PINN and CRVPINN converge. However, there is a discrepancy of one order of magnitude between the true error and the PINN loss, as it is illustrated in Figure \ref{fig:conv_pinn}. Introducing the CRVPINN loss makes it related with the true error, as it is presented in Figure \ref{fig:conv_crvpinn}.}

\revision{We employed three fully connected layers with 64 neurons each, 20000 epochs of training, learning rate $\eta=0.001$. We executed the code on A100 Google compute engine with 83 GB of RAM. The training times for CRVPINN method is the following: 232s for $12\times 12 \times 12$ collocation points, 236 s for $25\times 25\times 25$ points, 378s for $50\times 50\times 50$, and 2446s for $100\times 100\times 100$ collocation points. The training time for PINN method is similar (since the Gram matrix is inverted only once at the beginning). Namely, 228s for $12\times 12 \times 12$ collocation points, 229 s for $25\times 25\times 25$ points, 366s for $50\times 50\times 50$, and 2431s for $100\times 100\times 100$ collocation points. }

\begin{figure}
\includegraphics[width=\textwidth]{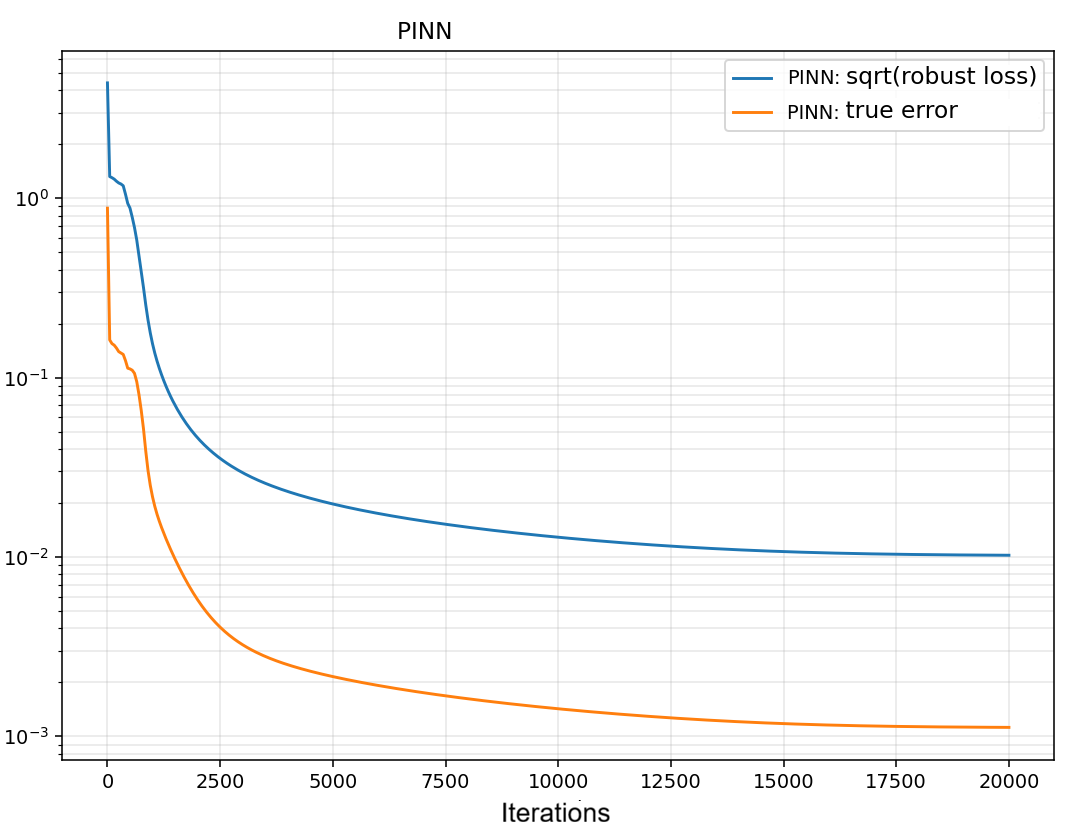}
\caption{Convergence of the PINN training and comparison with the true error.}
\label{fig:conv_pinn}
\end{figure}

\begin{figure}
\includegraphics[width=\textwidth]{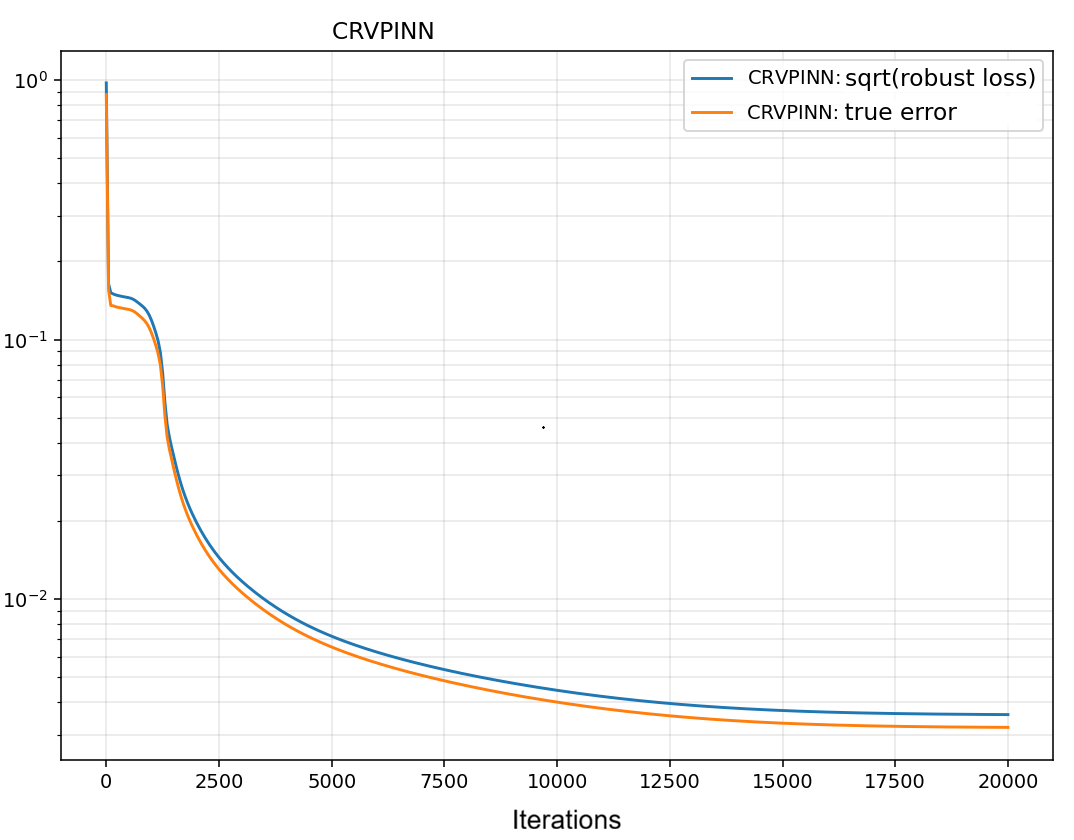}
\caption{Convergence of the CRVPINN training and comparison with the true error.}
\label{fig:conv_crvpinn}
\end{figure}

\revision{We also compare the convergence of the method for increasing number of the collocation points. We run the CRVPINN training for 12, 25, 50, and 100 collocation points in each axis, and we plot the convergence of the true error. This is illustrated in Figure \ref{fig:convergence}. We also present in Figure \ref{fig:ratios} how the loss function converges when we increase the number of collocation points. We can read from these figures that increasing the number of the collocation points results in an increase of the accuracy and better agreement between the loss function and the true error.}

\begin{figure}
\includegraphics[width=\textwidth]{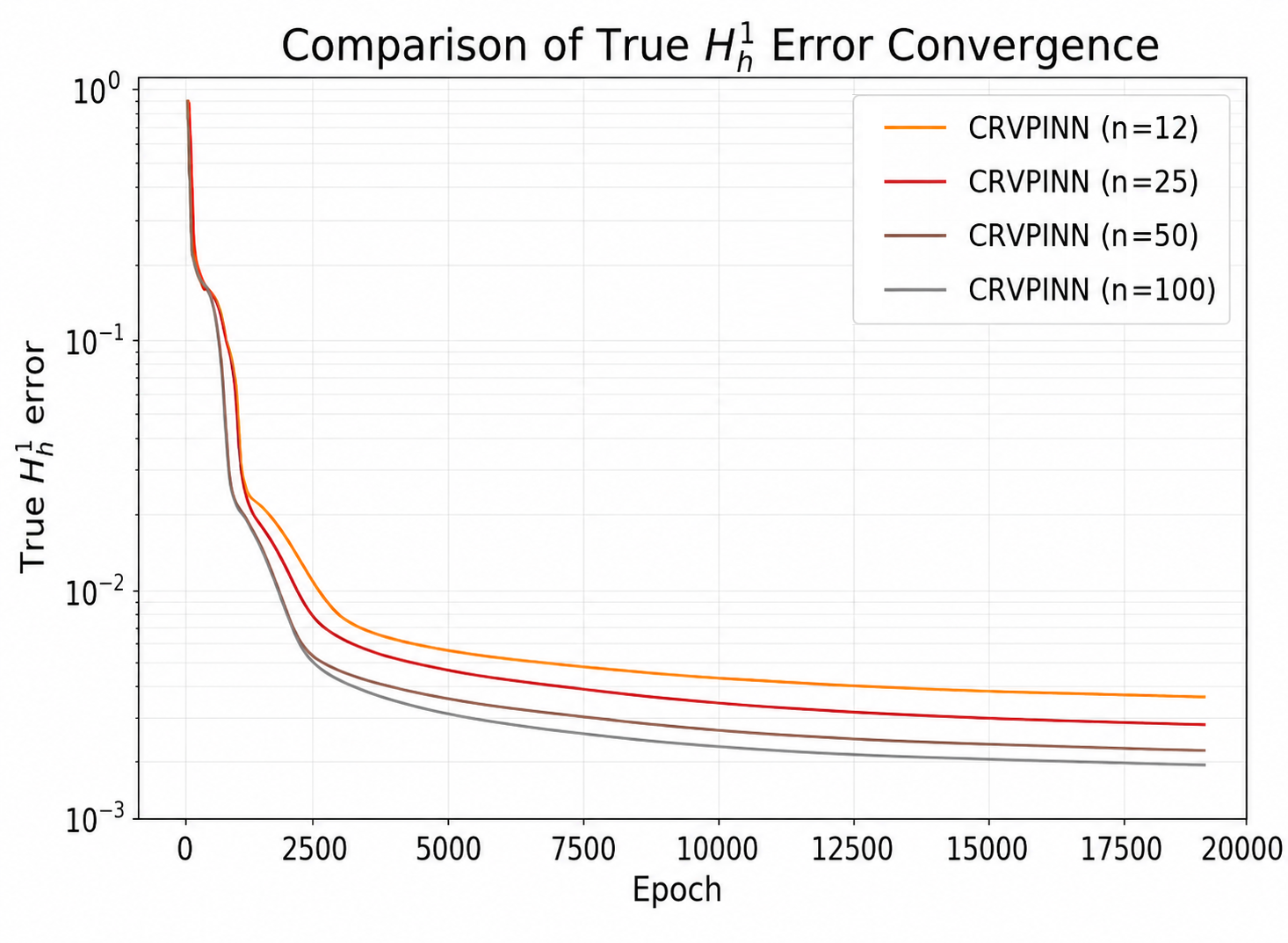}
\caption{\revision{Convergence of the true error for CRVPINN training when increasing the number of the collocation points, from $12\times 12\times 12$, $25\times 25\times 25$, $50\times 50\times 50$, and $100\times 100\times 100$.}}
\label{fig:convergence}
\end{figure}

\begin{figure}
\includegraphics[width=\textwidth]{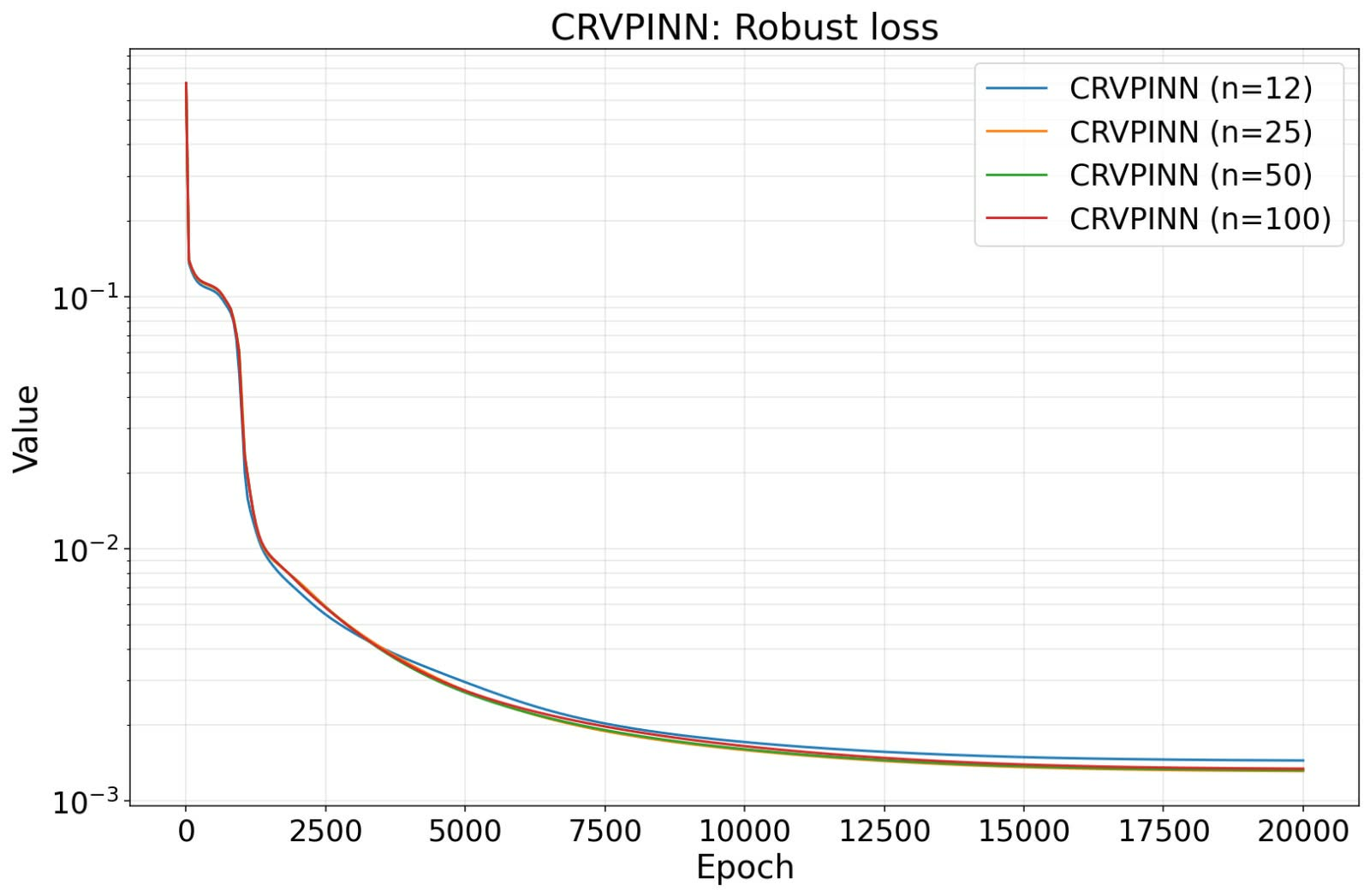}
\caption{\revision{Convergence of the robust loss for CRVPINN training when increasing the number of the collocation points, from $12\times 12\times 12$, $25\times 25\times 25$, $50\times 50\times 50$, and $100\times 100\times 100$.}}
\label{fig:ratios}
\end{figure}

\subsection{\revision{Comparison with IGA-ADS simulation}}

\revision{We will compare our CRPINN solver to the  isogeometric analysis (IGA) alternating direction (ADS) solver we developed in \cite{10.1007/978-3-031-35995-8_13,LOS2024102238}.
The reason for selection of the IGA-ADS solver is that it delivers a linear computational cost solver and smooth, higher-order and continuity approximation of the solution, similarly to the approximation with neural networks.}

\revision{The IGA-ADS employ a sequence of isogeometric $L^2$ projections for the advection-diffusion problem with diffusion $\epsilon=1$ and advection ${\bf b}=(1,1)^T$. We denote $u_t(x,y)=u(x,y,t)$ and employ an explicit time integration scheme with time step size $dt$, 
\begin{eqnarray}
u_{t+1}(x,y) = u_t(x,y)+ dt \left(\frac{\partial u_t(x,y)}{\partial x} + \frac{\partial u_t(x,y)}{\partial y} - \frac{\partial^2  u_t(x,y)}{\partial x^2}- \frac{\partial^2  u_t(x,y)}{\partial y^2} +  f(x,y,t)\right)
\end{eqnarray}
We obtain the weak formulation by testing with B-spline basis functions
\begin{equation}\begin{aligned}
\left(u_{t+1}(x,y),v(x,y)\right) =  \left(u_t(x,y),v(x,y)\right)+  dt \left(\frac{\partial u_t(x,y)}{\partial x}+\frac{\partial u_t(x,y)}{\partial y},v(x,y)\right) \\ + dt\left(\frac{\partial^2  u_t(x,y)}{\partial x^2} + \frac{\partial^2  u_t(x,y)}{\partial y^2},v(x,y)\right)+\left(f(x,y),v(x,y)\right)
\end{aligned}
\end{equation}
We discretize with B-spline basis functions
\begin{eqnarray}
u^{t+1}=\sum_{i=1,...,N_x;j=1,...,N_y} u^{t+1}_{ij} B^x_i B^y_j \inblue{,} \notag \\
u^{t}=\sum_{i=1,...,N_x;j=1,...,N_y} u^{t}_{ij} B^x_i B^y_j \inblue{,}
\end{eqnarray}
and we test with B-spline basis functions
\begin{eqnarray}
\begin{aligned}
\sum_{ij}u^{t+1}_{ij}\left(B^x_iB^y_j,B^x_mB^y_n\right) = & \sum_{ij}u^t_{ij}\left(B^x_iB^y_j,B^x_mB^y_n\right) + \\
& dt \sum_{ij}u^t_{ij}\left(B^y_j\frac{\partial B^x_i}{\partial x},B^x_mB^y_n\right)+
& dt \sum_{ij}u^t_{ij}\left(B^x_i\frac{\partial B^y_j}{\partial y},B^x_mB^y_n\right)+
\\ &  dt \sum_{ij}u^t_{ij}\left(  \frac{\partial B^x_i}{\partial x}B^y_j,\frac{\partial B^x_m}{\partial x}B^y_n \right)+   \\ 
& dt \sum_{ij}u^t_{ij}\left(   B^x_i\frac{\partial B^y_j}{\partial y},B^x_m\frac{\partial B^y_n}{\partial y}\right)  \\
&\left(f(x,y,t), B^x_m B^y_n\right)\\ 
&\quad m=1,...,N_x;n=1,...,N_y\inblue{,}
\label{eq:system}
\end{aligned}
\end{eqnarray}
where $(u,v)=\int_{\Omega} u(x,y)v(x,y)dxdy$. 
We separate directions
\begin{eqnarray}
\begin{aligned}
\sum_{ij}u^{t+1}_{ij}\left(B^x_i,B^x_m\right)_x\left(B^y_j,B^y_n\right)_y = & \sum_{ij}u^t_{ij}\left(B^x_i,B^x_m\right)_x\left(B^y_j,B^y_n\right)_y + \\
& dt \sum_{ij}u^t_{ij}\left( \frac{\partial B^x_i}{\partial y} ,B^x_m\right)_x\left( B^y_j,B^y_n\right)_y +\\
& dt \sum_{ij}u^t_{ij}\left( B^x_i,B^x_m\right)_x \left( \frac{\partial B^y_j}{\partial y} ,B^y_n\right)_y+
\\ & dt \sum_{ij}u^t_{ij}\left(  \frac{\partial B^x_i}{\partial x},\frac{\partial B^x_m}{\partial x} \right)_x\left(  B^y_j,B^y_n \right)_y+   \\ 
& dt \sum_{ij}u^t_{ij}\left(   B^x_i,B^x_m\right)_x
\left(   \frac{\partial B^y_j}{\partial y},\frac{\partial B^y_n}{\partial y}\right)_y\\ 
&\left(f(x,y,t), B^x_m B^y_n\right)\\ 
&\quad m=1,...,N_x;n=1,...,N_y\inblue{,}
\label{eq:system}
\end{aligned}
\end{eqnarray}
We introduce 
\begin{eqnarray}
{\bf M}_x=\{ (B^x_i,B^x_m)_x \}_{im}=\{\int B^x_iB^x_l dx \}_{im}, \nonumber \\
{\bf M}_y=\{ (B^y_j,B^y_n)_y \}_{jn}=\{\int B^y_jB^y_n dy \}_{jn}, \nonumber \\  
{\bf A}_x=\{ (\frac{\partial B^x_i}{\partial x},B^x_m)_x \}_{im}=\{\int \frac{\partial B^x_i}{\partial x}B^x_m dx \}_{im}, \\
{\bf A}_y=\{ (B^y_j,\frac{\partial B^y_n}{\partial y})_y \}_{jn}=\{\int B^y_j \frac{\partial B^y_n}{\partial y}dy \}_{jn} \nonumber \\
{\bf S}_x=\{ (\frac{\partial B^x_i}{\partial x},\frac{\partial B^x_m}{\partial x})_x \}_{im}=\{\int \frac{\partial B^x_i}{\partial x}\frac{\partial B^x_m}{\partial x} dx \}_{im} \nonumber \\
{\bf S}_y=\{ (\frac{\partial B^y_j}{\partial y},\frac{\partial B^y_n}{\partial y})_y \}_{jn}=\{\int \frac{\partial B^y_j}{\partial y} \frac{\partial B^y_n}{\partial y}dy \}_{jn} \notag\end{eqnarray}
We rewrite our formulation using the Kronecker product property of matrices $\mathcal{M} = \mathcal{A}_x \otimes \mathcal{B}_y $ where $\mathcal{M}_{ijmn} = {\mathcal{A}_x}_{im}  {\mathcal{B}_y}_{jn}$ over over a 2D domain $\Omega = \Omega_x \times \Omega_y $
\begin{eqnarray}
\begin{aligned}
{\bf M}_x \otimes {\bf M}_y {\bf u}^{t+1} = & {\bf M}_x\otimes {\bf M}_y {\bf u}^t - 
dt {\bf A}_x \otimes {\bf M}_y {\bf u}^t -
dt {\bf M}_x \otimes {\bf A}_y {\bf u}^t +  
 \\
&  dt  {\bf S}_x\otimes {\bf M}_y {\bf u}^t 
+ dt {\bf M}_x \otimes {\bf S}_y {\bf u}^t  \end{aligned}
\label{eq:explicit}
\end{eqnarray}
We run the IGA-ADS solver for the problem (\ref{eq:explicit}) in a loop, with the time step size limited by the CFL condition. The solver employs the fact that the inverse of the Kronecker product matrix is a Kronecker product of inverses of the component matrices, resulting in a linear computational cost of generation and solution of a single time step.
}

\revision{
We use the computational mesh with $32 \times 32$ finite elements with quadratic B-splines of $C^1$ continuity. 
We test our method on a benchmark advection-diffusion problem with manufactured solution~$u(x, y, t) = e^{-t}\sin(\pi x) \sin(\pi y)$
with diffusion $\epsilon = 1$, advection vector $\mathbf{b} = (1, 1)^T$, and the forcing~$f$ chosen
to match the desired exact solution.
The explicit method requires 100,000 time iterations to simulate the time interval $[0,1]$ in stable way. The difference between the exact and numerical solution computed as
$\| u_{exact}-u_h \|_{\mathbb{V}}$ is 0.0012, similar to the accuracy of CRVPINN.
The total computing time is around 1 minute on 11th Gen Intel Core i7 laptop with 32 GB of RAM.}

\begin{figure}
    \centering
    \includegraphics[width=0.3\linewidth]{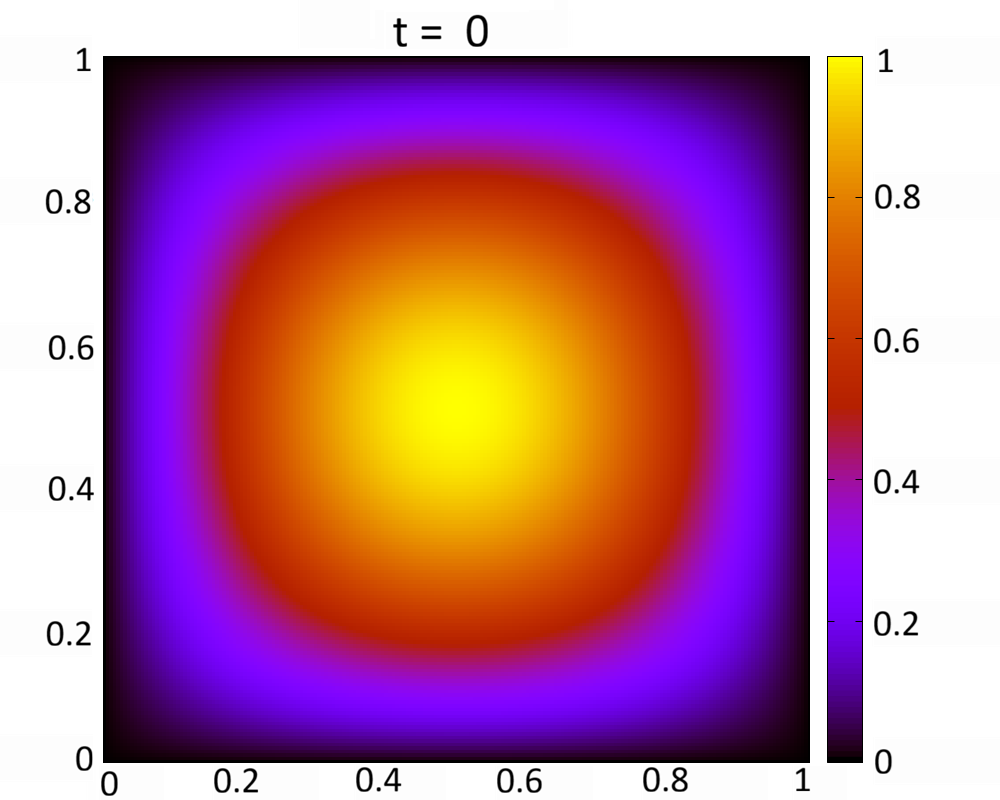}    \includegraphics[width=0.3\linewidth]{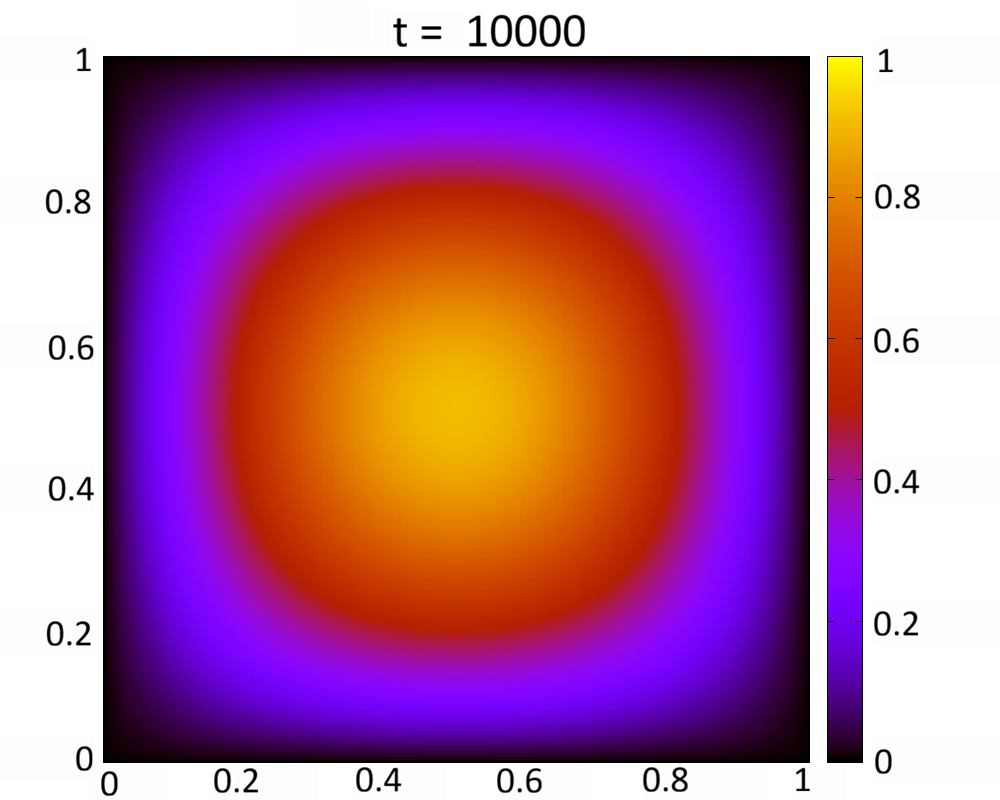}     \includegraphics[width=0.3\linewidth]{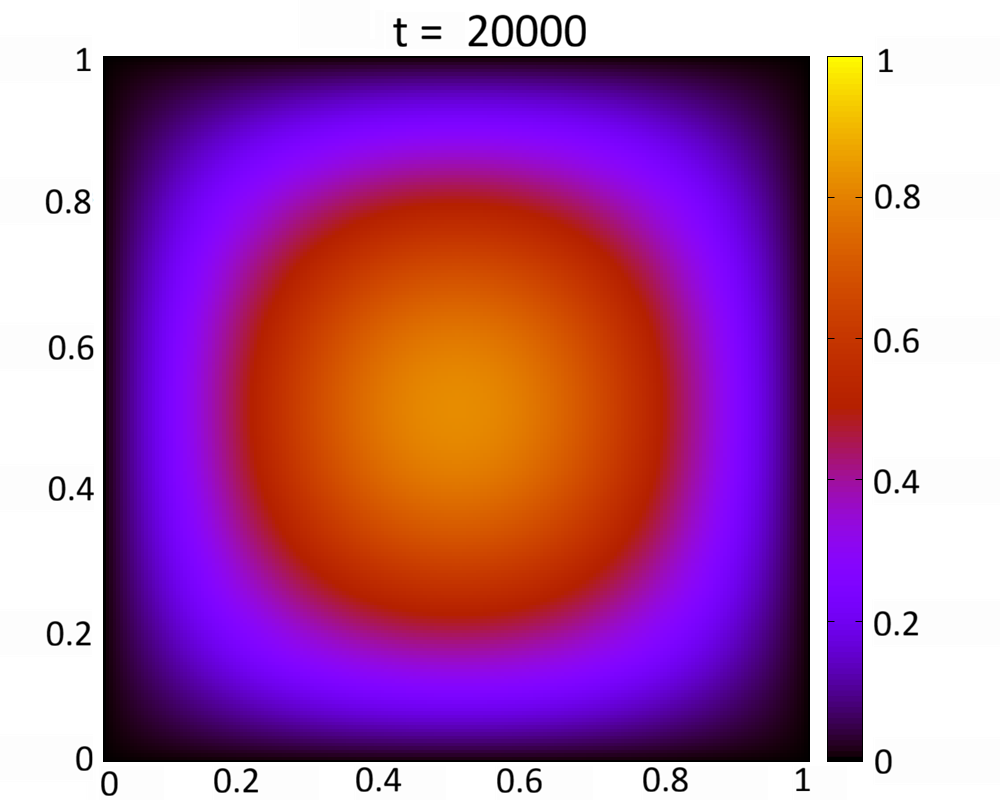} \\
    \includegraphics[width=0.3\linewidth]{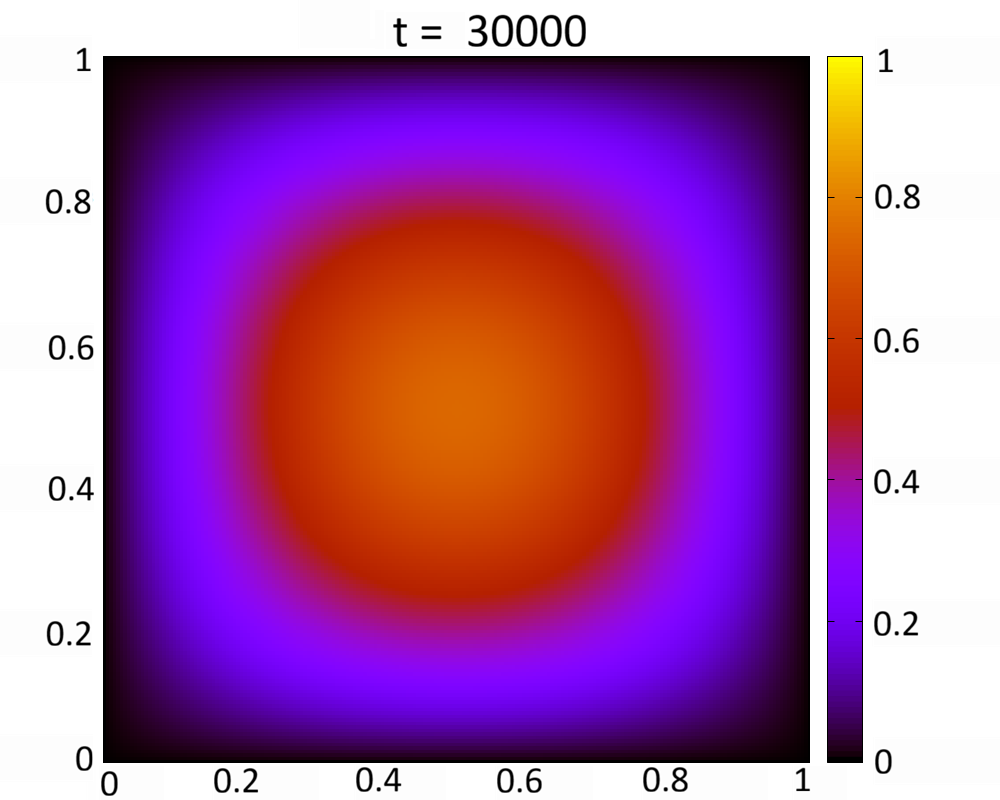}    \includegraphics[width=0.3\linewidth]{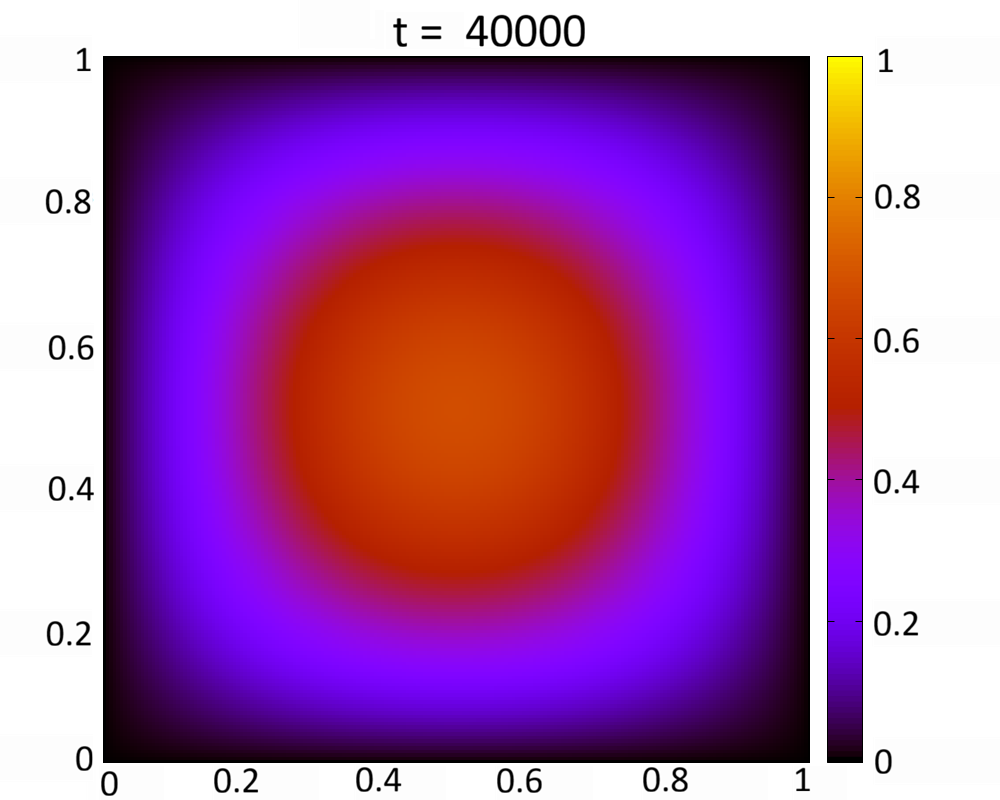}     \includegraphics[width=0.3\linewidth]{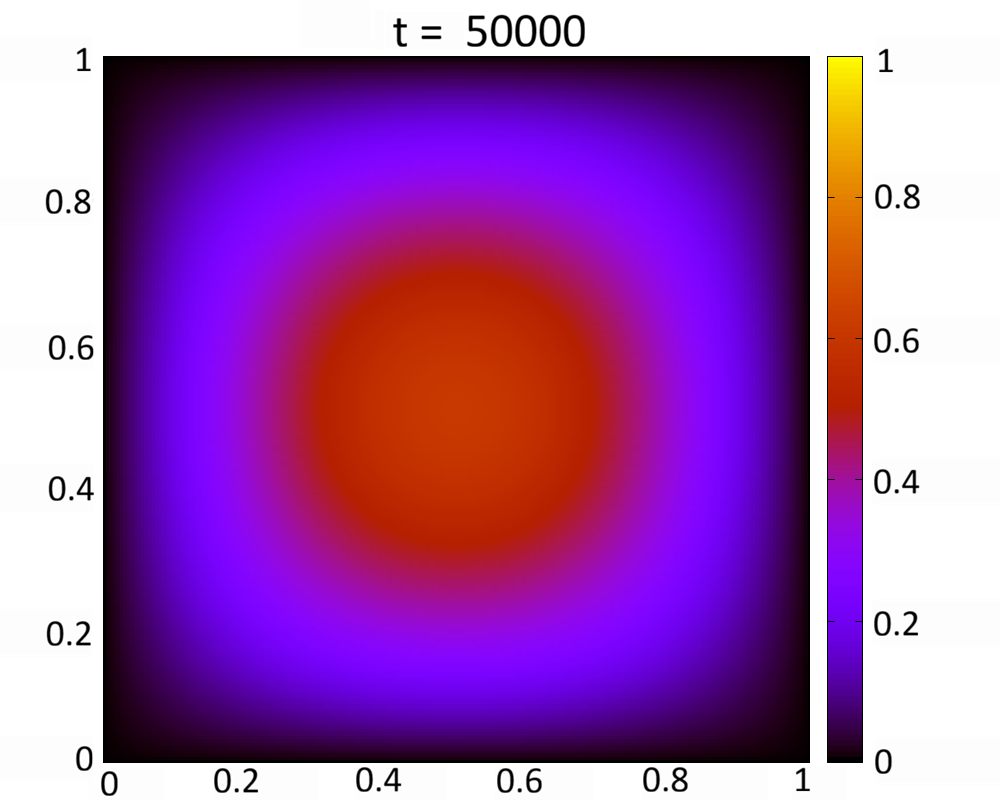} \\
    \includegraphics[width=0.3\linewidth]{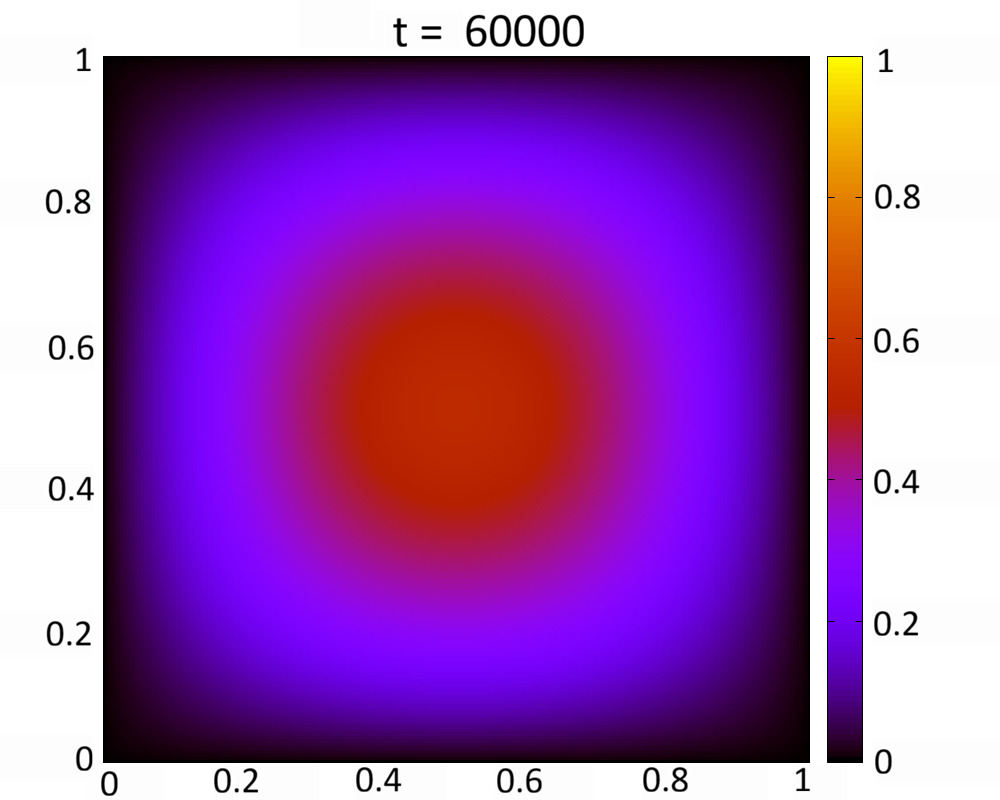}    \includegraphics[width=0.3\linewidth]{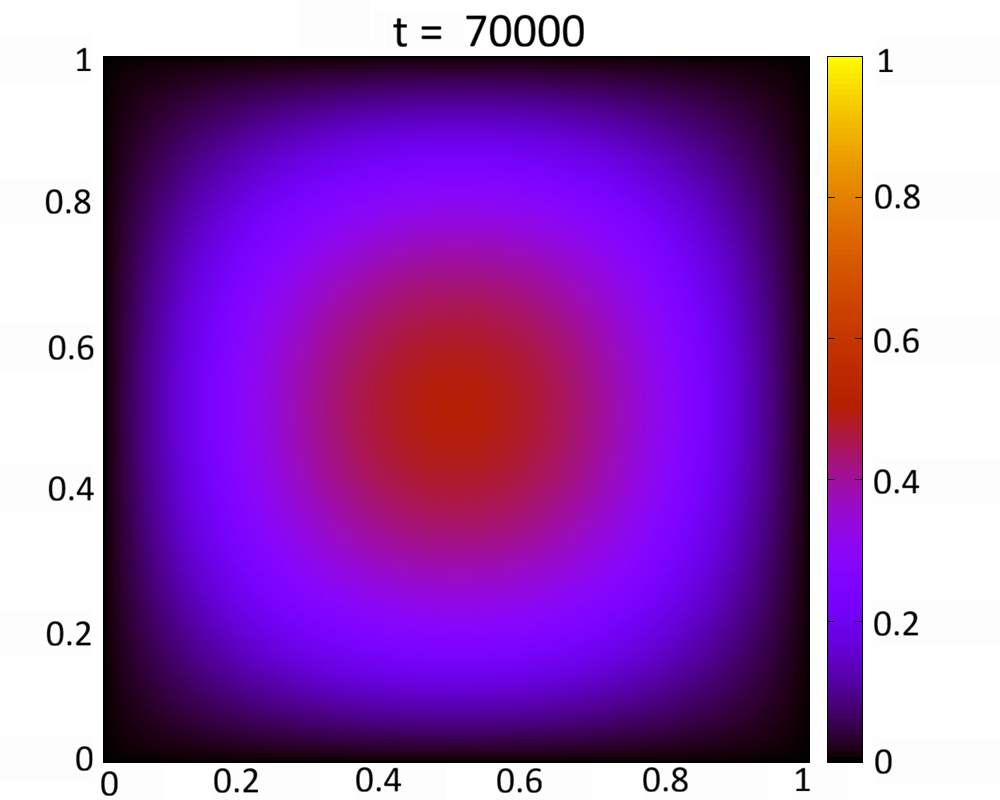}     \includegraphics[width=0.3\linewidth]{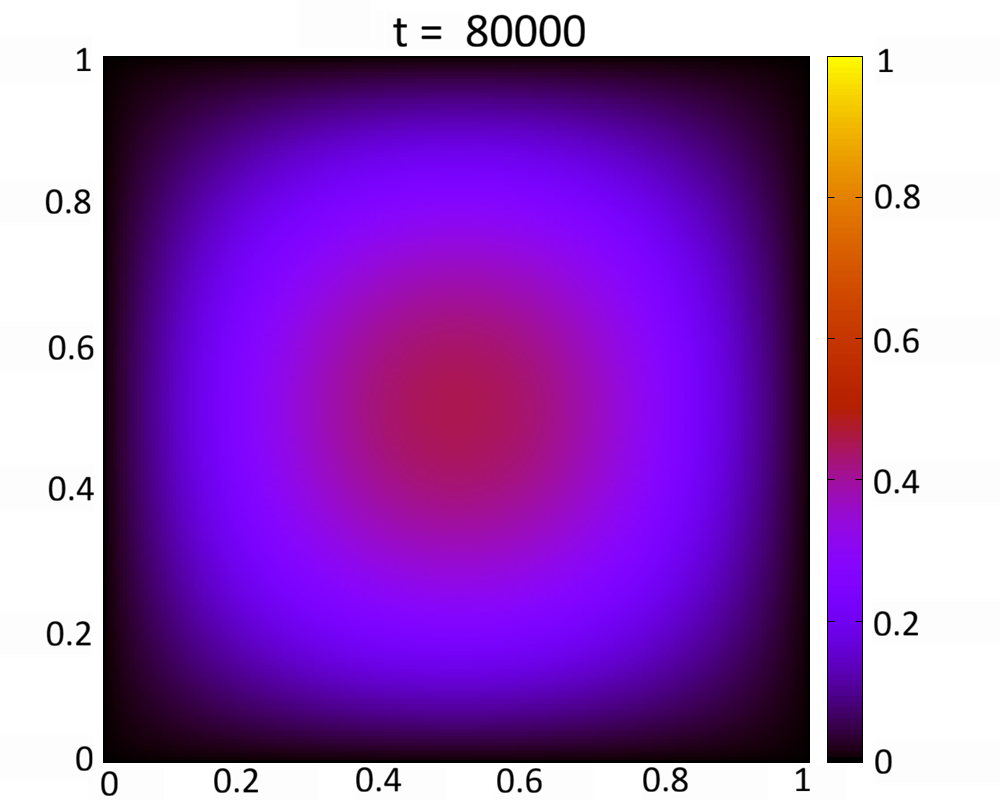} \\
    \includegraphics[width=0.3\linewidth]{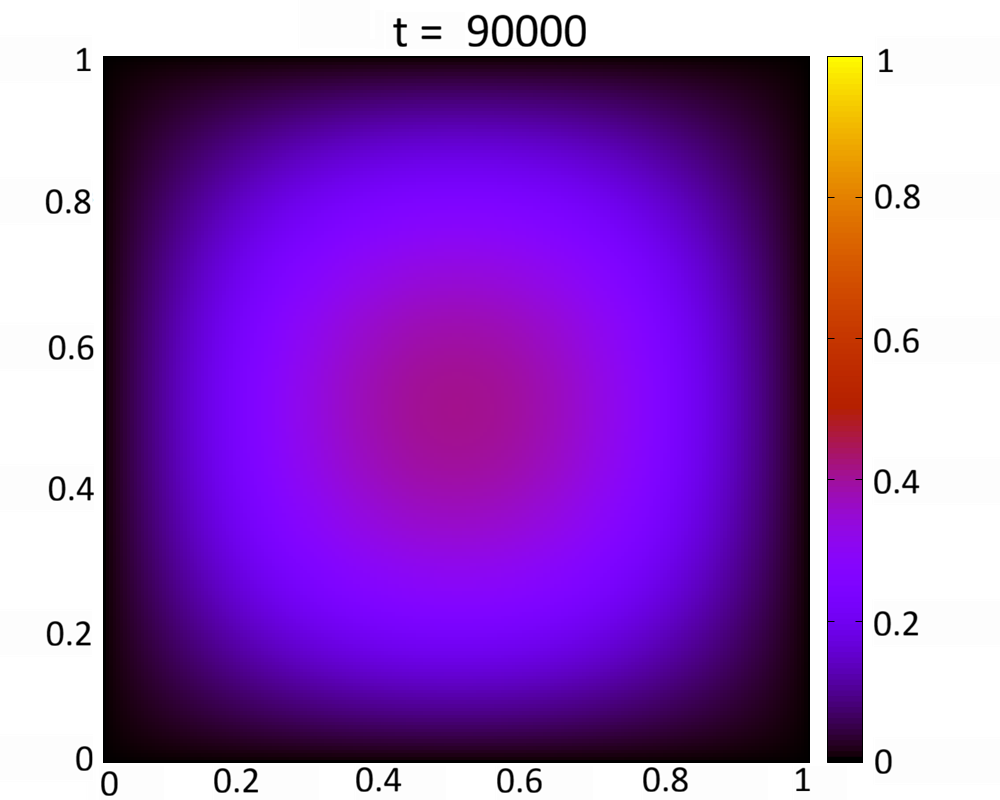}    \includegraphics[width=0.3\linewidth]{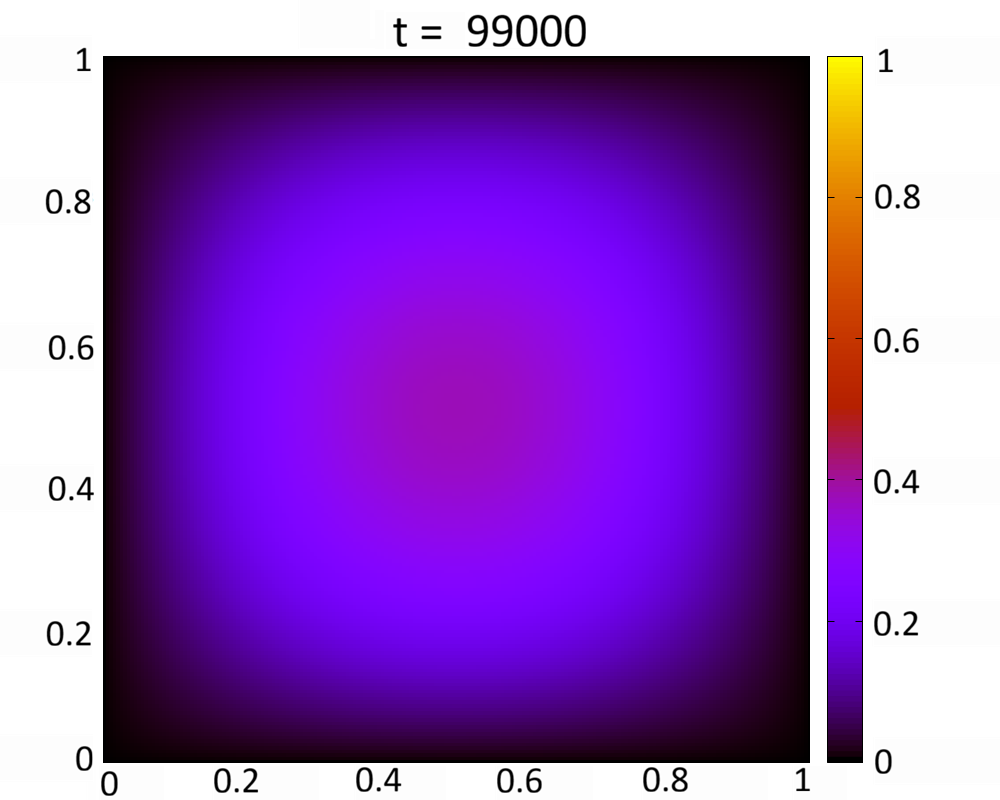}
    \caption{\revision{Snapshots of the IGA-ADS numerical simulation of the manufactured solution advection-diffusion problem.}}
    \label{fig:iga_manufactured}
\end{figure}

\section{In-field measurements of intermittent polluted air}
\label{sec:infield}
To measure the concentration of pollutants in the air, we used two snowmobiles, one behind the other. To the second snowmobile, we attached a  pollution sensor. The sensor for these in-situ air quality measurements  was obtained from  Airly Sp. z o.o. These sensors allow for time-series measurements of air temperature, humidity, pressure,
PM10, PM2.5, PM1, NO$_2$, and O$_3$ values. The sensor utilized was calibrated to deliver a relative error of less than 10 percent. The sensor and the snowmobile are illustrated in Figure~\ref{fig:measureexp}. The measurements have been performed at Adventalen valley\footnote{\tt  https://maps.app.goo.gl/5MiVg6D62yQsJHTQA}.
The measurements were taken on 17/03/2024 using a Yamaha Venture Lite 600cc snowmobile with two-stroke engines burning oil along with gasoline, and an Airly air quality sensor (see Fig.~\ref{fig:measureexp}) starting at 10:30~AM Svalbard local time in 10\nobreakdash-second intervals. The Airly sensor used in our experiment is a multipollutant air quality monitoring device that provides real-time measurements of CO, NO$_2$, O$_3$, and SO$_2$, as well as PM1, PM2.5, and PM10 mass concentrations, and environmental parameters such as pressure, temperature, and relative humidity. 
The devices use electrochemical sensors for gas measurements and a PMS5003 factory-calibrated laser particle counter, which detects particles based on their reflectance.
These sensors count suspended particles of $0.3$, $0.5$, $1.0$, $2.5$, $5.0$, and $10$ $\mu$m \cite{aqmd_airly_sensor, airly_sensors}. 


%
\begin{figure}[!h]
    \centering
    \includegraphics[width=0.5\linewidth]{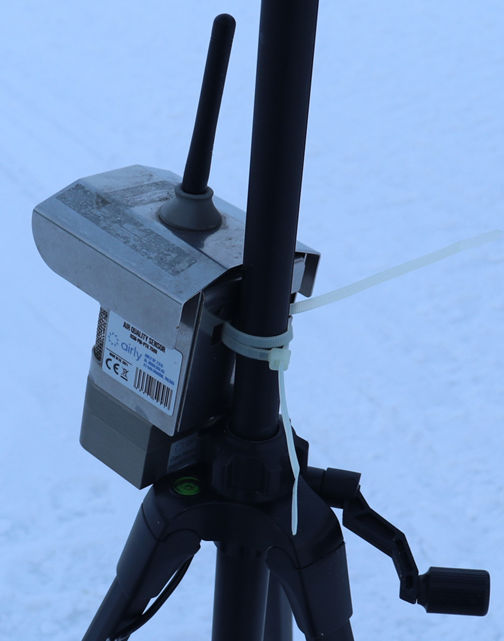} \includegraphics[width=0.358\linewidth]{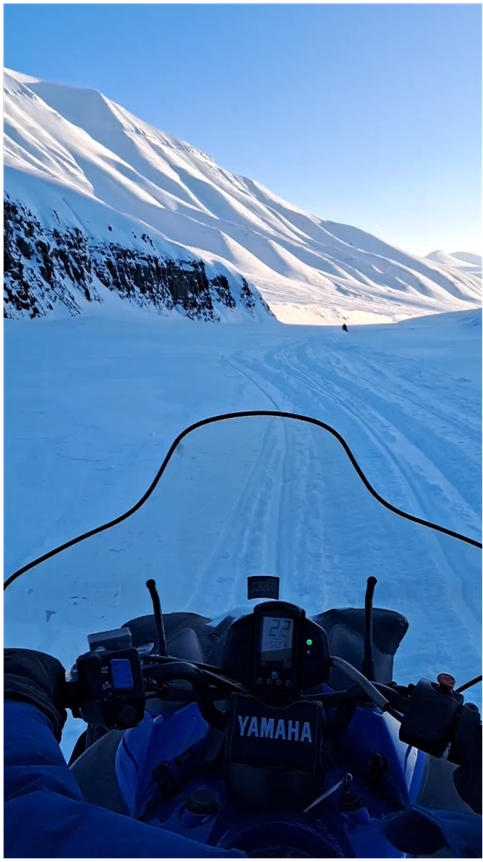}
    \caption{Snowmobiles and the air quality sensor used to take the pollution concentration measurements in Svalbard
on 17/03/2024.}
    \label{fig:measureexp}
\end{figure}

The results of the measurements with moving snowmobiles are presented in Figures ~\ref{fig:meas1}-\ref{fig:meas3}. As one may observe, there is a significant peak during acceleration and another at the end of the measurements, which was due to keeping the working vehicle in place. This concerns the measured NO$_2$ concentration (it goes as high as 40 $\mu$g/m$^3$ at the beginning, and 20-30 $\mu$g/m$^3$ at the end of the measurement), PM2.5  concentration (it reaches up to 250~$\mu$g/m$^3$ at the beginning) and
(\(\text{PM}_{10}\)) concentration (it reaches up to 250~$\mu$g/m$^3$ at the beginning). When the vehicle was moving at an approximately constant speed,  the concentration of NO$_2$ measured oscillated between 0 and 10 $\mu$g/m$^3$, the concentration of PM2.5 measured oscillated between  $0$ and around $5\mu$g/m$^3$, and  the concentration of PM10 measured oscillated between $0$ and 11~$\mu$g/$m^3$.
\begin{figure}[!htb]
    \centering
    \includegraphics[width=0.7\linewidth]{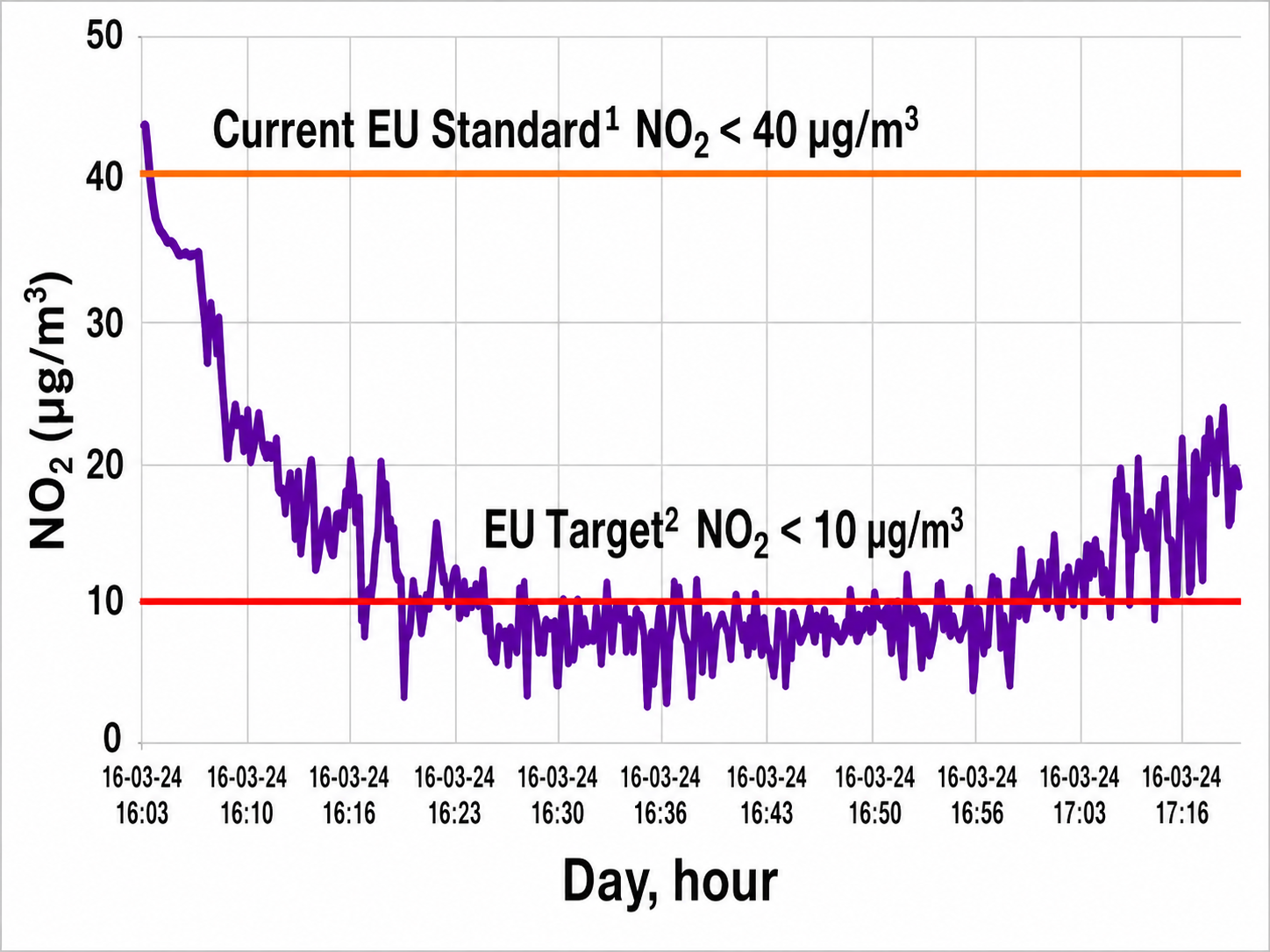}
    \caption{Mobile measurements of NO$_2$ from snowmobiles versus current Directive~\cite{DirectiveUE} and future policy~\cite{RevisionDirectiveUE}.}
    \label{fig:meas1}
\end{figure}
\begin{figure}[!htb]
    \centering
    \includegraphics[width=0.7\linewidth]{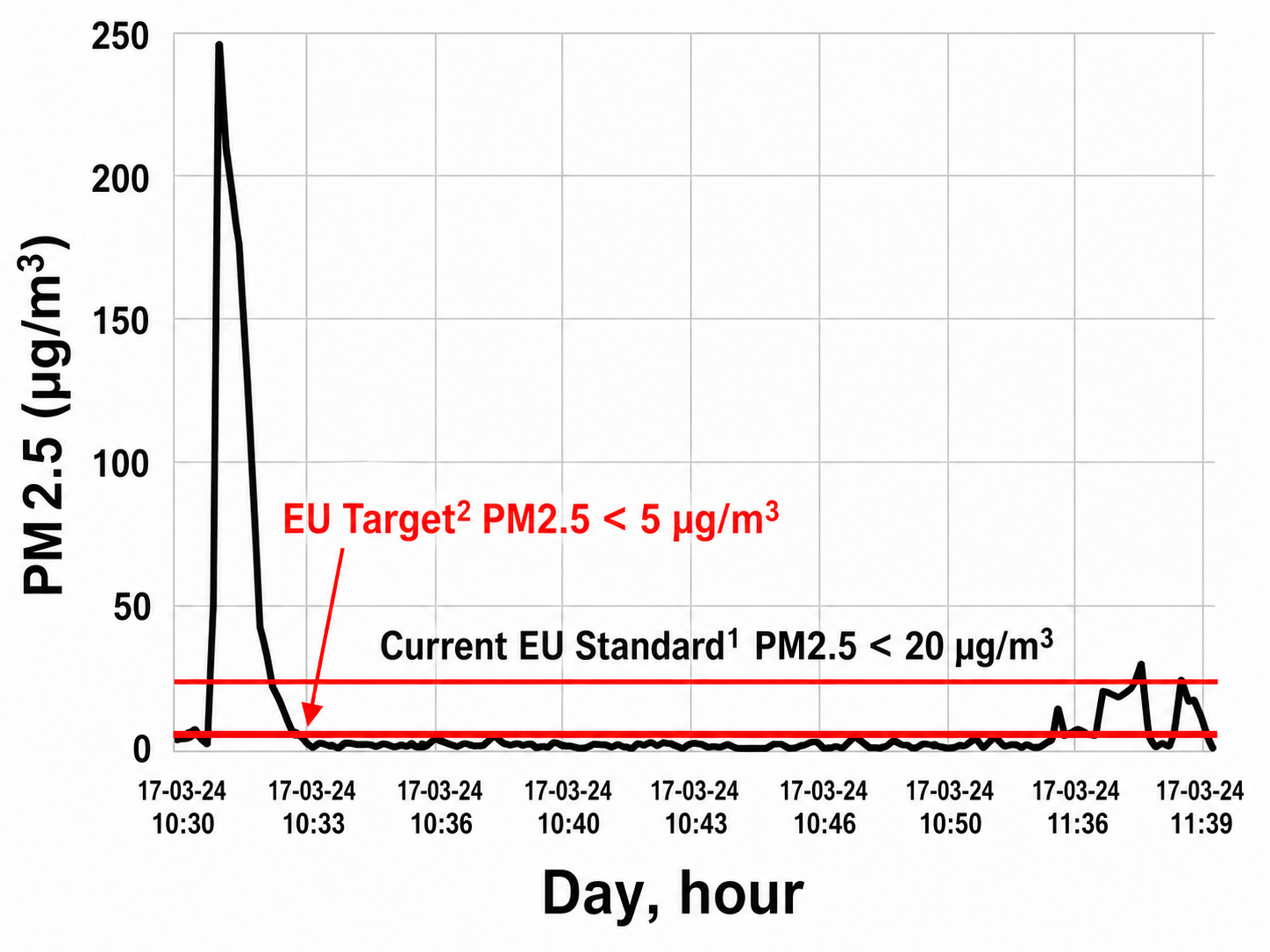}
    \caption{Mobile measurements of PM2.5 from snowmobiles versus current Directive~\cite{DirectiveUE} and future policy~\cite{RevisionDirectiveUE}}
    \label{fig:meas2}
\end{figure}
\begin{figure}[!htb]
    \centering
    \includegraphics[width=0.7\linewidth]{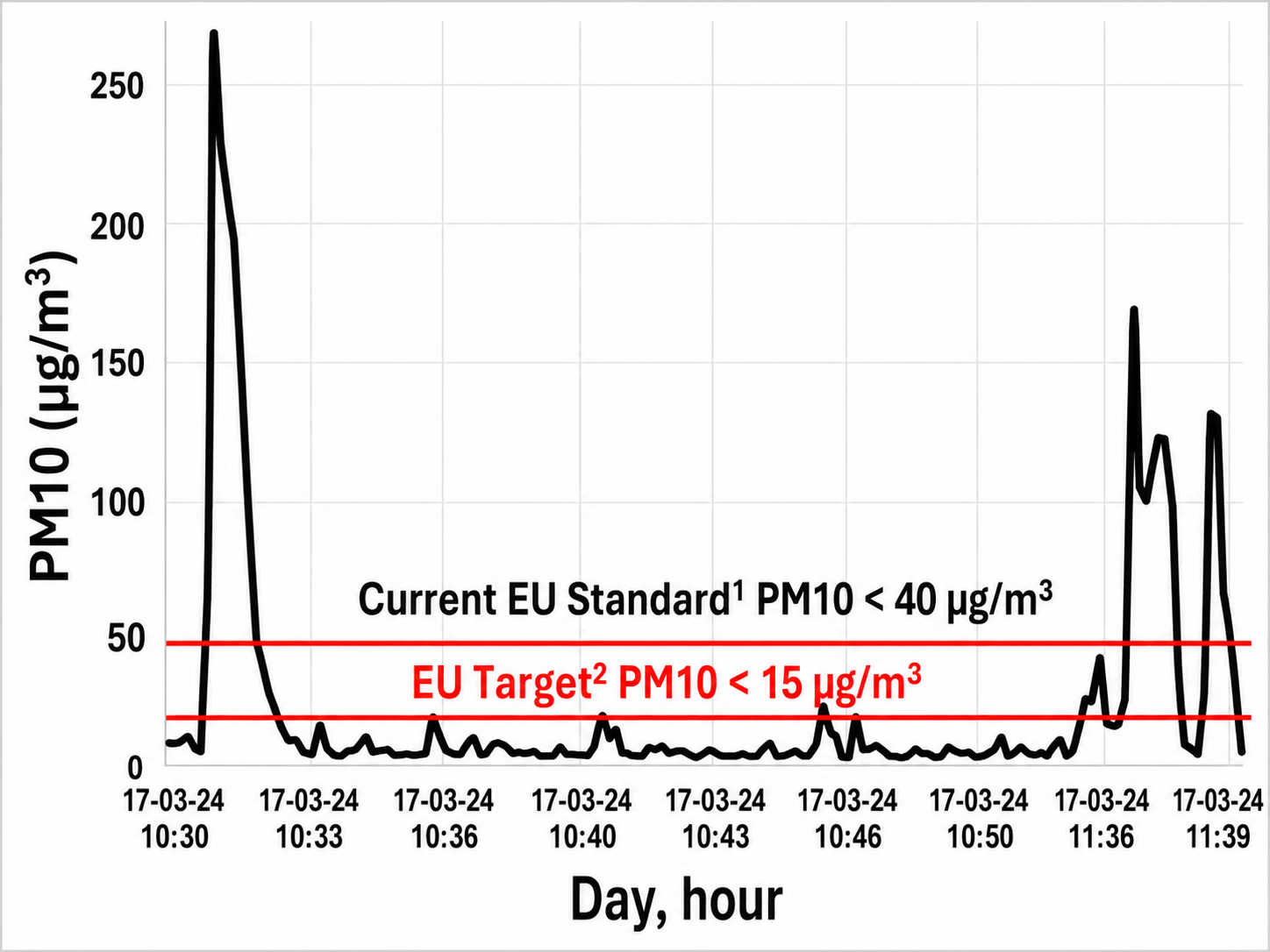}
    \caption{Mobile measurements of PM10 from snowmobiles versus current Directive~\cite{DirectiveUE} and future policy~\cite{RevisionDirectiveUE}.}
    \label{fig:meas3}
\end{figure}

In addition to the moving snowmobile measurements, we perform stationary measurements near Longyearbyen\footnote{\tt  https://maps.app.goo.gl/LBKn3ojpMowcNXUY8} and the results are presented in Figures~\ref{fig:meas4}-\ref{fig:meas6}. As one may observe, the ambient concentration of (\(\text{NO}_{2}\)) measured oscillated between 0 and 10 $\mu g/m^3$, the concentration of (\(\text{PM}_{2.5}\)) measured oscillated between $0$ and around $5 \mu g/m^3$ and the concentration of (\(\text{PM}_{10}\)) measured oscillated between $0$ and 11$\mu$g/m$^3$.
\begin{figure}[!htb]
    \centering
    \includegraphics[width=0.7\linewidth]{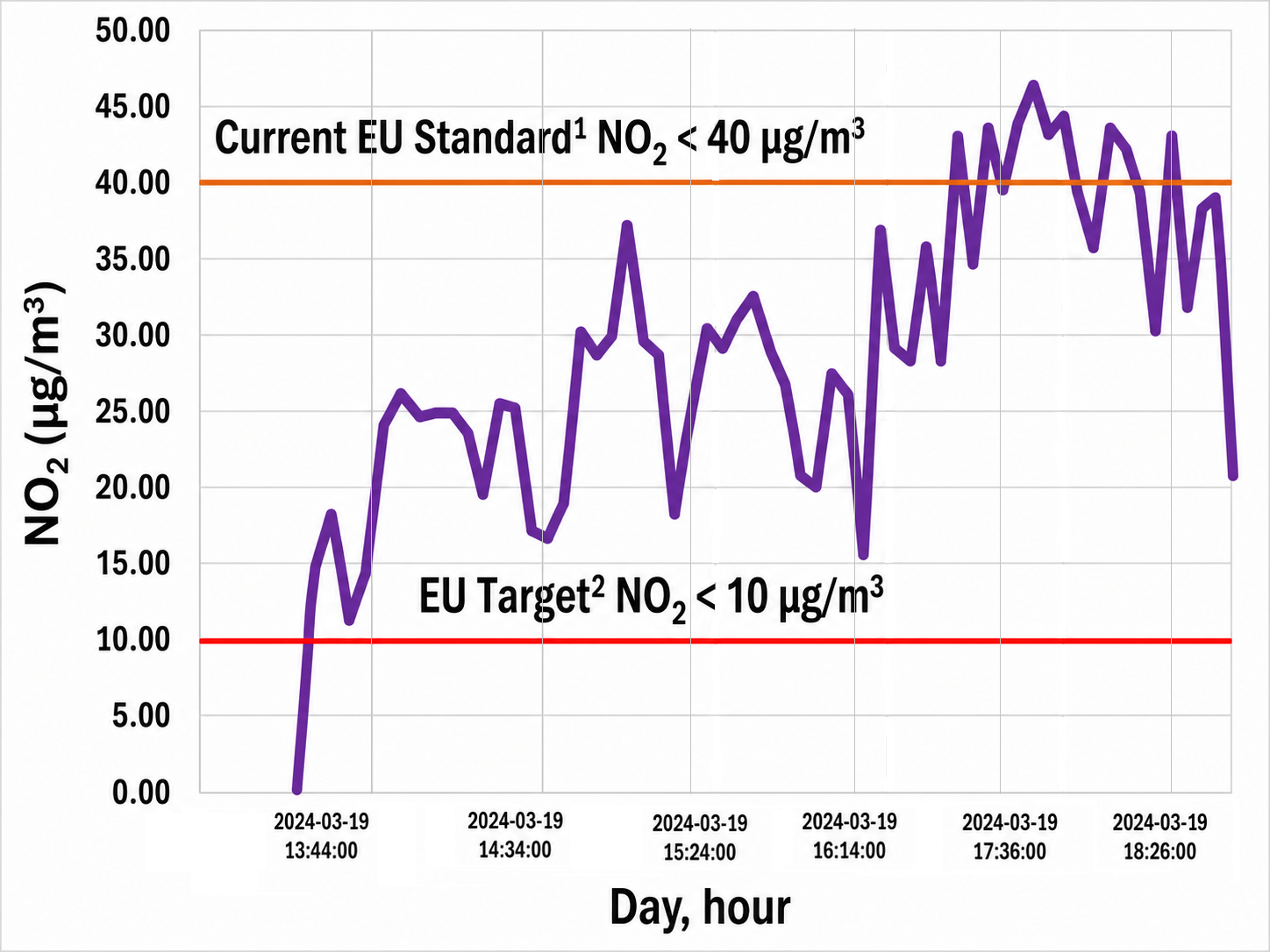}
    \caption{Stationary measurements of NO$_2$ versus current Directive~\cite{DirectiveUE} and future policy~\cite{RevisionDirectiveUE}.}
    \label{fig:meas4}
\end{figure}
\begin{figure}[!htb]
    \centering
    \includegraphics[width=0.7\linewidth]{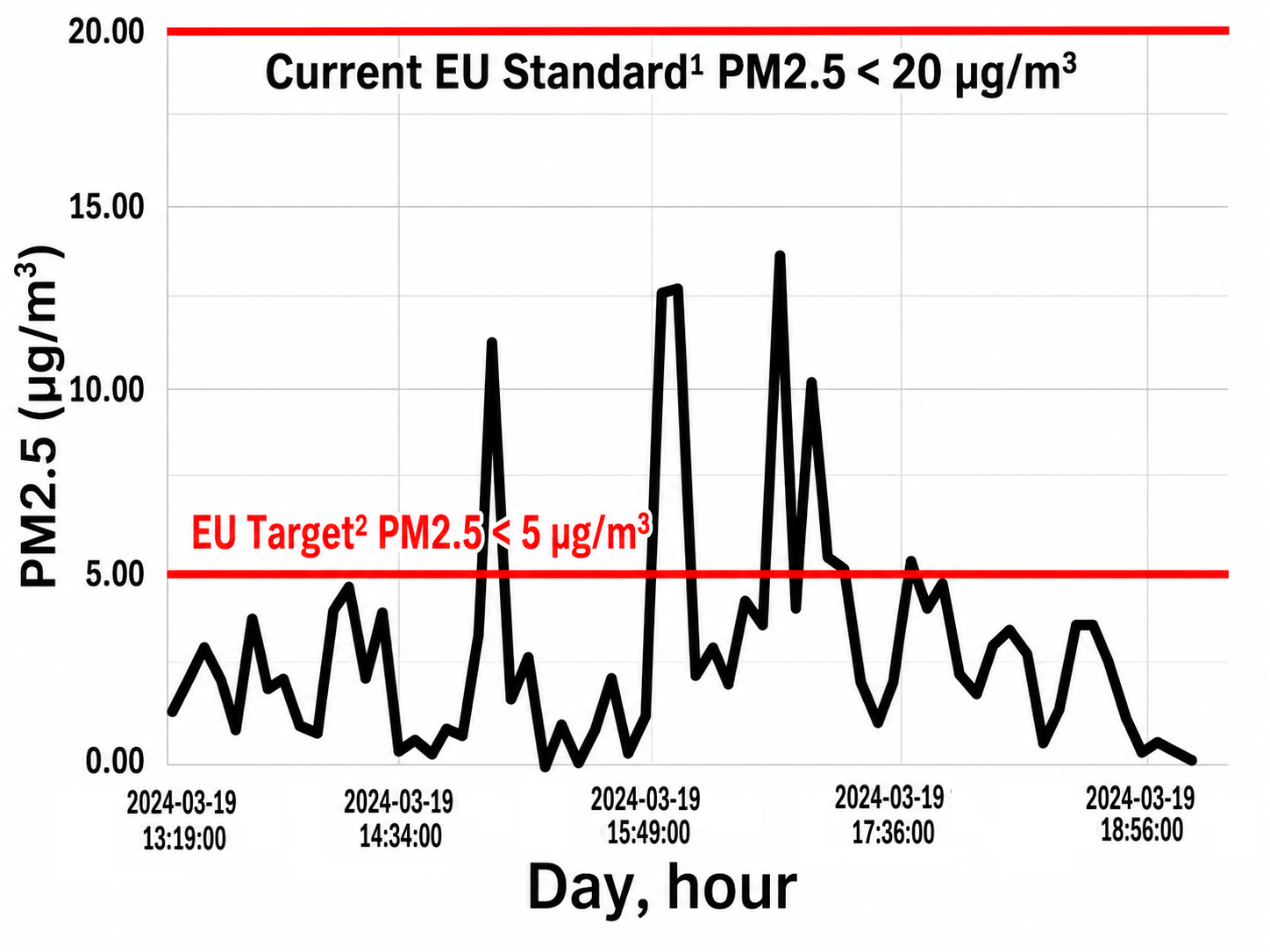}
    \caption{Stationary measurements of PM2.5 versus current Directive~\cite{DirectiveUE} and future policy~\cite{RevisionDirectiveUE}.}
    \label{fig:meas5}
\end{figure}
\begin{figure}[!htb]
    \centering
    \includegraphics[width=0.7\linewidth]{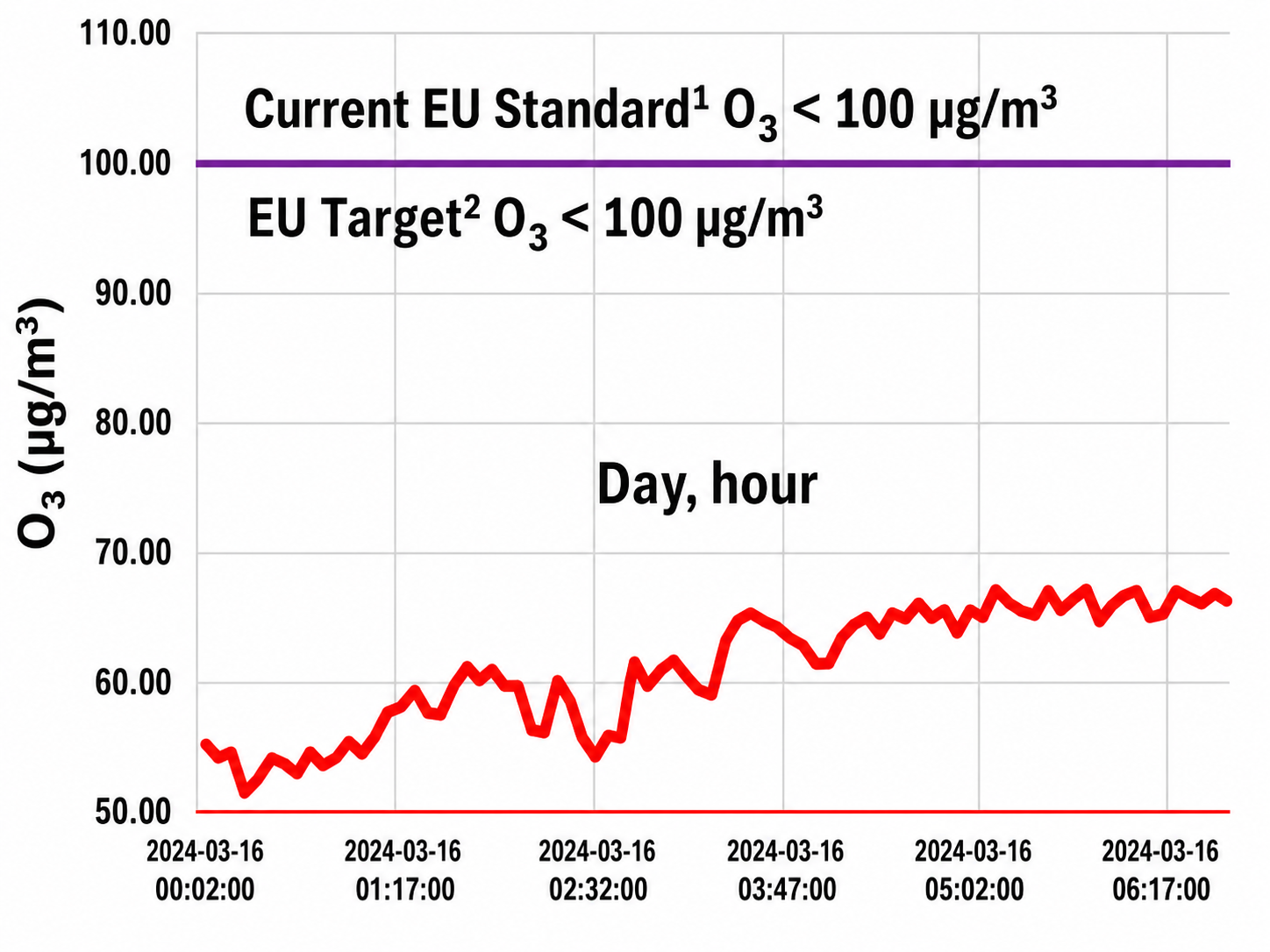}
    \caption{Stationary measurements of O$_3$ versus current Directive~\cite{DirectiveUE} and future policy~\cite{RevisionDirectiveUE}.}
    \label{fig:meas6}
\end{figure}
\begin{figure}[!htb]
    \centering
    \includegraphics[width=0.7\linewidth]{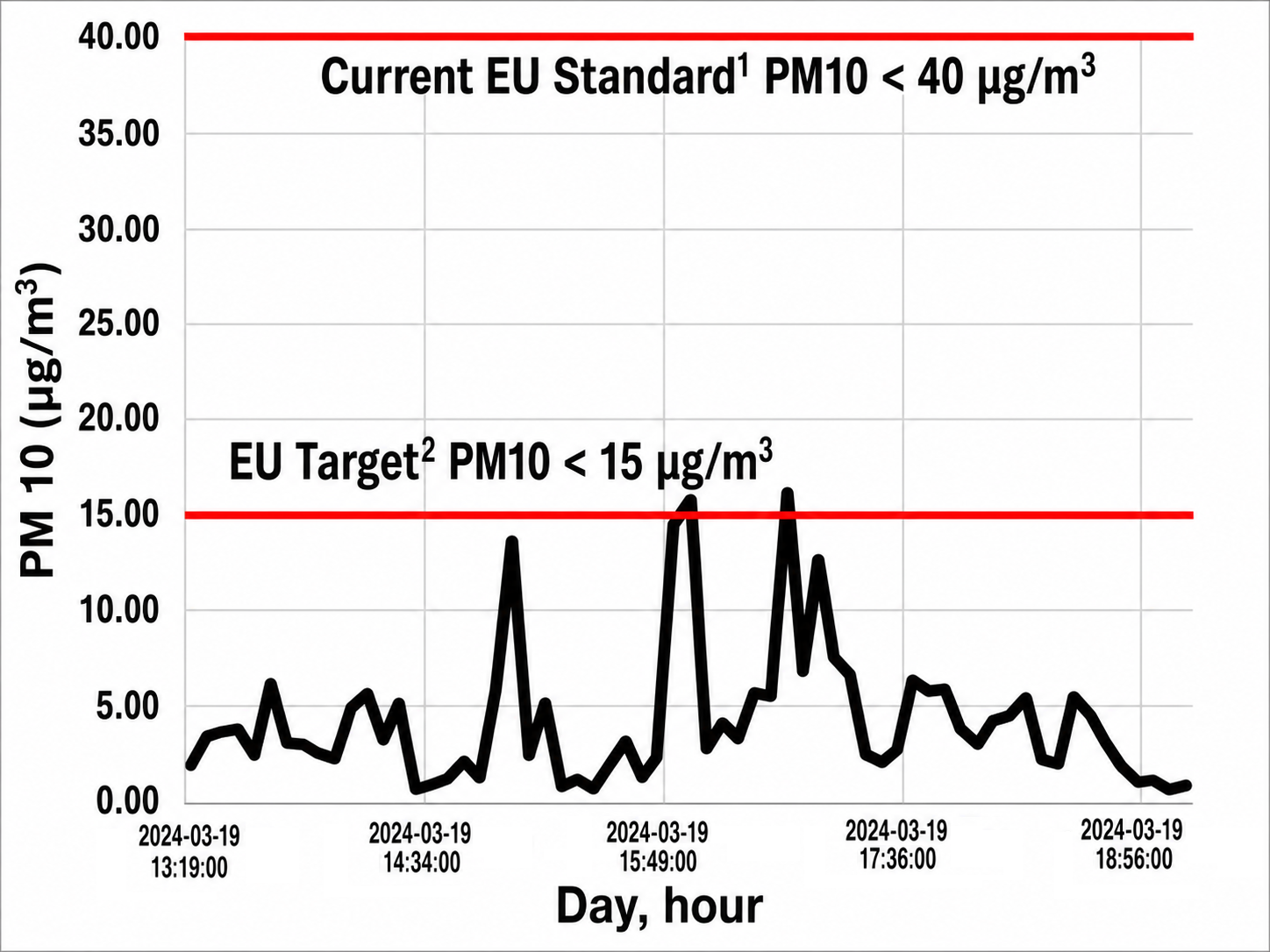}
    \caption{Stationary measurements of PM10 versus current Directive~\cite{DirectiveUE} and future policy~\cite{RevisionDirectiveUE}.}
    \label{fig:meas7}
\end{figure}
Analyzing such results, one can see the locally polluted air. 
Namely, from Figure~\ref{fig:meas4}, we can see over a 40 $\mu g/m^3$ concentration of NO$_2$, which is higher than the current EU standard. From Figure~\ref{fig:meas5}, we can observe a concentration of PM2.5 of up to 13 $\mu g/m^3$, which is higher than the EU target of 5 $\mu g/m^3$. We can also see from Figure~\ref{fig:meas6} the increased concentration of $O_3$, and from Figure~\ref{fig:meas7} the concentration of PM10 reaching the EU Target of 15 $\mu g/m^3$.

\section{\revision{Simulation of pollution generated by snowmobiles with CRVPINN}}

\revision{We are now performing the simulation of the pollution generated by snowmobiles. Our computational domain is a two-dimensional cross-section of the Adventalen valley\footnote{\tt  https://maps.app.goo.gl/5MiVg6D62yQsJHTQA} starting from the location at Longyearbyean. We employ the model described in Section \ref{sec:pinnsec}, where the generated pollution is trapped near the ground due to the vertical advection profile, the horizontal diffusion is ten times larger than the vertical diffusion ($K_x=1.0$, $K_y=0.1$), and the pollution source is modeled as moving Gaussian profile with the setup described in Section \ref{sec:pinnsec}.}

\revision{We employ the CRVPINN method with the Gram matrix build using the $H^1_0$ norm in the space-time formulation. We use the residual, as described in Section \ref{sec:pinns}. 
\revision{The initial concentration is prescribed as
$
u(x,y,0)=u_0(x,y),
$
where $u_0$ is specified by the initial-condition function used in the numerical implementation.}
\revision{Homogeneous Neumann boundary conditions are imposed on all boundaries,
$
\frac{\partial u}{\partial x}=0$
 on
$x=0,L_x,
$
and
$
\frac{\partial u}{\partial y}=0$
on $y=0,L_y.
$
These conditions model zero diffusive flux through the domain boundaries.}
The residual loss is transformed into the robust loss by employing the inverse of the Gram matrix.}

\begin{figure}
    \centering
    a)\includegraphics[width=0.45\linewidth]{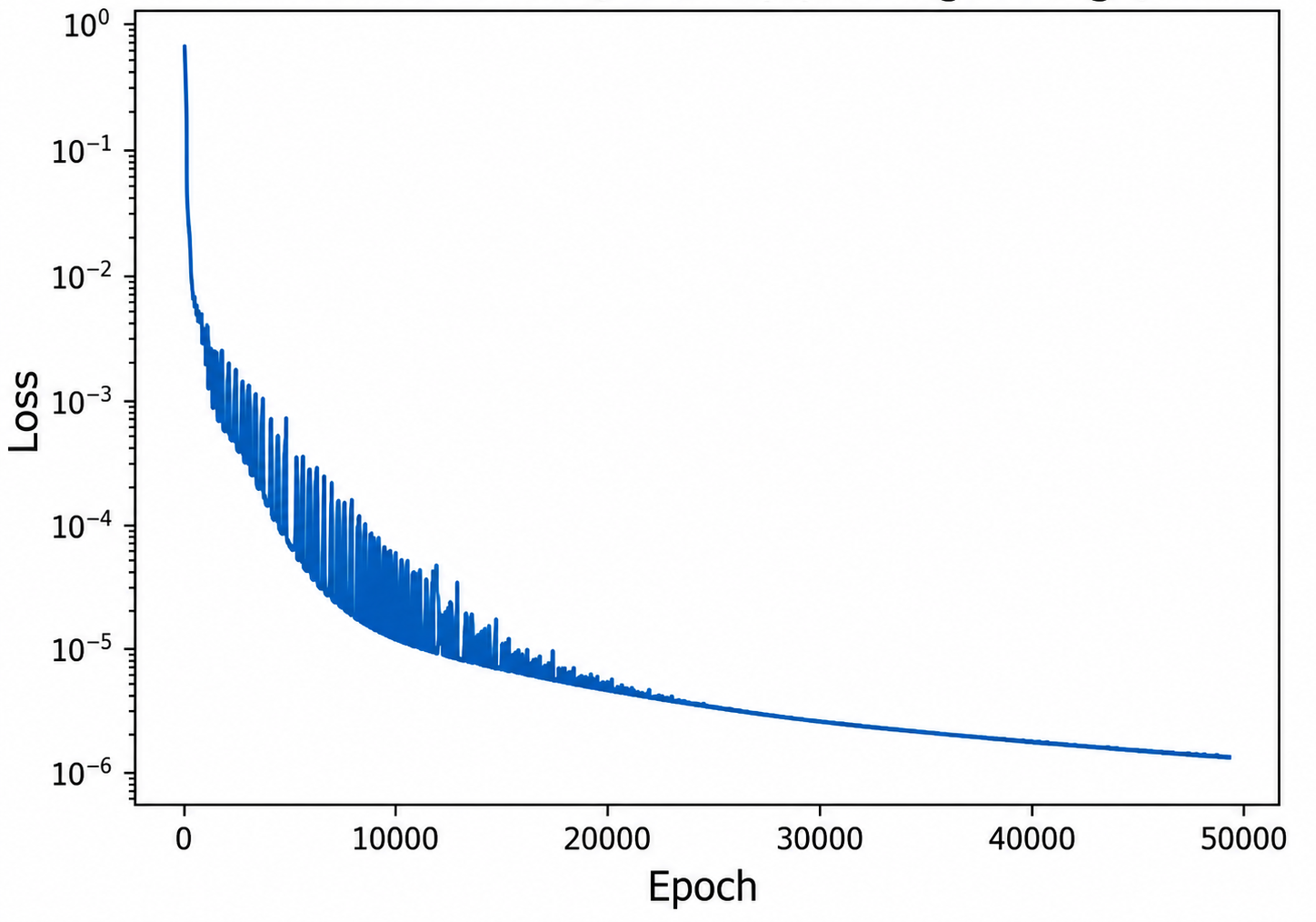}
    b)\includegraphics[width=0.45\linewidth]{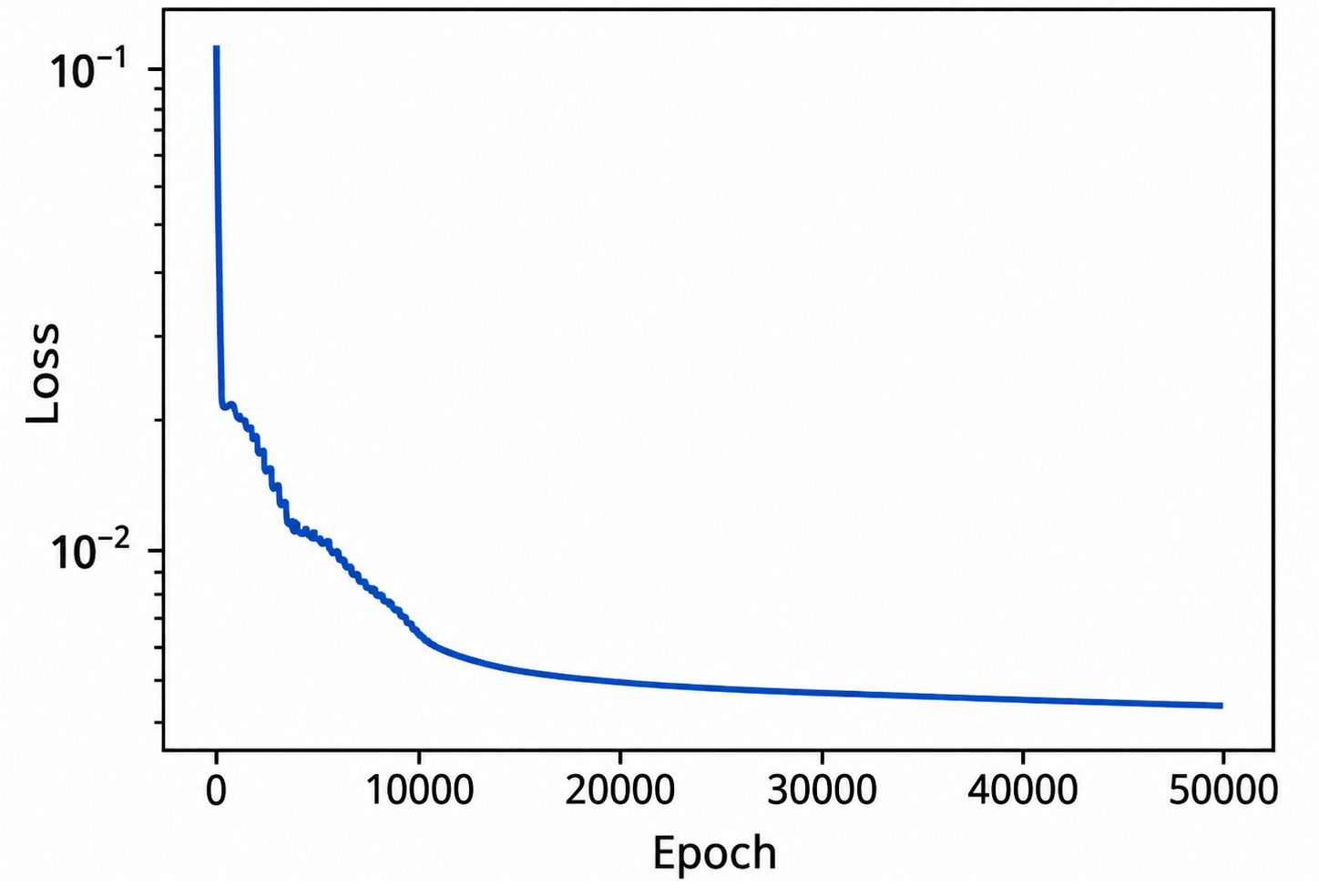}
    \caption{The convergence of the PINN (a) and CRVPINN (b) training.}
    \label{fig:residual}
\end{figure}

The numerical experiments are presented in Figures \ref{fig:residual}-\ref{fig:snowmobile1}.
We compare the convergence of the PINN training with the convergence of the CRVPINN training in the Figures \ref{fig:residual}(a) and \ref{fig:residual}(b). 
The PINN loss converges in 50,000 iterations to the value of $10^{-6}$. However, we do not know how this loss value is related to the accuracy of the solution. 
The PINN loss can be small but the numerical error of the neural network solution, defined as the $\| u_\theta - u_{exact}\|_{H^1_0(\Omega\times I)}$ can be arbitrarily large.
The benefit of the CRVPINN method is that the CRVPINN loss is robust since it fulfills  
$\sqrt{{\cal L}_{crvpinn}} \leq \| u_\theta - u_{exact}\|_{H^1_0(\Omega\times I)} \leq \sqrt{{\cal L}_{crvpinn}}$~\cite{LOS2025107839}. 
Introducing the robust loss of the CRVPINN method makes the loss bounded by the true error, so the convergence of CRVPINN loss presents the estimate of the numerical accuracy of the solution. In our case, the CRVPINN loss after 50,000 iterations converges to the value of 0.005, which means that
$\| u_\theta - u_{exact}\|_{H^1_0(\Omega\times I)} \leq \sqrt{{\cal L}_{crvpinn}}=\sqrt{0.005}=0.07$.
In other words, the numerical accuracy of the obtained simulations in $H^1_0(\Omega \times I)$ norm is equal 0.07.
\begin{figure}
    \centering
    \includegraphics[width=0.3\linewidth]{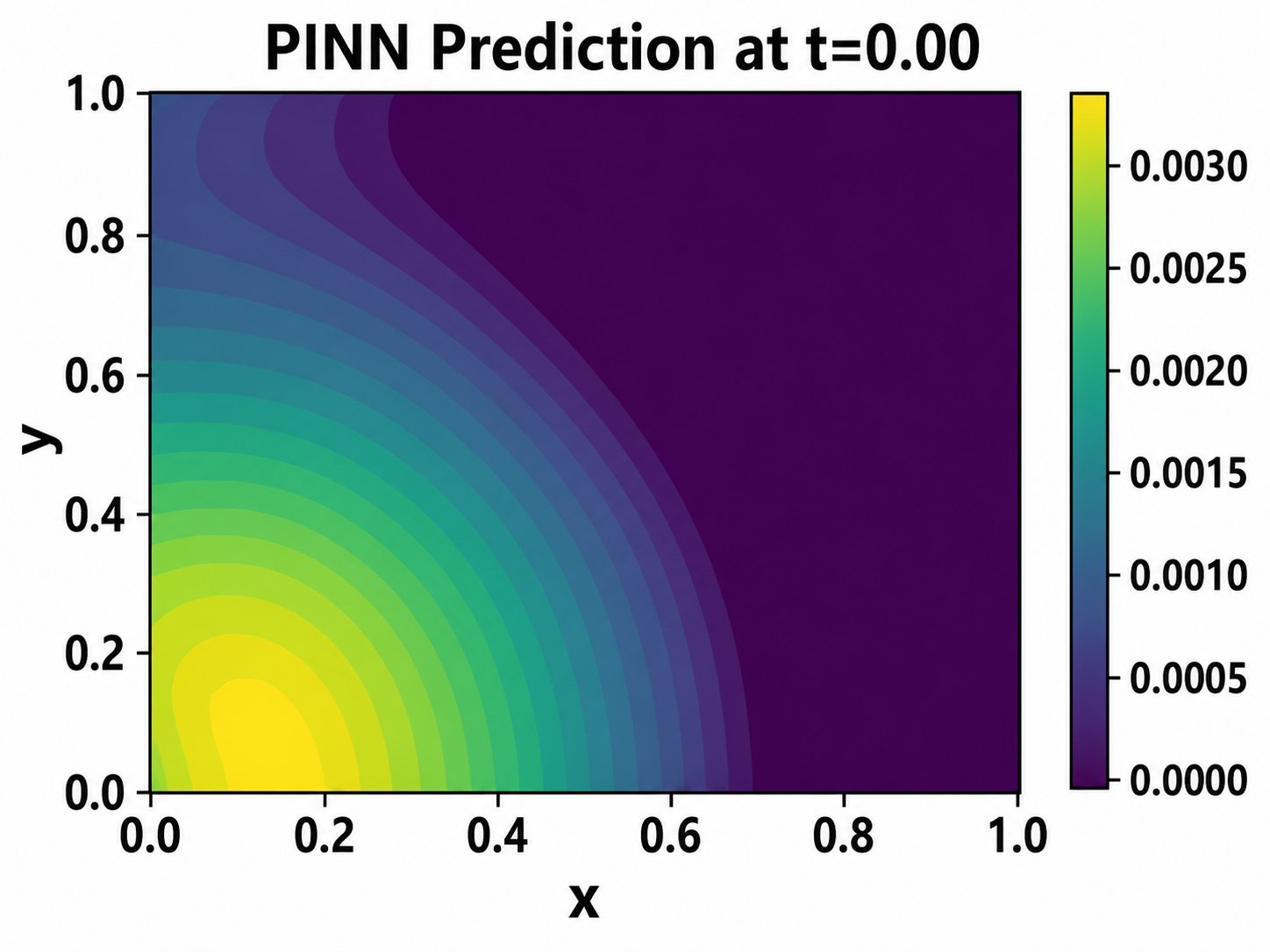}    \includegraphics[width=0.3\linewidth]{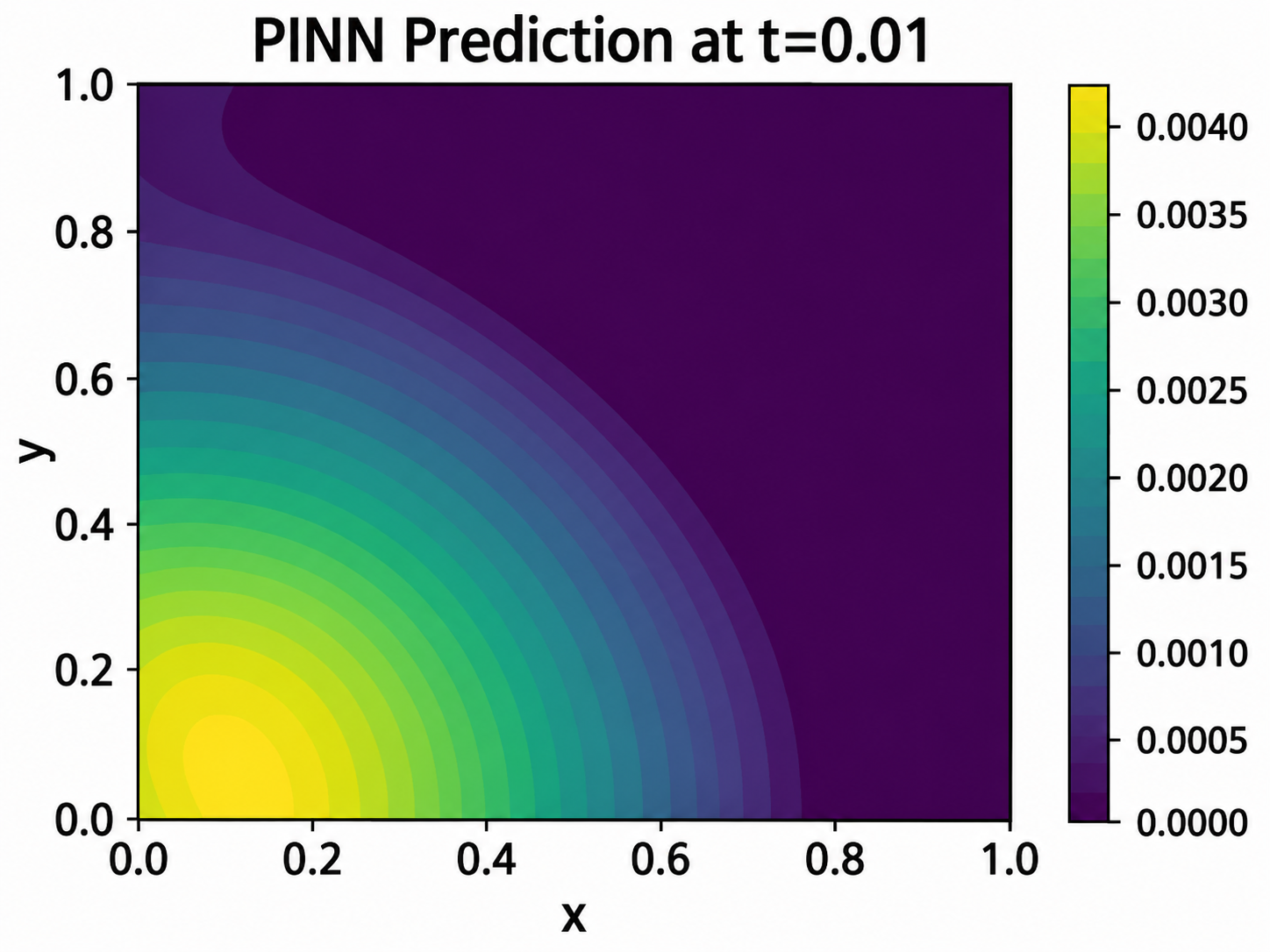}     \includegraphics[width=0.3\linewidth]{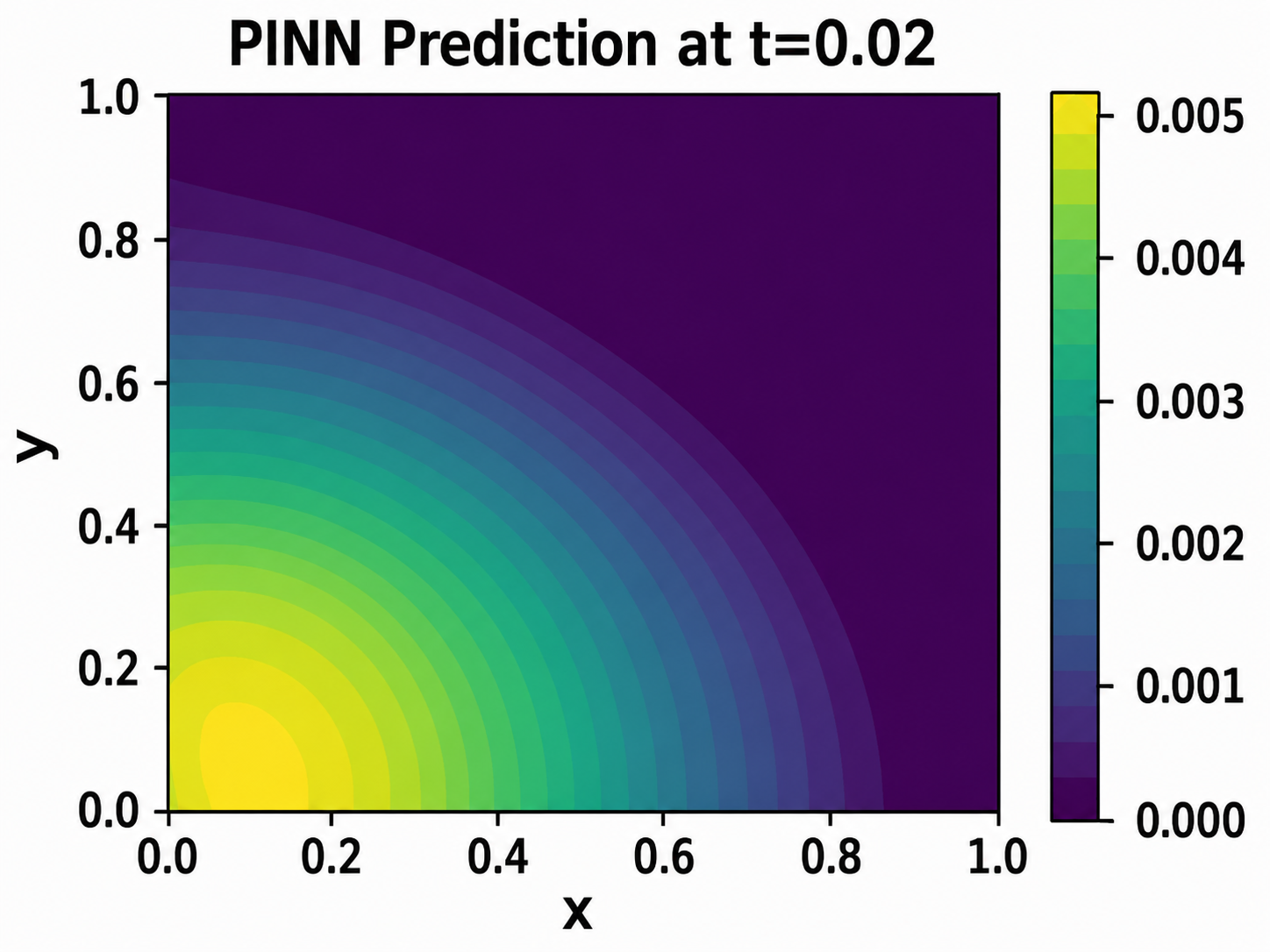} \\
    \includegraphics[width=0.3\linewidth]{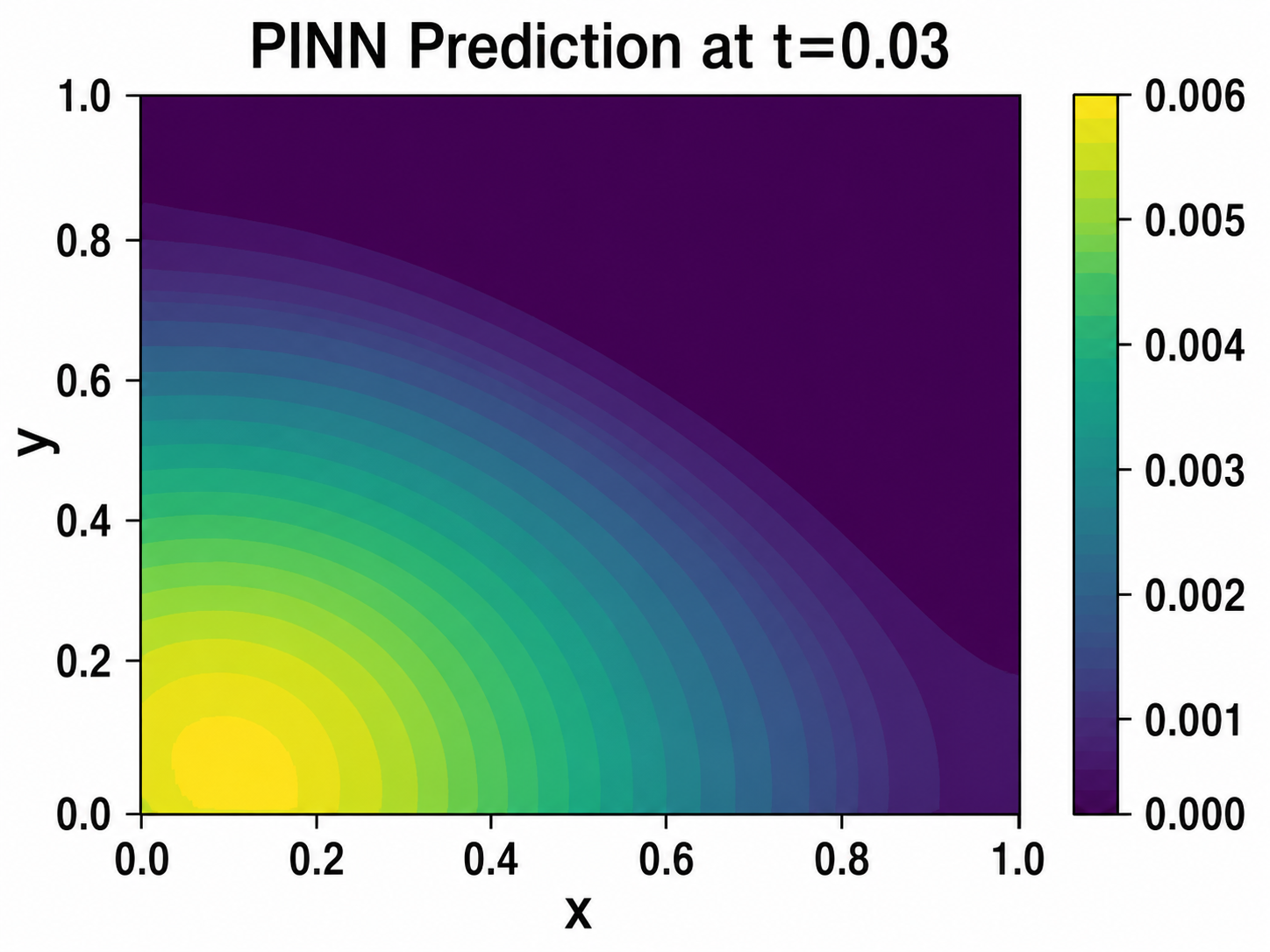}    \includegraphics[width=0.3\linewidth]{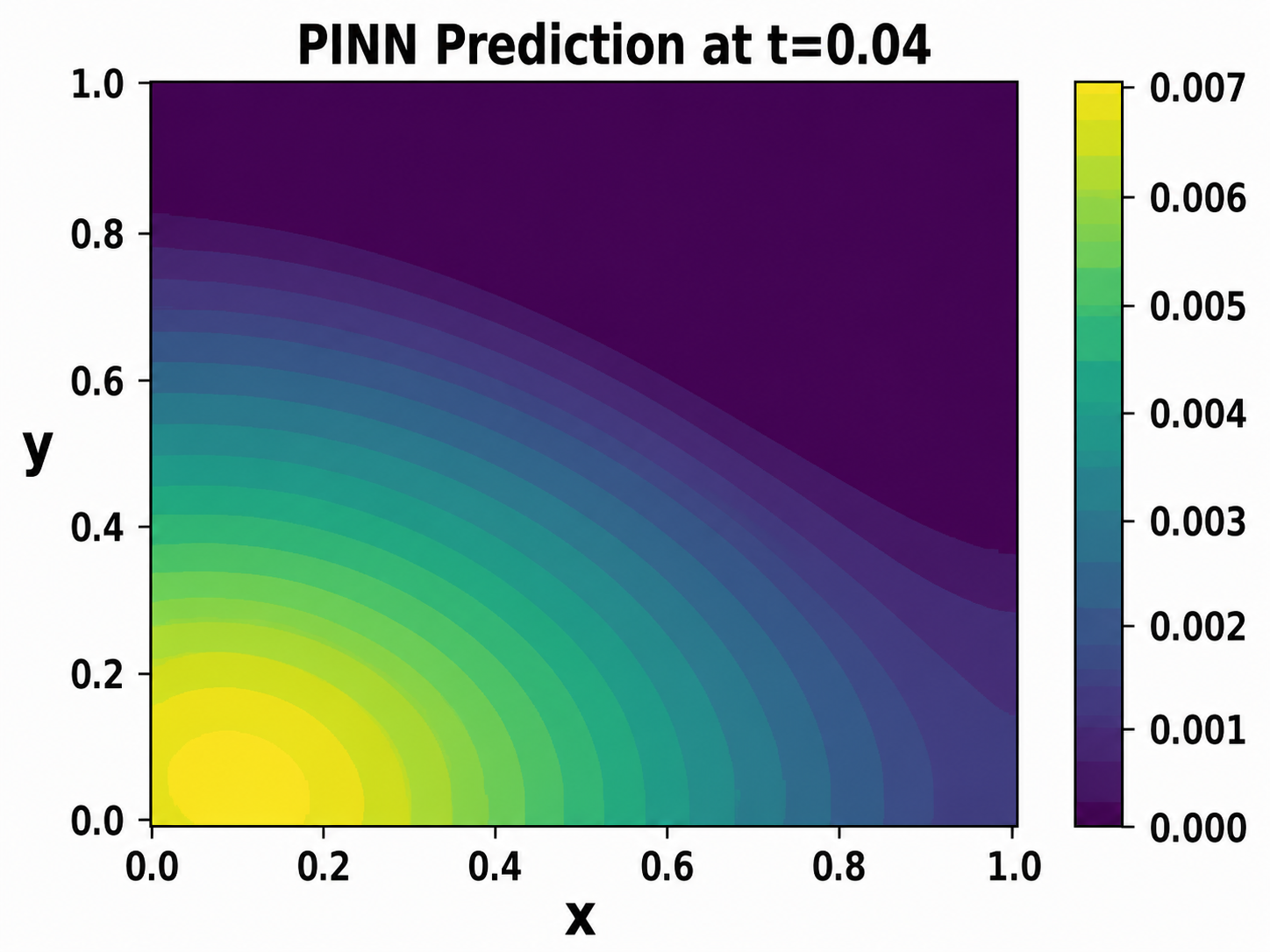}     \includegraphics[width=0.3\linewidth]{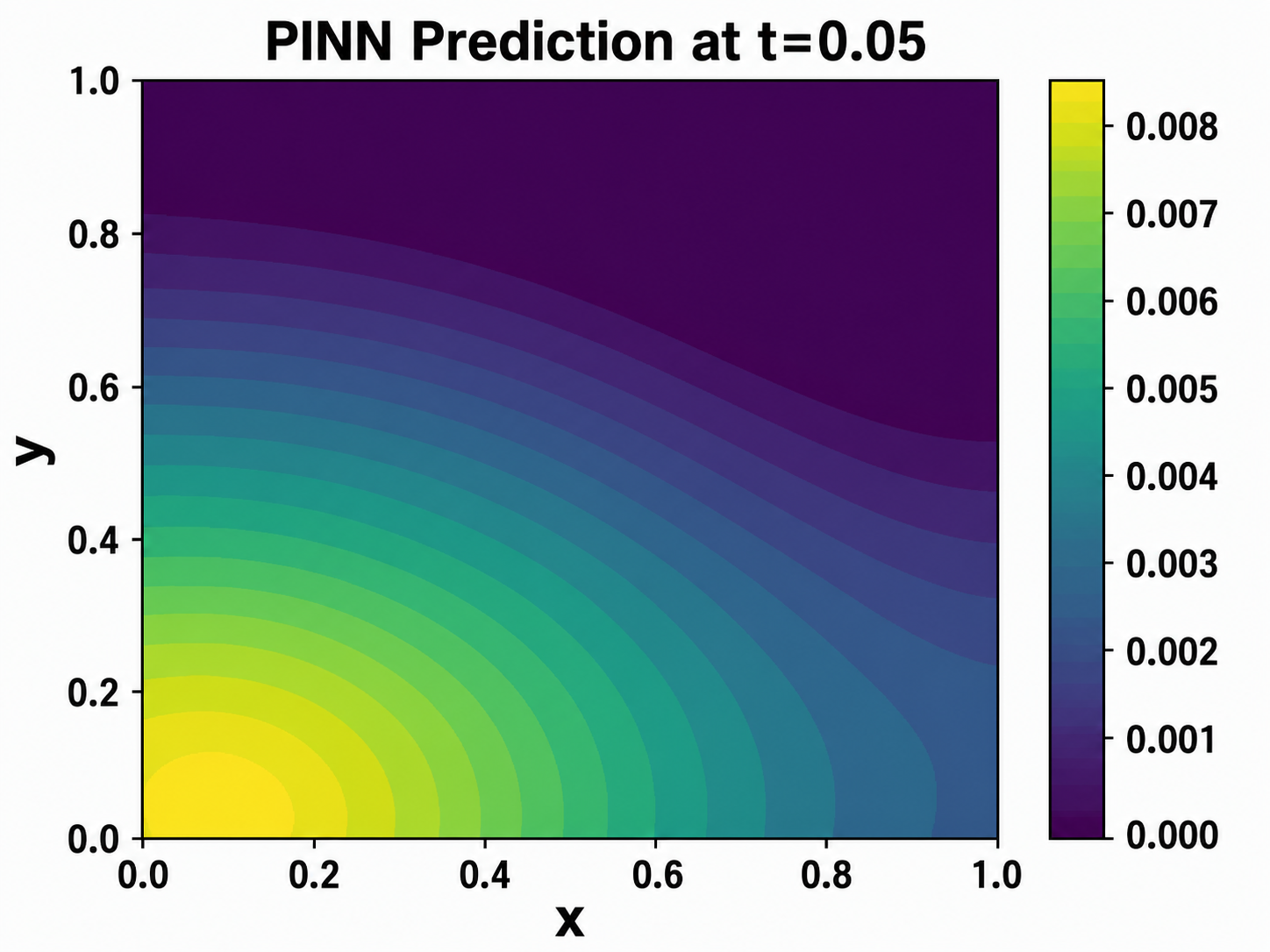} \\
    \includegraphics[width=0.3\linewidth]{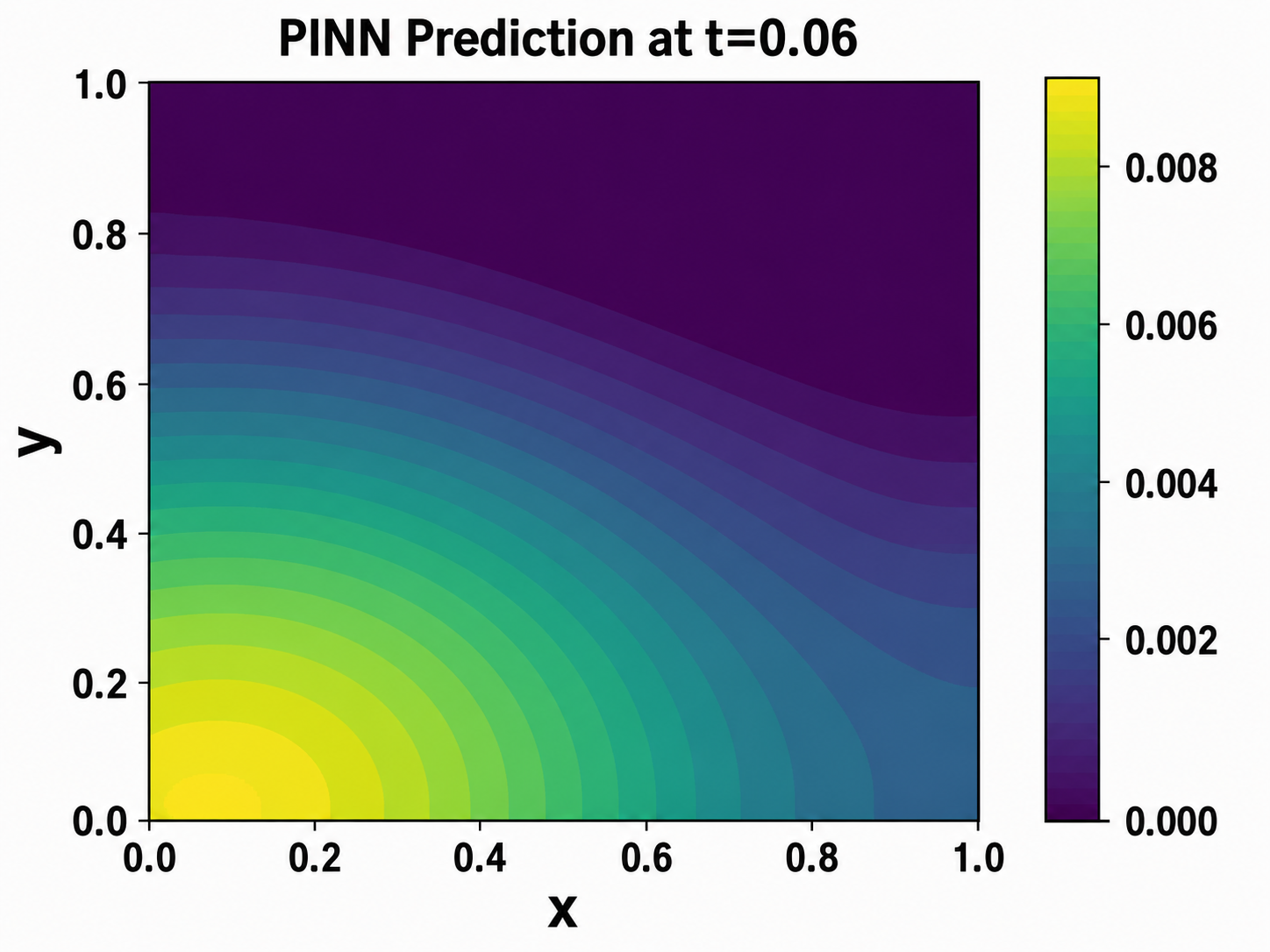}    \includegraphics[width=0.3\linewidth]{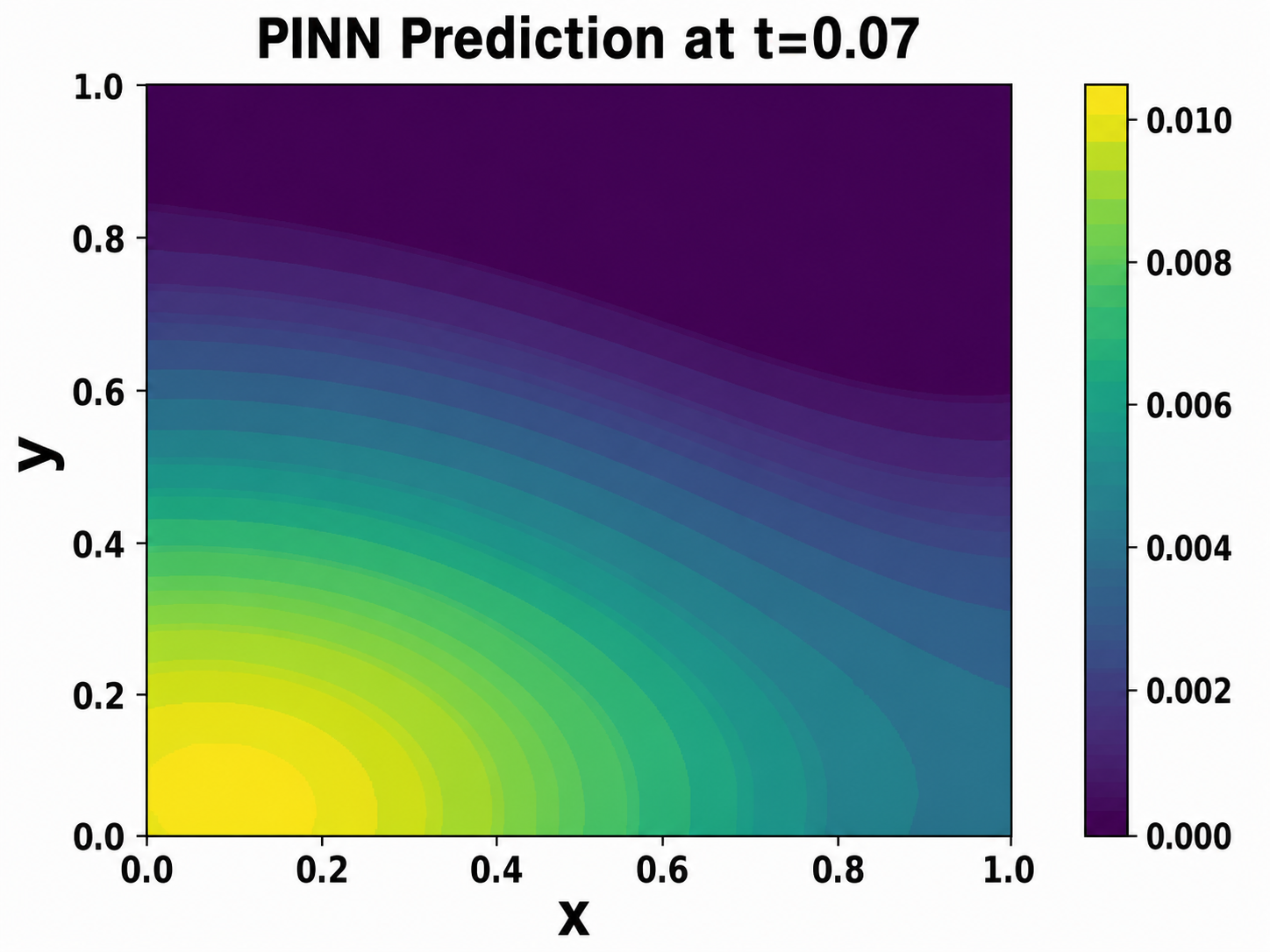}     \includegraphics[width=0.3\linewidth]{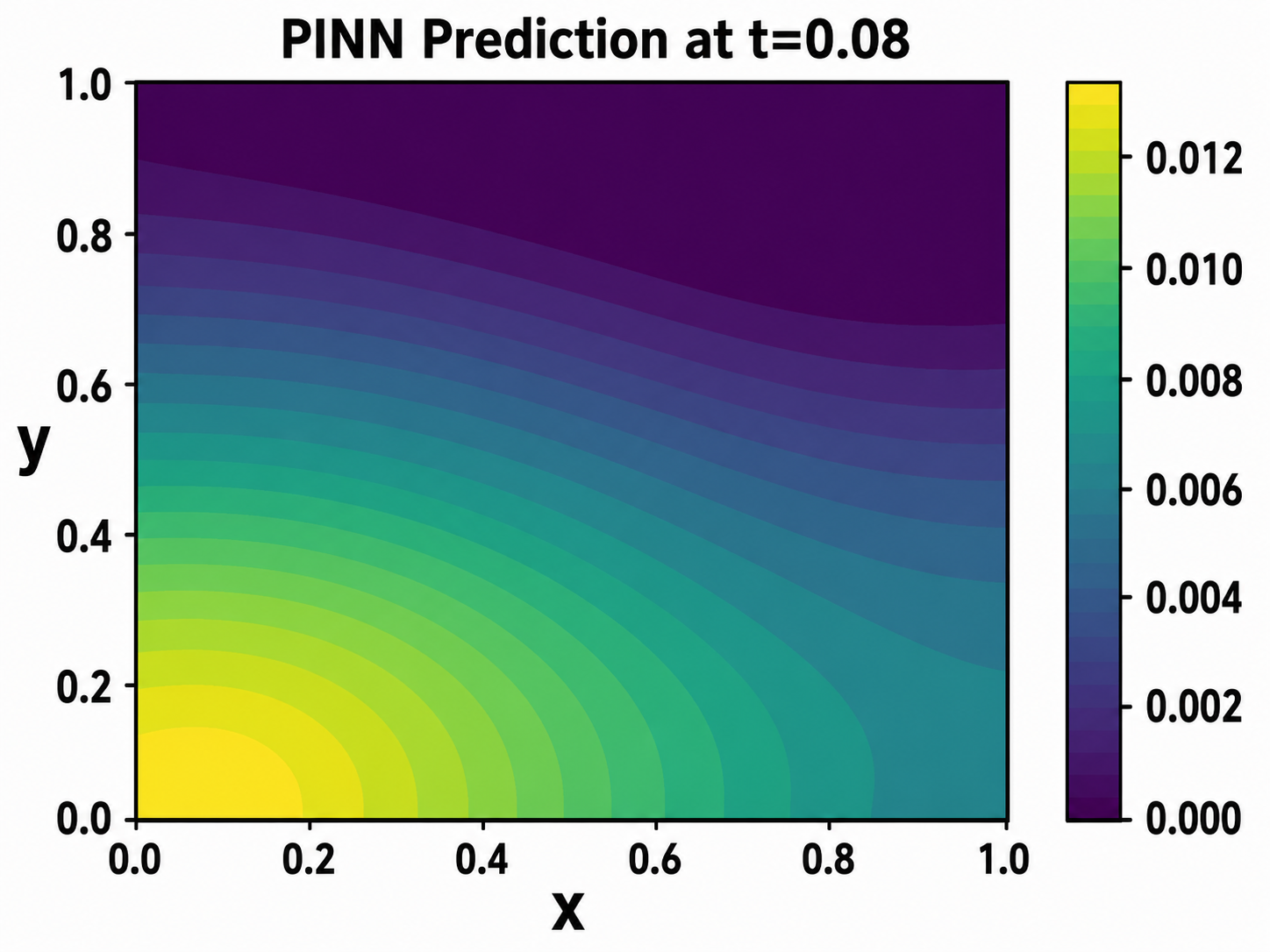} \\
    \includegraphics[width=0.3\linewidth]{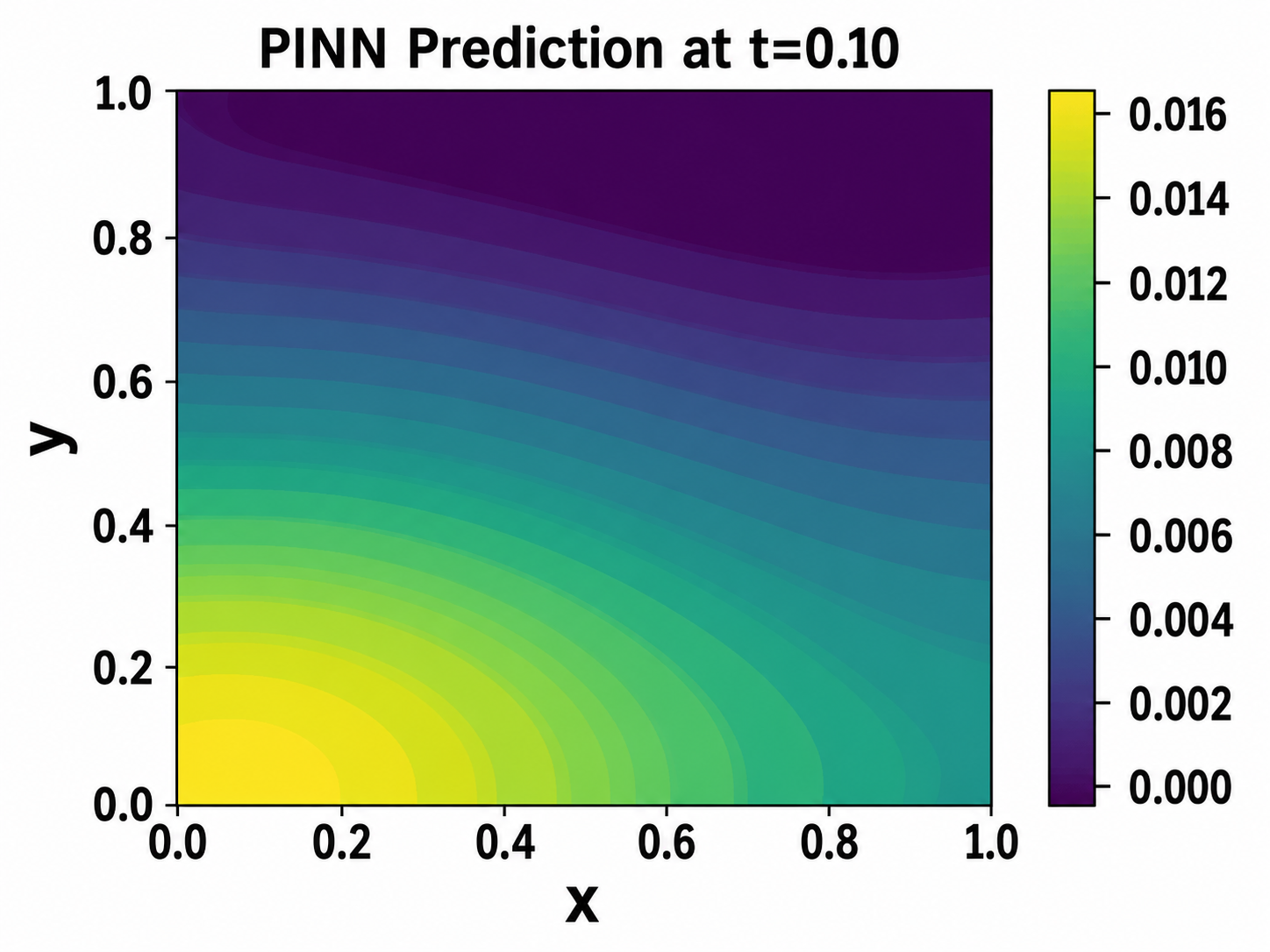}    \includegraphics[width=0.3\linewidth]{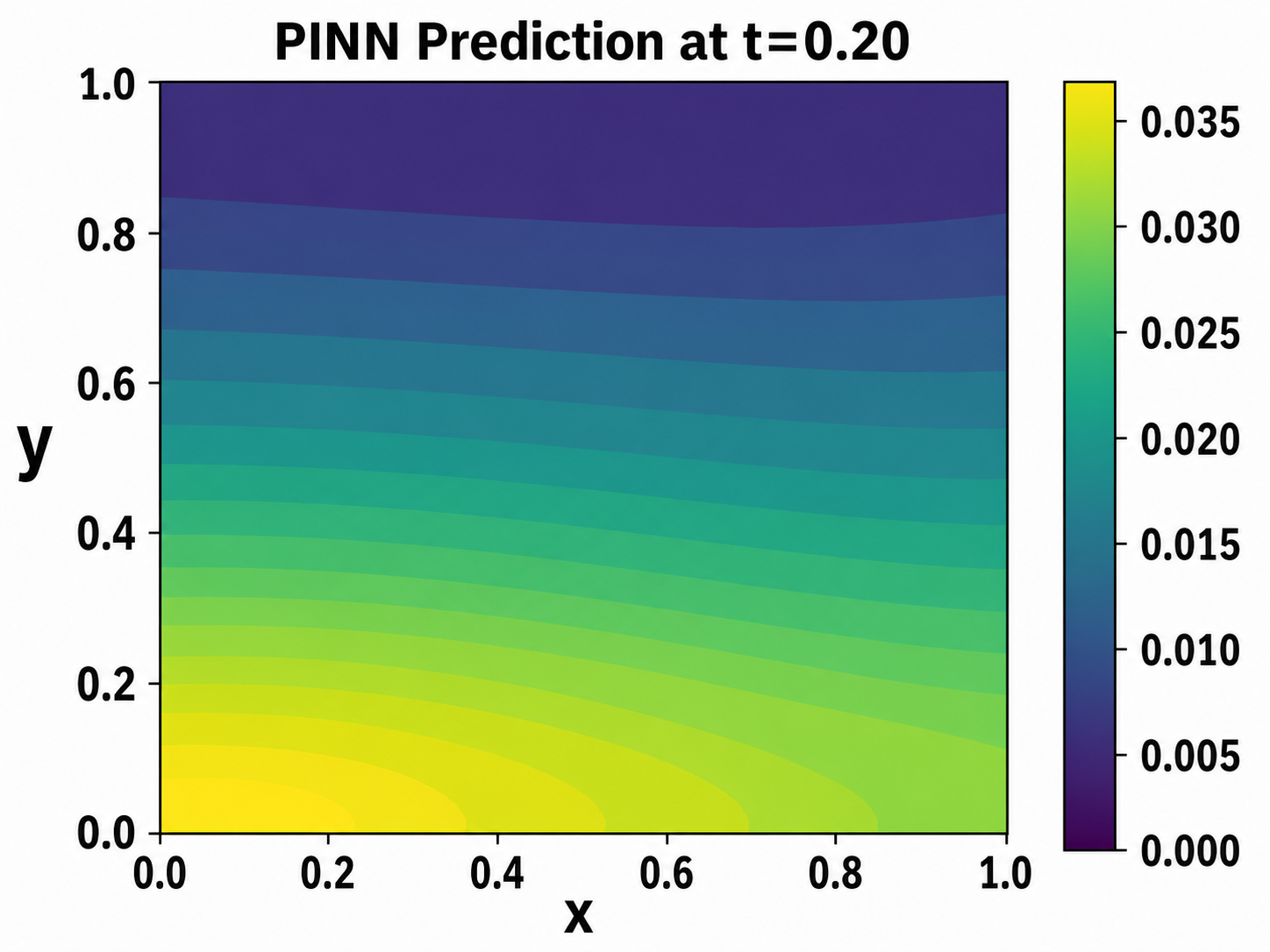}     \includegraphics[width=0.3\linewidth]{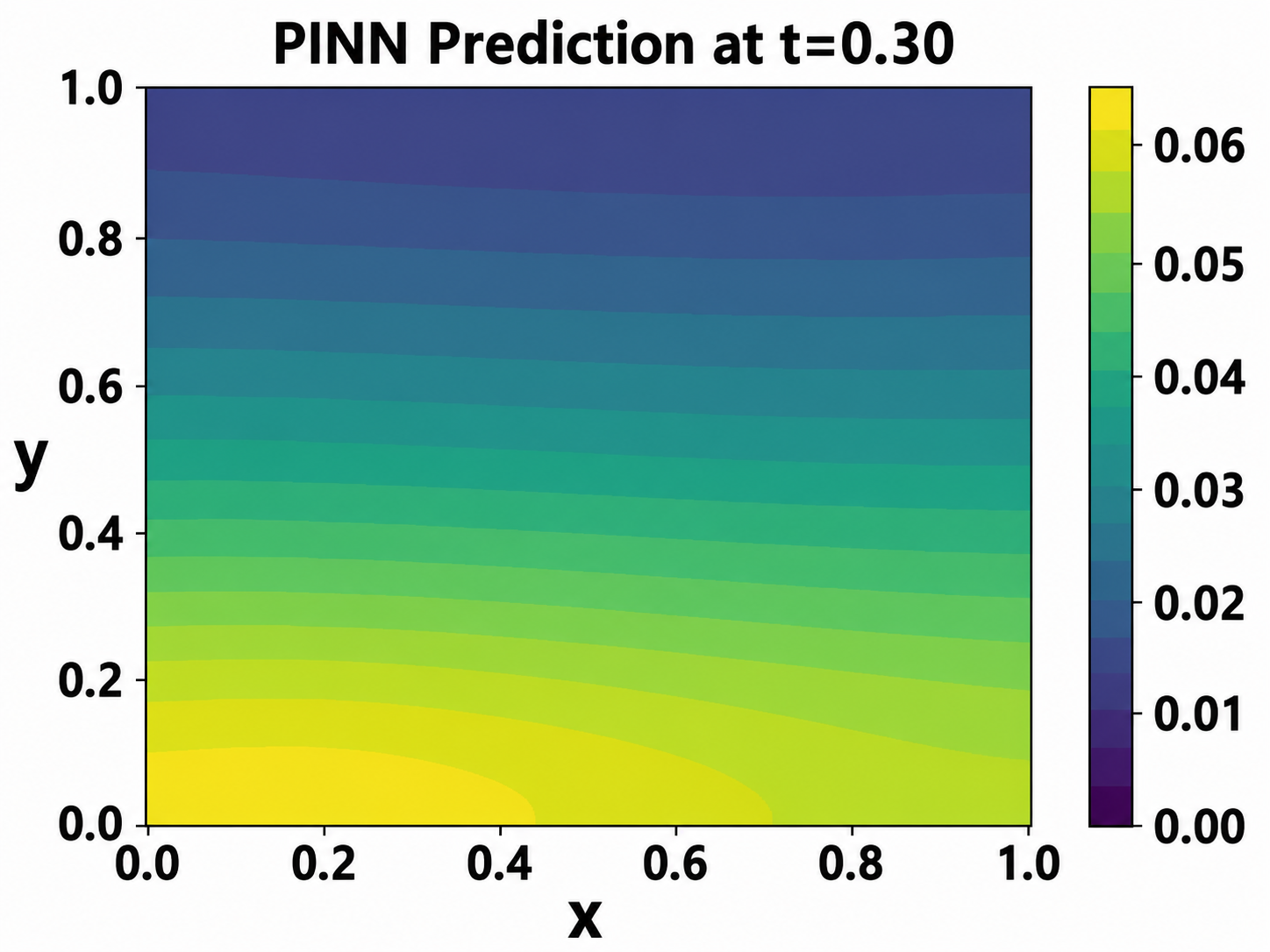} 
    \caption{Snapshots of the snowmobiles pollution CRVPINN simulation.}
    \label{fig:snowmobile1}
\end{figure}

\revision{The numerical simulations presented in Figure \ref{fig:snowmobile1} are consistent with the pollution measurements from snowmobiles shown in Figures \ref{fig:meas1}-\ref{fig:meas3} and with the stationary pollution measurements shown in Figures \ref{fig:meas4}-\ref{fig:meas6} and reproduce the principal transport mechanisms observed during the field experiments. The measurements in Figures \ref{fig:meas1}-\ref{fig:meas3} indicate that elevated pollutant concentrations persist near the snowmobile starting location, suggesting limited vertical mixing and the accumulation of exhaust gases close to the ground. Likewise, the mobile measurements in Figure \ref{fig:meas1} exhibit high NO$_2$ concentrations immediately after vehicle acceleration, followed by a decrease as the snowmobile moves away from the observation point. Similar effect can be observed in PM2.5 measurements in Figure \ref{fig:meas2} and PM10 measurements in Figure \ref{fig:meas3}.}

\revision{These observations are consistent with the physical assumptions incorporated into the mathematical model described in Section \ref{sec:pinnsec}. The moving Gaussian source represents the exhaust emitted by the snowmobile, while the prescribed vertical advection profile models thermal inversion by suppressing upward transport within the near-ground layer and allowing vertical motion only above the inversion height. As a consequence, the simulated pollutant cloud remains concentrated close to the ground while gradually spreading in the horizontal direction due to diffusion. The sequence of snapshots in Figure \ref{fig:snowmobile1} illustrates the evolution of this near-ground pollution plume as it propagates through the computational domain, demonstrating that pollutants remain trapped within the lower atmospheric layer rather than being rapidly dispersed upward.}

\revision{Although the present model is not calibrated to reproduce the measured concentration values quantitatively, it captures the dominant qualitative features observed in the experiments. In particular, both the measurements and the numerical simulations indicate that pollutants generated by snowmobiles accumulate near the ground under thermal inversion conditions, leading to locally elevated concentrations and slow atmospheric dispersion. This agreement supports the validity of the proposed advection-diffusion model as a physically meaningful representation of pollution transport in the Adventalen valley and provides a foundation for future inverse analyses in which model parameters will be identified directly from measurement data.}

\subsection{\revision{Nonlinear Burgers-Type Extension}}
\label{sec:nonlinear_burgers}

\revision{To investigate the performance of the robust physics-informed neural network for nonlinear partial differential equations, we extend the linear advection-diffusion model by introducing a solution-dependent advective velocity. The original model considered in the snowmobile pollution simulation has the form
\begin{equation}
    \frac{\partial u}{\partial t}
    + \mathbf{v}(x,y,t)\cdot\nabla u
    - \nabla\cdot\left(\mathbf{K}\nabla u\right)
    = S,
    \label{eq:linear_advection_diffusion}
\end{equation}
where $u=u(x,y,t)$ denotes the transported concentration, $\mathbf{v}$ is the prescribed velocity field, $\mathbf{K}$ is the diffusion tensor, and $S(x,y,t)$ represents the moving source associated with the snowmobile traffic. In the present implementation, the diffusion is anisotropic,
\begin{equation}
    \mathbf{K}
    =
    \begin{pmatrix}
        K_x & 0\\
        0 & K_y
    \end{pmatrix},
\end{equation}
with $K_x$ and $K_y$ denoting the diffusion coefficients in the horizontal and vertical directions, respectively.}

\revision{A nonlinear extension can be obtained by assuming that the horizontal transport velocity depends on the concentration itself. In particular, we introduce the Burgers-type nonlinear flux
$u\frac{\partial u}{\partial x}$.
The resulting nonlinear advection--diffusion equation reads
\begin{equation}
    \frac{\partial u}{\partial t}
    + u\frac{\partial u}{\partial x}
    + v_y(y,t)\frac{\partial u}{\partial y}
    - K_x\frac{\partial^2 u}{\partial x^2}
    - K_y\frac{\partial^2 u}{\partial y^2}
    = S.
\label{eq:burgers_advection_diffusion}
\end{equation}
The term $u \frac{\partial u}{\partial x}$ introduces a quadratic nonlinearity analogous to that appearing in the viscous Burgers equation. Consequently, the propagation velocity in the $x$-direction is no longer prescribed independently of the solution, but is determined locally by the concentration $u$. Regions characterized by larger values of $u$ are therefore transported with a different effective velocity than regions with smaller concentration. This mechanism may lead to nonlinear steepening of the transported concentration profile, while the diffusion terms counteract this effect by smoothing sharp gradients.}

\revision{The model \eqref{eq:burgers_advection_diffusion} can equivalently be written in conservative form as
\begin{equation}
    \frac{\partial u}{\partial t}
    + \frac{\partial}{\partial x}
      \left(\frac{u^2}{2}\right)
    + v_y(y,t)\frac{\partial u}{\partial y}
    - K_x u_{xx}
    - K_y u_{yy}
    = S,
\end{equation}
since
$\frac{\partial}{\partial x}
    \left(\frac{u^2}{2}\right)
    =
    u\frac{\partial u}{\partial x}.
$
This formulation emphasizes the connection of the proposed extension with nonlinear conservation laws and the classical viscous Burgers equation.}

\revision{For the neural-network approximation $u_\theta(x,y,t)$, the corresponding strong residual is defined as
\begin{equation}
\begin{split}
    RES(u_\theta)
    ={}&
    \frac{\partial u_\theta}{\partial t}
    + u_\theta\frac{\partial u_\theta}{\partial x}
    + v_y(y,t)\frac{\partial u_\theta}{\partial y}-
    K_x\frac{\partial^2u_\theta}{\partial x^2}
    -
    K_y\frac{\partial^2u_\theta}{\partial y^2}
    -
    S.
\end{split}
\label{eq:nonlinear_residual}
\end{equation}
All first- and second-order derivatives appearing in
\eqref{eq:nonlinear_residual} are evaluated using automatic
differentiation. The nonlinear product
$    u_\theta\frac{\partial u_\theta}{\partial x}
$
is therefore incorporated directly into the computational graph, and its contribution to the gradient of the training objective is automatically taken into account during optimization.}


\begin{figure}
    \centering
    \includegraphics[width=0.43\linewidth]{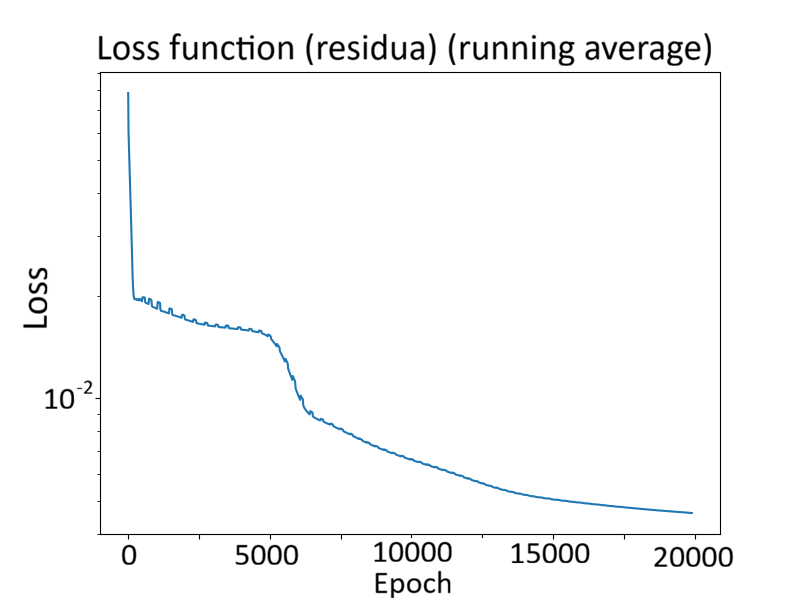}
    \caption{The convergence of the CRVPINN for non-linear model training.}
    \label{fig:nonlintrain}
\end{figure}

\begin{figure}
    \centering
    \includegraphics[width=0.3\linewidth]{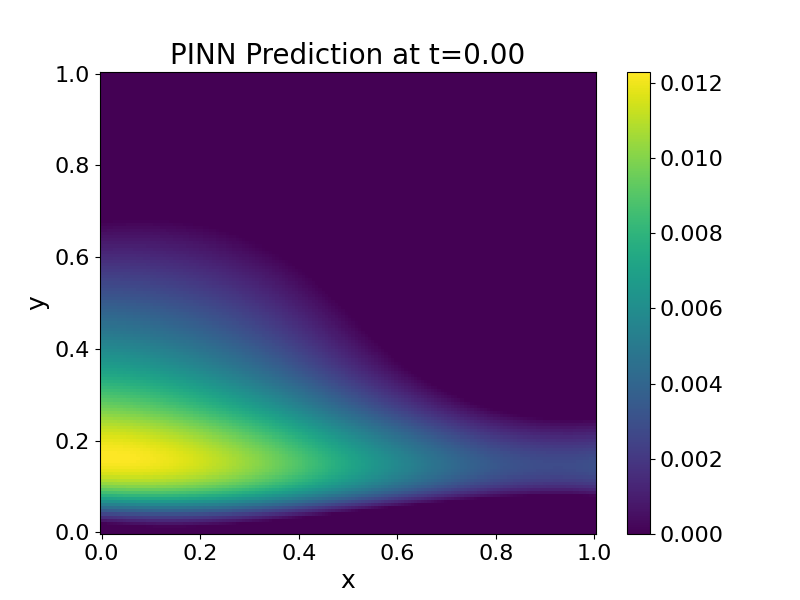}    \includegraphics[width=0.3\linewidth]{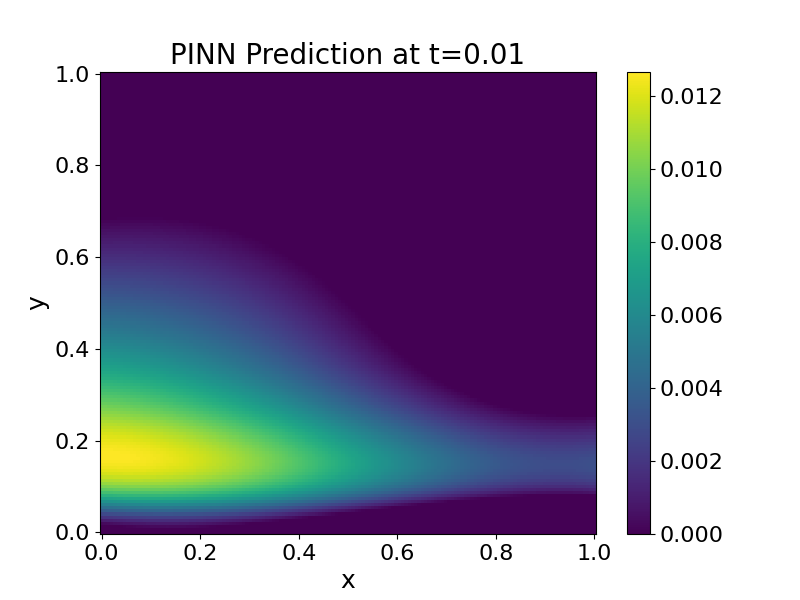}     \includegraphics[width=0.3\linewidth]{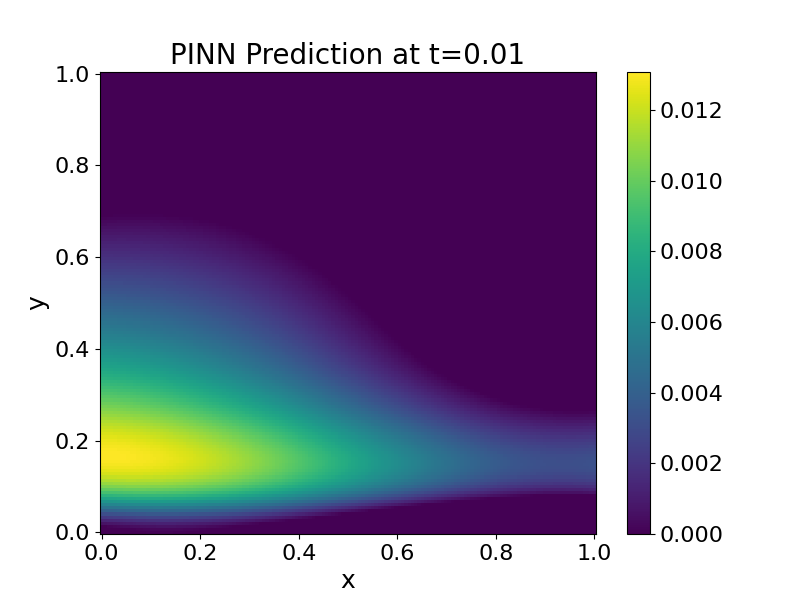} \\
    \includegraphics[width=0.3\linewidth]{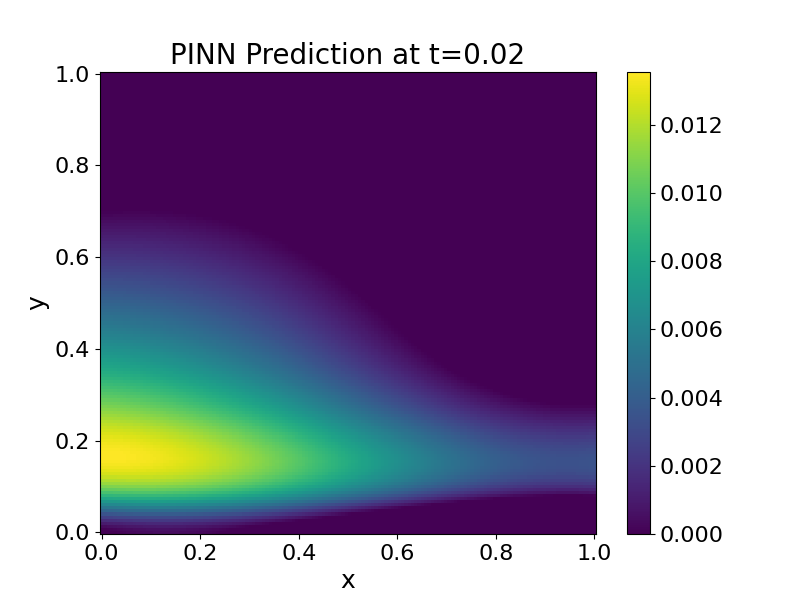}    \includegraphics[width=0.3\linewidth]{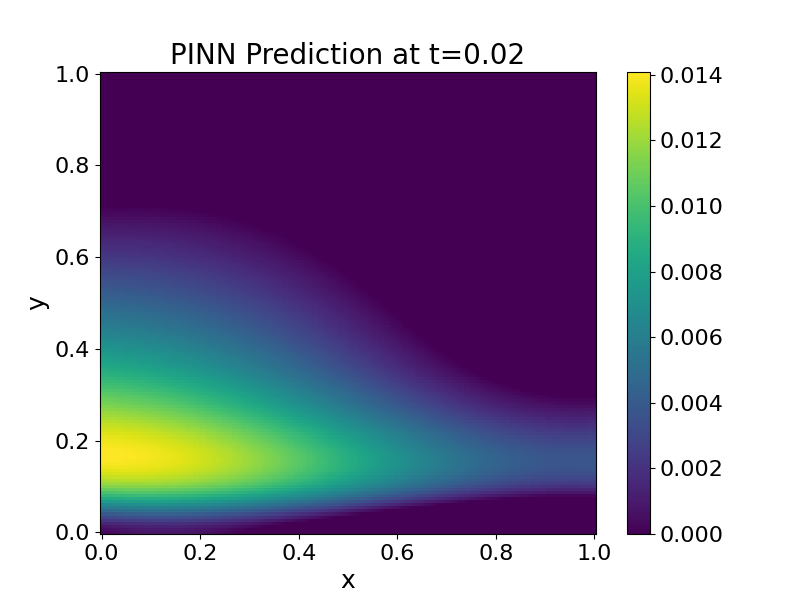}     \includegraphics[width=0.3\linewidth]{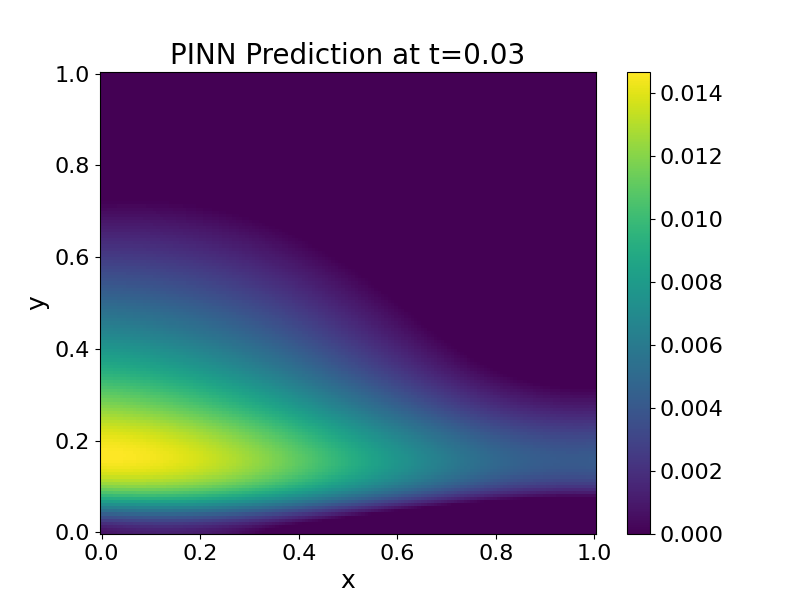} \\
    \includegraphics[width=0.3\linewidth]{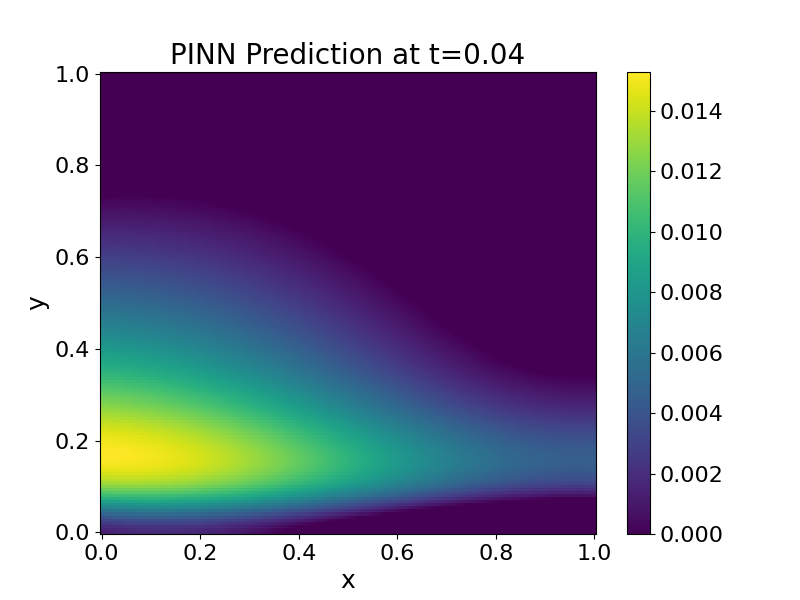}    \includegraphics[width=0.3\linewidth]{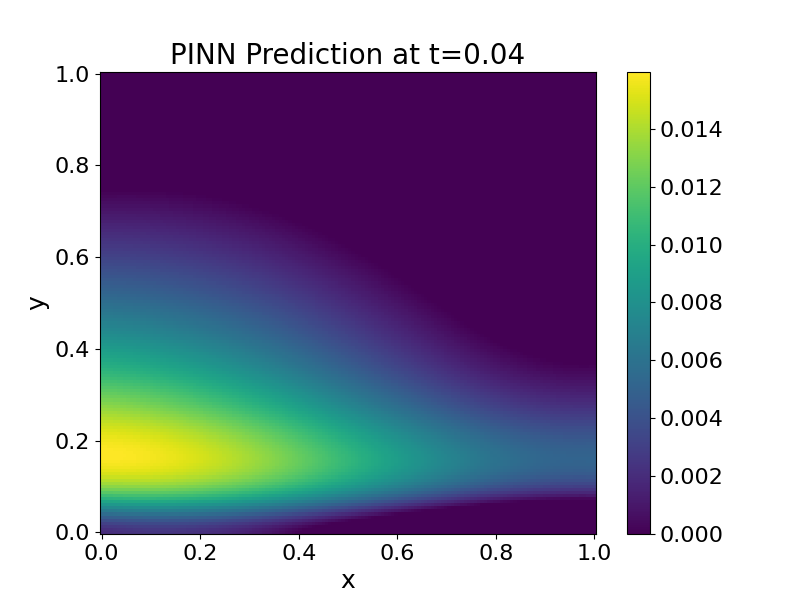}     \includegraphics[width=0.3\linewidth]{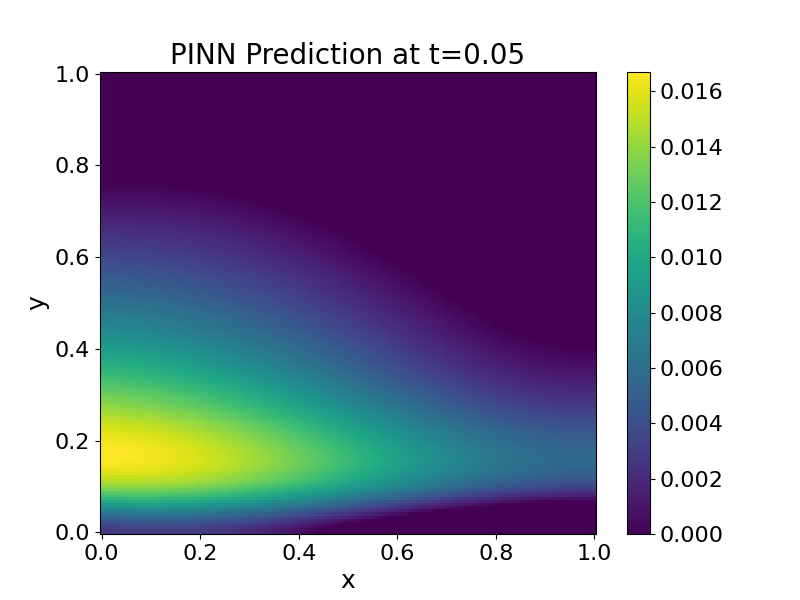} \\
    \includegraphics[width=0.3\linewidth]{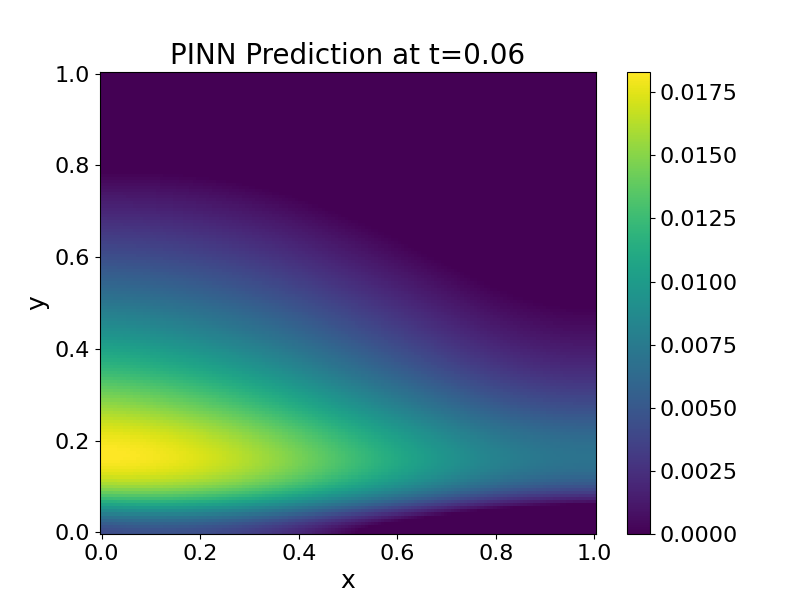}    \includegraphics[width=0.3\linewidth]{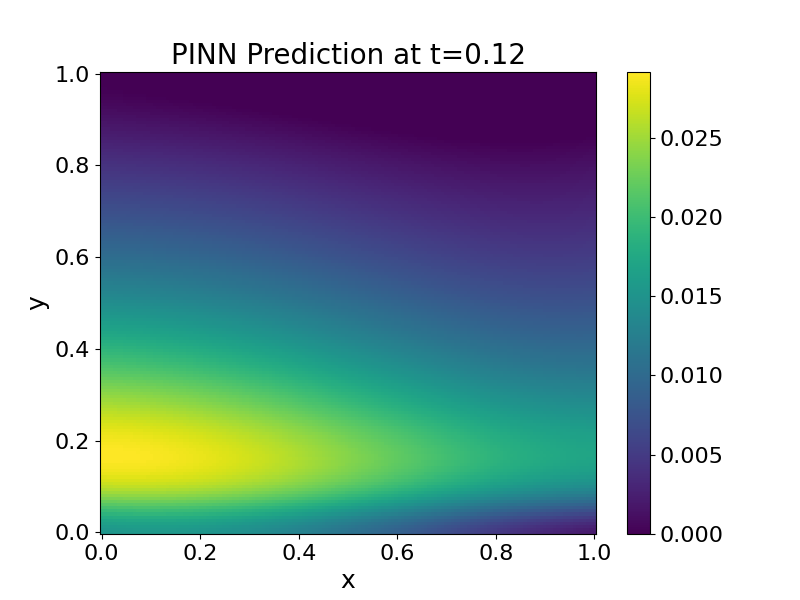}     \includegraphics[width=0.3\linewidth]{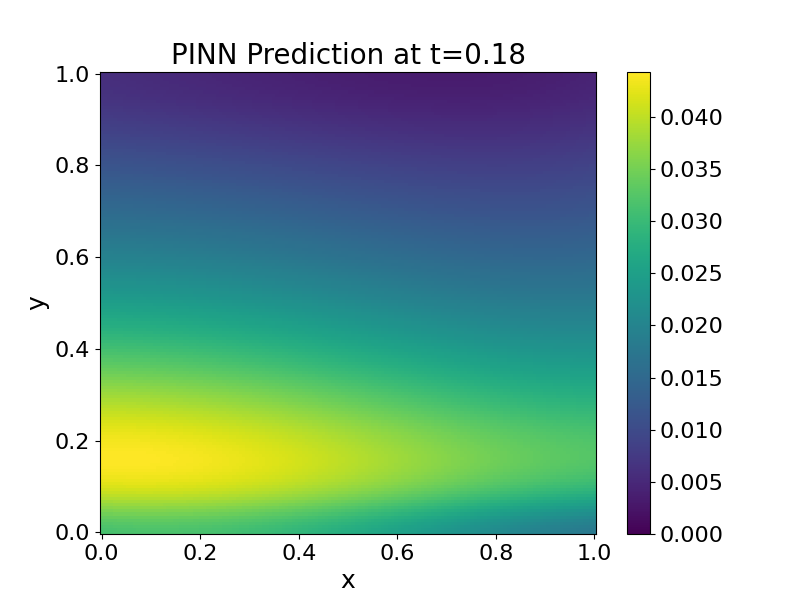} 
    \caption{Snapshots of the snowmobiles pollution using CRVPINN solving non-linear Burgers extension.}
    \label{fig:snowmobile1_nonlinear}
\end{figure}

\revision{The linear advection-diffusion loss}
\begin{lstlisting}[
    language=Python,
    caption={Fuel map extraction.},
    basicstyle=\ttfamily,
    keywordstyle=\color{blue},
    stringstyle=\color{red},
    commentstyle=\color{green},
    frame=single,
    numbers=left,
    numberstyle=\tiny\color{gray},
    stepnumber=1,
    numbersep=5pt,
    backgroundcolor=\color{lightgray!25}
]
    def residual_loss(self, pinn: PINN):
        x, y, t = get_interior_points
          (self.x_domain, self.y_domain, 
          self.t_domain, self.n_points, pinn.device())
        loss = dfdt(pinn, x, y, t).to(device) + 
          self.dTy(y, t)*dfdy(pinn, x, y, t).to(device) - 
          self.Kx*dfdx(pinn, x, y, t,order=2).to(device) - 
          self.Ky*dfdy(pinn, x, y, t,order=2).to(device) - 
          self.moving_source(x,y,t).to(device)
        return loss.pow(2).mean()
\end{lstlisting}
\revision{is replaced by the non-linear loss}
\begin{lstlisting}[
    language=Python,
    caption={Fuel map extraction.},
    basicstyle=\ttfamily,
    keywordstyle=\color{blue},
    stringstyle=\color{red},
    commentstyle=\color{green},
    frame=single,
    numbers=left,
    numberstyle=\tiny\color{gray},
    stepnumber=1,
    numbersep=5pt,
    backgroundcolor=\color{lightgray!25}
]
    def residual_loss(self, pinn: PINN):
        x, y, t = get_interior_points
          (self.x_domain, self.y_domain, 
          self.t_domain, self.n_points, pinn.device())
    loss = (
          dfdt(pinn, x, y, t) + f(pinn, x, y, t) 
          * dfdx(pinn, x, y, t)
          + self.dTy(y, t) * dfdy(pinn, x, y, t)
          - self.Kx * dfdx(pinn, x, y, t, order=2) 
          - self.Ky * dfdy(pinn, x, y, t, order=2)
          - self.moving_source(x, y, t)
            )
    return loss.pow(2).mean()
\end{lstlisting}

\revision{This non-linear model requires the same number of iterations, same architecture of the neural network, and it only differs by the loss function formula. Thus, it has similar computational cost from the point of view of the neural network training as the linear model. The training is presented in Figure \ref{fig:nonlintrain}, while the sequence of solution snapshots in Figure \ref{fig:snowmobile1_nonlinear}. The non-linear model results in wider spread of the pollutants. }

\section{Conclusions}

In the paper, we introduced a Physics-Informed Neural Network (PINN) framework
for time-dependent simulations of pollution propagation from moving objects.
Starting from the PINN formulation, we extended the model to a Robust
Variational Physics-Informed Neural Network (Robust VPINN) formulation to
obtain a robust loss function for the time-dependent advection-diffusion problem.
We also introduced the collocation-based CRVPINN method to accelerate the
training process.
\revision{We proved the robustness of the loss function associated with the
discrete weak formulation by establishing its inf-sup stability and boundedness.
We verified the robustness of the CRVPINN method on a model problem with a
manufactured solution. We also compared the CRVPINN method with the IGA-ADS
explicit dynamics solver.}

\revision{As an interesting case study, we considered pollution propagation from
snowmobiles in the Svalbard area, based on in-field measurements.
The CRVPINN method was applied to investigate the effects of thermal inversion
on the accumulation of pollution generated by snowmobiles.
We employed a time-dependent advection-diffusion model and further extended the
linear advection-diffusion equation to a nonlinear Burgers-type equation. The
results show that both the linear and nonlinear CRVPINN models require
approximately four times the computational time of the linear IGA-ADS
simulator. This result makes the nonlinear extension particularly interesting,
as it provides a substantially more general model at a comparable
computational cost.}
As for future work, we plan to formulate and solve an inverse problem for fitting model parameters to measurement data. We may include our inverse solvers
described in \cite{Barabasz01012011,NaturalComputing}.
 Future work will also include the development of large parallel simulations using isogeometric analysis explicit dynamics solvers \cite{Woźniak_Łoś_Paszyński_Dalcin_Calo_2017} or $hp$-adaptive finite element method solvers \cite{FI1,FI2}.

\section*{Acknowledgments}
The Authors gratefully acknowledge the support and assistance of the Polish Polar Station Hornsund for help with data collection.

The authors are grateful for the support from the funds that the Polish Ministry of Science and Higher Education assigned to AGH University of Krakow.
The work is supported by the ``Excellence initiative - research university" for AGH University of Krakow.

This work has been supported by the National Science Center, Poland, grant no. 2025/57/B/ST6/00058. 

\section*{Code Availability Statement}

\revision{The source code and Jupyter notebooks used to produce the numerical
experiments reported in this study are openly available in the public
GitHub repository:}

\url{https://github.com/sluzalec/crvpinn-pollution-spitsbergen}

\revision{The version of the code corresponding to this manuscript has been
permanently archived in Zenodo and is available at:}

\url{https://doi.org/10.5281/zenodo.22019451}

\revision{The repository contains the implementations of the PINN and CRVPINN
manufactured-solution benchmarks and the CRVPINN snowmobile pollution
propagation simulations. The codes were developed, tested, and executed
using Google Colab.}

\revision{
The IGA-ADS code for the problem with manufactured solution employed for comparison with CRVPINN code is available at:}

 \url{https://github.com/marcinlos/iga-ads/tree/develop/examples/advection}.

\section*{Declaration of Generative AI and AI-assisted technologies in the writing process}
During the preparation of this work, the author(s) used Gemini (Google)  and ChatGPT to assist with writing, correct English language usage, improve text clarity and organization, and debuging the codes. After using this tool/service, the author(s) reviewed and edited the content as needed and take(s) full responsibility for the content of the publication.
\section{Appendix}

%
%


The PINN implementation involves the definition of the advection field that models the horizontal movement of the air due to the thermal inversion. It is defined in the {\tt dTy} routine. According to the vertical temperature profile measurements, it is assumed that we have zero temperature gradient close to the ground and a negative temperature gradient above.

\begin{lstlisting}[
    language=Python,
    caption={Fuel map extraction.},
    basicstyle=\ttfamily,
    keywordstyle=\color{blue},
    stringstyle=\color{red},
    commentstyle=\color{green},
    frame=single,
    numbers=left,
    numberstyle=\tiny\color{gray},
    stepnumber=1,
    numbersep=5pt,
    backgroundcolor=\color{lightgray!25}
]
    def dTy(self,y,t):
      res = torch.where(t < 0.5, -2, -2)
      res = torch.where(y < 0.1, 0, res)
      return res.to(device)
\end{lstlisting}

The following definition involves the source term in the {\tt source} routine.
We assume that near the ground, pollution is emitted from left to right across the domain due to the moving snowmobile. 

\begin{lstlisting}[
    language=Python,
    caption={Fuel map extraction.},
    basicstyle=\ttfamily,
    keywordstyle=\color{blue},
    stringstyle=\color{red},
    commentstyle=\color{green},
    frame=single,
    numbers=left,
    numberstyle=\tiny\color{gray},
    stepnumber=1,
    numbersep=5pt,
    backgroundcolor=\color{lightgray!25}
]
    def source(self,y,t):
      H_MAX = 0.5
      VELOCITY = 10.0
      WAVE_WIDTH = .8
      Y_SPREAD = 0.4
# The state `y` is ignored by this source term, 
# which is purely a function of space and time.
# 1. Calculate the y-dependent maximum 
# height using a Gaussian profile.
      H_y = H_MAX * torch.exp
         (-y**2 / (2 * Y_SPREAD**2))
# 2. Calculate the position of the wave's 
# leading edge at time t.
      pos = VELOCITY * t
# 3. Normalize the x-position relative 
# to the width of the wave front.
     alpha = (pos - x) / WAVE_WIDTH
# 4. Clamp alpha to the [0, 1] range  to handle regions 
# ahead of, inside, and behind the wave front.
     alpha_clamped = torch.clamp(alpha, 0.0, 1.0)
# 5. Use a smooth cosine function (a "smoothstep") 
# for the transition from 0 to 1 across the wave front.
     height_factor = 0.5 * 
      (1 - torch.cos(math.pi * alpha_clamped))
# 6. The final source value is 
# the y-dependent max height scaled by the factor.
     final_source_value = H_y * height_factor        
     final_source_value = 
       torch.where(y <= 0.2, final_source_value, 0)
# return torch.where(t <= TOTAL_TIME / 2, 
    return final_source_value
\end{lstlisting}

\noindent The initial loss defined in {\tt initial\_loss} introduces zero everywhere.

\begin{lstlisting}[
    language=Python,
    caption={Fuel map extraction.},
    basicstyle=\ttfamily,
    keywordstyle=\color{blue},
    stringstyle=\color{red},
    commentstyle=\color{green},
    frame=single,
    numbers=left,
    numberstyle=\tiny\color{gray},
    stepnumber=1,
    numbersep=5pt,
    backgroundcolor=\color{lightgray!25}
]
    def initial_loss(self, pinn: PINN):
        x, y, t = get_initial_points
          (self.x_domain, self.y_domain, 
          self.t_domain, self.n_points, pinn.device())
        pinn_init = self.initial_condition(x, y)
        loss = f(pinn, x, y, t) - pinn_init
        return loss.pow(2).mean()
\end{lstlisting}

\noindent The {\tt boundary\_loss} defines zero Dirichlet boundary condition. 

\begin{lstlisting}[
    language=Python,
    caption={Fuel map extraction.},
    basicstyle=\ttfamily,
    keywordstyle=\color{blue},
    stringstyle=\color{red},
    commentstyle=\color{green},
    frame=single,
    numbers=left,
    numberstyle=\tiny\color{gray},
    stepnumber=1,
    numbersep=5pt,
    backgroundcolor=\color{lightgray!25}
]
    def boundary_loss(self, pinn: PINN):
        down, up, left, right = get_boundary_points
          (self.x_domain, self.y_domain, self.t_domain, 
          self.n_points, pinn.device())
        x_down, y_down, t_down = down
        x_up, y_up, t_up = up
        x_left,  y_left, t_left = left
        x_right, y_right, t_right = right
        loss_down = f(pinn, x_down, y_down, t_down)
        loss_up = f(pinn, x_up, y_up, t_up)
        loss_left = f(pinn, x_left, y_left, t_left)
        loss_right = f(pinn, x_right, y_right, t_right)
        return loss_down.pow(2).mean() + \
            loss_up.pow(2).mean() + \
            loss_left.pow(2).mean() + \
            loss_right.pow(2).mean()
\end{lstlisting}

\noindent Finally, we define the PDE loss in the {\tt residual\_loss} routine.

\begin{lstlisting}[
    language=Python,
    caption={Fuel map extraction.},
    basicstyle=\ttfamily,
    keywordstyle=\color{blue},
    stringstyle=\color{red},
    commentstyle=\color{green},
    frame=single,
    numbers=left,
    numberstyle=\tiny\color{gray},
    stepnumber=1,
    numbersep=5pt,
    backgroundcolor=\color{lightgray!25}
]
    def residual_loss(self, pinn: PINN):
        x, y, t = get_interior_points
          (self.x_domain, self.y_domain, 
          self.t_domain, self.n_points, pinn.device())
        loss = dfdt(pinn, x, y, t).to(device) + 
          self.dTy(y, t)*dfdy(pinn, x, y, t).to(device) - 
          self.Kx*dfdx(pinn, x, y, t,order=2).to(device) - 
          self.Ky*dfdy(pinn, x, y, t,order=2).to(device) - 
          self.source(y,t).to(device)
        return loss.pow(2).mean()
\end{lstlisting}

In the CRVPINN method, we test with Kronecker deltas and consider a discrete domain of points.
However, it is possible to implement the CRVPINN method with a very simple modification of the PINN code 
and still benefit from the better convergence provided by the robust loss function.
Namely, if we integrate back to the strong form multiplied by the Kronecker Deltas:
\begin{equation}
\begin{aligned}
( \nabla_{t+}u +\beta_x \nabla_{x+}u + \beta_y \nabla_{y+}u 
-\epsilon \Delta_h u,\delta_{i,j,k})_{h} =
(f,\delta_{i,j,k})_{h},
\end{aligned}
\label{eq:weakstrong}
\end{equation} 
and: 
\begin{equation}
\begin{aligned}
( \nabla_{t+}u,\delta_{i,j,k})_{h} = & \nabla_{t+}u_{i,j,k}, \\
(\beta_x \nabla_{x+}u,\delta_{i,j,k})_{h} = & \beta_x \nabla_{x+}u_{i,j,k}, \\
(\beta_y \nabla_{y+}u,\delta_{i,j,k})_{h} = &  \beta_y \nabla_{y+}u_{i,j,k}, \\
-(\epsilon \Delta_h u,\delta_{i,j,k})_{h} = & -\epsilon \Delta_h u_{i,j,k}, \\
(f,\delta_{i,j,k})_{h} = & f_{i,j,k}. 
\end{aligned}
\end{equation}
We recover the original PINN formulation of the residual estimated at the collocation points, thus:
\begin{equation}
\text{RES}(u)_{i,j,k} =  \nabla_{t+}u_{i,j,k} +\beta_x \nabla_{x+}u_{i,j,k} + \beta_y \nabla_{y+}u_{i,j,k} 
-\epsilon \Delta_h u_{i,j,k} - f_{i,j,k},
\end{equation}
and:
\begin{equation}
\begin{aligned}
  {\cal L}_{crvpinn}=  \text{RES}(u)^T\; \mathbf{G}^{-1} \; \text{RES}(u).
\end{aligned}
\end{equation}
where:
\begin{eqnarray}
\mathbf{G}_{i,j,k;l,m,n}=
& \frac{1}{h_xh_yh_t}\displaystyle{\left\{\begin{aligned}
&\phantom{-}\,\,6\ \quad \textrm{ for }(i,j,k)=(l,m,n) \\
&-1 \quad \textrm{ for }(l,m,n)\in \{(i+1,j,k),(i-1,j.k)\}\\
&-1 \quad \textrm{ for }(l,m,n)\in\{(i,j+1,k),(i,j-1,k)\} \\
&-1 \quad \textrm{ for }(l,m,n)\in\{(i,j,k+1),(i,j,k-1)\} \\
\end{aligned}
\right.} 
\label{eq:gram}
\end{eqnarray} 
In other words, we just take the PINN code residual and we multiply it by the Gram matrix.

\begin{lstlisting}[
    language=Python,
    caption={Fuel map extraction.},
    basicstyle=\ttfamily,
    keywordstyle=\color{blue},
    stringstyle=\color{red},
    commentstyle=\color{green},
    frame=single,
    numbers=left,
    numberstyle=\tiny\color{gray},
    stepnumber=1,
    numbersep=5pt,
    backgroundcolor=\color{lightgray!25}
]
G = torch.eye(N_POINTS * N_POINTS * N_POINTS)
def linearized(ix, iy, it):
  return ix * N_POINTS * N_POINTS + iy * N_POINTS + it

def nearby(ix, iy, it):
  return 
  [(ix + 1, iy, it), (ix - 1, iy, it), (ix, iy - 1, it), 
    (ix, iy + 1, it), (ix, iy, it -1), (ix, iy, it+1)]

for ix in range(N_POINTS):
  for iy in range(N_POINTS):
    for it in range(N_POINTS):
      i = linearized(ix, iy, it)
      G[i, i] = 1
for ix in range(1, N_POINTS - 1):
  for iy in range(1, N_POINTS - 1):
    for it in range(1, N_POINTS - 1):
      i = linearized(ix, iy, it)
      G[i, i] = 4
      for jx, jy, jt in nearby(ix, iy, it):
        j = linearized(jx, jy, jt)
        G[i, j] = -1
hx = 1.0 / N_POINTS
hy = 1.0 / N_POINTS
ht = 1.0 / N_POINTS
G = G / (hx*hy*ht)
G = G.to(device)
G_LU = torch.linalg.lu_factor(G)
\end{lstlisting}

Instead of inverting the Gram matrix, we perform LU factorization and use it to compute the loss value at each iteration.

 \bibliographystyle{elsarticle-num} 
 \bibliography{cas-refs-no-url}

\end{document}